\documentclass{article}

\usepackage[preprint]{neurips_2026}
\usepackage{subcaption}

\usepackage[utf8]{inputenc} % allow utf-8 input
\usepackage[T1]{fontenc}    % use 8-bit T1 fonts
\usepackage{hyperref}       % hyperlinks
\usepackage{url}            % simple URL typesetting
\usepackage{booktabs}       % professional-quality tables
\usepackage{amsfonts}       % blackboard math symbols
\usepackage{nicefrac}       % compact symbols for 1/2, etc.
\usepackage{microtype}      % microtypography
\usepackage{xcolor}         % colors
\usepackage[utf8]{inputenc} % allow utf-8 input
\usepackage[T1]{fontenc}    % use 8-bit T1 fonts
\usepackage{hyperref}       % hyperlinks
\usepackage{url}            % simple URL typesetting
\usepackage{booktabs}       % professional-quality tables
\usepackage{amsfonts}       % blackboard math symbols
\usepackage{nicefrac}       % compact symbols for 1/2, etc.
\usepackage{graphicx}
\usepackage{microtype}      % microtypography
\usepackage{overpic}
\usepackage{amsmath, amssymb, amsfonts}
\usepackage{amsthm}
\usepackage{multirow}
\usepackage{colortbl}
\usepackage[ruled,vlined]{algorithm2e}
\usepackage{float}
\usepackage{placeins}
\usepackage{wrapfig}
\usepackage[table]{xcolor}
\usepackage{booktabs} % 用于三线表格式 \toprule, \midrule, \bottomrule
\usepackage{tabularx} % 用于可变宽度列（X列）的表格环境
\usepackage{array}    % 用于增强表格列格式控制

\usepackage{pgfplots}
\pgfplotsset{compat=1.18}
\usepackage{sansmath} % 保证字体和论文一致

\usepackage{wrapfig}

\graphicspath{{./figs/}}

\newtheorem{theorem}{Theorem}
\newtheorem*{theoreml}{Theorem}
\newtheorem{lemma}{Lemma}
\newtheorem{proposition}{Proposition}
\newtheorem{assumption}{Assumption}
\newtheorem{remark}{Remark}

\definecolor{abstractGray}{HTML}{F5F7FA}
\definecolor{ruleNavy}{HTML}{244B73}
\title{{Geometry Conflict}: Explaining and Controlling Forgetting in LLM Continual Post-Training}
\author{%
\textbf{Yuanyi Wang}$^{1}$\thanks{Equal contribution.}\, \footnotemark[3],
\textbf{Yifan Yang}$^{1}$\footnotemark[1],
\textbf{Su Lu}$^{1}$,
\textbf{Yanggan Gu}$^{1}$,
\textbf{Pengkai Wang}$^{1}$, \\
\textbf{Wenjun Wang}$^{1}$,
\textbf{Zhaoyi Yan}$^{2}$,
\textbf{Congkai Xie}$^{2}$,
\textbf{Jianmin Wu}$^{1}$, \\
\textbf{Jialun Cao}$^{3}$,
\textbf{Shing-Chi Cheung}$^{3}$,
\textbf{Hongxia Yang}$^{1,2,4}$\thanks{Corresponding author.} \thanks{yuan-yi.wang@connect.polyu.hk, hongxia.yang@polyu.edu.hk} \\
$^{1}$The Hong Kong Polytechnic University, PolyU \quad
$^{2}$InfiX.ai \\
$^{3}$The Hong Kong University of Science and Technology, HKUST \\
$^{4}$PolyU-Daya Bay Technology and Innovation Research Institute \\
\\
\textbf{Code:} \textcolor{ruleNavy}{\url{https://github.com/wyy-code/GCWM}}
}

\begin{document}

\maketitle

\begin{abstract}
Continual post-training aims to extend large language models (LLMs) with new knowledge, skills, and behaviors, yet it remains unclear when sequential updates enable capability transfer and when they cause catastrophic forgetting.
Existing methods mitigate forgetting through sequential fine-tuning, replay, regularization, or model merging, but offer limited criteria for determining when incorporating new updates is beneficial or harmful.
In this work, we study LLM continual post-training through three questions: \textit{What drives forgetting? When do sequentially acquired capabilities transfer or interfere? How can compatibility be used to control update integration?}
We address these questions through \textit{task geometry}: we represent each post-training task by its parameter update and study the covariance geometry induced by the update.
Our central finding is that: \textbf{forgetting can be considered as a state-relative update-integration failure, it arises when the covariance geometries induced by tasks misalign with the geometry of the evolving model state}.
Sequential updates transfer when they remain compatible with the model state shaped by previous updates, and interfere when state-relative geometry conflict becomes high.
Motivated by this finding, we propose \textbf{G}eometry-\textbf{C}onflict \textbf{W}asserstein \textbf{M}erging (GCWM), a data-free update-integration method that constructs a shared Wasserstein metric via Gaussian Wasserstein barycenters and uses geometry conflict to gate geometry-aware correction.
Across Qwen3 0.6B--14B on domain-continual and capability-continual settings, GCWM consistently outperforms data-free baselines, improving retention and final performance without replay data.
These results identify geometry conflict as both an explanatory signal for forgetting and a practical control signal for LLM continual post-training.
\end{abstract}

\section{Introduction}

Continual post-training is becoming an increasingly important paradigm for extending large language models (LLMs) \cite{shi2025continual,kumar2025llm}. Rather than learning jointly over all desired capabilities or data, a model is expected to learn through a sequence of post-training stages, each targeting a new domain \cite{ke2025demystifying,zhao2025redone}, skill \cite{tang2025synthesizing,yano2025lamdagent}, or behavior \cite{tan2025scaling,du2025post}. This process is natural for real scenarios as capabilities are introduced incrementally. However, sequential post-training faces a fundamental challenge: learning a new task undermines the knowledge acquired from previous ones, a phenomenon known as catastrophic forgetting \cite{van2024continual,loke2025overcoming}, often driven by interference between sequential parameter updates.

Existing approaches can be broadly categorized into four classes: sequential fine-tuning \cite{ke2022continual,wang2025see}, replay-based methods that revisit past data \cite{hickok2025scalable,rolnick2019experience}, regularization methods that constrain update drift \cite{ahn2019uncertainty,pomponi2020efficient}, or model merging strategies that combine task-specific adaptations  \cite{feng2025aimmerging,zhang2025merge}.
These approaches have led to important progress, but they still lack a principled account of \emph{task compatibility} in continual post-training. As a result, they often struggle to answer a central practical question: when should new parameter updates be strongly integrated into the current model, and when should such integration be restrained? This issue is particularly pronounced for LLMs, where tasks are highly heterogeneous, post-training objectives differ substantially, and the same update magnitude can lead to very different retention outcomes \cite{wang2025model}.

\vspace{-0.1em}
To address this problem, we study LLM continual post-training through three questions:
\textit{What drives forgetting? When do sequentially acquired capabilities transfer or interfere? How can compatibility be used to control update integration?}
We answer these questions through a task-geometry view of post-training updates.
Specifically, we represent each task by its parameter update and study the induced covariance geometry, which captures not only update magnitude but also the subspaces and spectral structure through which a task changes the model.
We define \emph{geometry conflict} as a normalized Bures--Wasserstein discrepancy \cite{bhatia2019bures} between task-induced covariance geometries in a shared space, and use its state-relative form to measure compatibility with the evolving LLM state.

\vspace{-0.1em}
Our analysis (Sec.~\ref{sec:findings}) across Qwen3 scales and continual strategies compares geometry conflict with update norm, subspace alignment ratio \cite{gargiulo2025task}, and gradient conflict \cite{wang2021gradient}.
It reveals a central mechanism: forgetting can be considered as a state-relative update-integration failure, it arises when the covariance geometries induced by tasks misalign with the geometry of the evolving model state, 
whereas transfer occurs when new updates remain compatible with the state shaped by previous updates.
This explains why raw update norm and isolated pairwise compatibility are insufficient, and why geometry conflict serves as a natural signal for controlling sequential update integration.

\vspace{-0.1em}
Motivated by this finding, we propose \textbf{G}eometry-\textbf{C}onflict \textbf{W}asserstein \textbf{M}erging (GCWM), a data-free update-integration method for LLM continual post-training. 
GCWM constructs task-induced covariance geometry, builds a shared Wasserstein metric via Gaussian Wasserstein barycenters, and uses geometry conflict to gate geometry-aware correction, which allows GCWM to perform compatibility-controlled update integration.
We further provide theoretical support showing that the induced loss change is controlled by geometry conflict and gated merge displacement.

\vspace{-0.1em}
Across \emph{domain-continual} and \emph{capability-continual} settings, GCWM consistently improves retention and final performance over data-free baselines without replay data. On Qwen3 models from 0.6B to 14B, GCWM remains the strongest data-free update-integration method across scales, showing that geometry conflict is useful not only as an explanatory signal for forgetting but also as a practical control signal for continual post-training.
In summary, our contributions are summarized as follows:
% \\
% (iii) We propose \textbf{G}eometry-\textbf{C}onflict \textbf{W}asserstein \textbf{M}erging (GCWM), a data-free update-integration method for LLM continual post-training, with theory showing that its induced loss change is controlled by geometry conflict and gated merge displacement. Across Qwen3 0.6B--14B, GCWM improves retention and performance over baselines in domain-continual and capability-continual settings.
\vspace{0.4em}
\\
\textbf{(i)} We develop a task-geometry analysis of LLM continual post-training and show that forgetting is better explained as a state-relative update-integration failure, beyond update norm and isolated pairwise compatibility.
\\
\textbf{(ii)} We introduce \emph{geometry conflict}, a Bures--Wasserstein distance over task-induced covariance geometries, and identify it as both an explanatory signal for forgetting and a compatibility signal for update integration, complementing existing subspace alignment ratio and gradient conflict.
\\
\textbf{(iii)} We propose Geometry-Conflict Wasserstein Merging, a data-free update-integration method that constructs a shared Wasserstein metric and gates geometry-aware correction by layer-wise conflict.
\\
% (iv) With theory proving that GCWM`s induced loss change is controlled by geometry conflict and gated merge displacement, GCWM is evaluated on Qwen3 0.6B--14B under domain- and capability-continual settings, showing improved final performance over data-free baselines without replay data.
\textbf{(iv)} We derive a conflict-controlled theory linking GCWM's relative loss to geometry conflict and gated merge displacement, and validate GCWM on Qwen3 0.6B--14B across domain- and capability-continual settings, improving final performance over data-free baselines without replay data.

\vspace{-0.2em}
\section{Preliminary}
\label{sec:background}

\subsection{Problem Setup}
We study continual post-training for LLMs. Starting from a pretrained model with parameters \(\theta_{\mathrm{pre}}\), the model is adapted through a sequence of tasks \(\mathcal{T}=\{T_1,\ldots,T_K\}\), where each task introduces a new domain, skill, or behavior. For task \(T_t\), we denote its task-specific update by
\[
\Delta_t = \theta_t - \theta_{\mathrm{pre}},
\]
where \(\theta_t\) is the model adapting to \(T_t\). We use these task updates, which may be parameter-efficient or full-model updates, as the basic objects for analyzing in LLM continual post-training.

\subsection{Task Geometry and Compatibility Signals}
\label{sec:task_geo_comp_sig}
A task update is not fully characterized by its norm: two updates with similar magnitude can affect different subspaces and induce different forgetting behavior. For a layer \(\ell\), let \(\Delta_t^{(\ell)} \in \mathbb{R}^{d_{\mathrm{out}}\times d_{\mathrm{in}}}\) denote the update matrix of task \(T_t\). 
Motivated by the task vector \cite{ilharcoediting}, we define task geometry as:
% \vspace{-0.4em}
{\footnotesize
\[
C_t^{(\ell)} = \big(\Delta_t^{(\ell)}\big)^\top \Delta_t^{(\ell)} ,
\]
}
which captures the dominant directions of the update.
To compare two tasks, we project them into a shared basis and measure their discrepancy using a normalized Bures--Wasserstein distance \cite{bhatia2019bures}:
\vspace{-0.1em}
{\footnotesize
\[
\gamma_{ij}^{(\ell)}
=
\frac{
d_{\mathrm{B}}^2(B_i^{(\ell)},B_j^{(\ell)})
}{
\mathrm{tr}(B_i^{(\ell)})+\mathrm{tr}(B_j^{(\ell)})+\varepsilon
},
\]
}
% \vspace{-0.1em}
where \(B_i^{(\ell)}\) and \(B_j^{(\ell)}\) are the projected geometries. We refer to \(\gamma_{ij}^{(\ell)}\) as \emph{geometry conflict}. Lower values indicate more compatible task-induced geometries. In Sec.~\ref{sec:findings}, we compare geometry conflict with three standard diagnostics: update norm, subspace alignment ratio (SAR) \cite{marczak2025no}, and gradient cosine conflict \cite{yu2020gradient}. State-relative variants replace one task update with the current continual-training state. Full metric definitions and aggregation details are provided in Appendix~\ref{app:analysis_metrics}.

\subsection{Related Work}
\textbf{Continual Post-training}
has become an increasingly important paradigm for extending LLMs beyond their original pretraining distribution, including domain adaptation \cite{saad2023udapdr,eschbach2024exploring}, capability acquisition \cite{yin2024enhancing,bansal2024llm}, and behavior alignment over sequential stages \cite{yang2024behavior,ye2026align3gr}. Existing approaches largely follow four lines. Sequential fine-tuning directly adapts the model stage by stage, but is highly prone to forgetting under heterogeneous task sequences \cite{ji2024reversing,qiao2024learn}. Replay-based methods mitigate forgetting by revisiting historical data \cite{zhang2025gere,feng2026forever}, while regularization-based methods constrain update drift to preserve prior knowledge \cite{lu2025controlled,ahn2019uncertainty}. Model merging that combines task-specific adaptations offers a plug-in workflow, but struggles to resolve cross-task interference \cite{zhang2025merge,marczak2024magmax}. 
However, most existing methods emphasize preserving prior performance during sequential updates while offering limited guidance on the task-compatibility conditions under which sequential interactions should be encouraged or suppressed. Our work addresses this gap through a task-compatibility perspective.

\textbf{Continual Model Merging}
provides a data-efficient alternative to standard sequential adaptation by composing task-specific parameter updates in weight space \cite{wang2026mergepipe, yang2026model,zhou2025democratizing}. 
Recent work studies sequential settings in which models arrive incrementally over time \cite{libecame,bui2026mergeslide,zhou2026model}, including projection-based sequential merging \cite{tang2025merging}, stability-based methods based on null-space filtering or test-time gating \cite{qiumingle,qiu2025null}, resource-constrained online merging of adapters \cite{shenaj2025k}, and broader hybrid frameworks that combine continual learning and model merging \cite{phan2025toward}. 
Our method is instantiated as a data-free continual merging method, but the broader goal is to study continual post-training through task compatibility and use merging as an mechanism for exploiting the resulting compatibility findings.

\textbf{Compatibility Metrics and Signals.}
Recent work studies compatibility via parameter discrepancy \cite{ke2025demystifying,chen2025coefficients}, gradient alignment \cite{wei2025modeling}, and subspace or spectral overlap \cite{marczak2025no,tammerging}. 
\emph{Demystifying Mergeability} \cite{ke2025demystifying} shows that subspace overlap and gradient alignment are stable method-agnostic indicators, but these signals remain largely diagnostic. 
In contrast, we introduce \emph{geometry conflict} as a method-native control signal derived from task-induced covariance geometry, and construct a shared merging metric via Bures--Wasserstein geometry \cite{bhatia2019bures} and Gaussian Wasserstein barycenters \cite{alvarez2016fixed}.

\section{What Governs Forgetting in Continual Post-Training?}
\label{sec:findings}
Before introducing GCWM, we first ask what makes a continual post-training step harmful.
Across Qwen3 models \cite{yang2025qwen3} from 0.6B to 14B and four representative strategies--Seq.\ SFT, EWC regularization \cite{kirkpatrick2017overcoming}, FOREVER replay \cite{feng2026forever}, and AIMMerging \cite{feng2025aimmerging}--we compare forgetting with update norm, SAR, gradient conflict, and our geometry conflict.
Here, retention loss is the positive old-task drop from each task's best previous score, reported in percentage points (pp) when scaled by 100, and \(\rho_s\) denotes Spearman rank correlation.
The analysis yields four findings: update norm is only a coarse drift baseline; geometry conflict refines SAR-based compatibility; state-relative geometry mismatch best tracks continual forgetting; geometry and gradient conflict reveal complementary failure modes. Extended diagnostics and bootstrap confidence intervals are organized in Appendix~\ref{app:sec3_additional}.

\subsection{Update Norm Is Insufficient to Explain Forgetting}
\label{sec:update_norm}

A natural hypothesis is that forgetting is mainly driven by parameter drift: larger updates should induce larger retention loss.
We test this by comparing update norm with forgetting, and contrast it with geometry signals that use different reference points.
In Figs.~\ref{fig:state_reference_core} and \ref{fig:state_reference_support}, active conflict is the mean pairwise geometry conflict among active task updates, while state and global gaps measure geometry mismatch between active task updates and the evolving model state.
Fig.~\ref{fig:state_reference_support}(a) shows that update norm has a nontrivial but coarse association with retention loss (\(|\rho_s|=0.48\)).
State-relative geometry is stronger: the global state-active gap reaches \(|\rho_s|=0.59\), exceeding both update norm and active-pair conflict (\(|\rho_s|=0.30\)).
The scale breakdown in Fig.~\ref{fig:state_reference_core}(b) further shows that this advantage becomes clearer in larger LLMs: the global gap increases from \(0.16\) at 0.6B to \(0.86\) at 14B, while update norm remains a weaker drift baseline.
Overall, \textbf{update norm measures how far the model moves, but not whether the movement remains compatible with task-induced geometries}. Bootstrap confidence intervals and additional step-level rankings are provided in Appendices~\ref{app:sec3_confidence} and \ref{app:additional_step_analysis}.

\begin{figure}[H]
  \centering
  \vspace{-0.4em}
  \includegraphics[width=\linewidth]{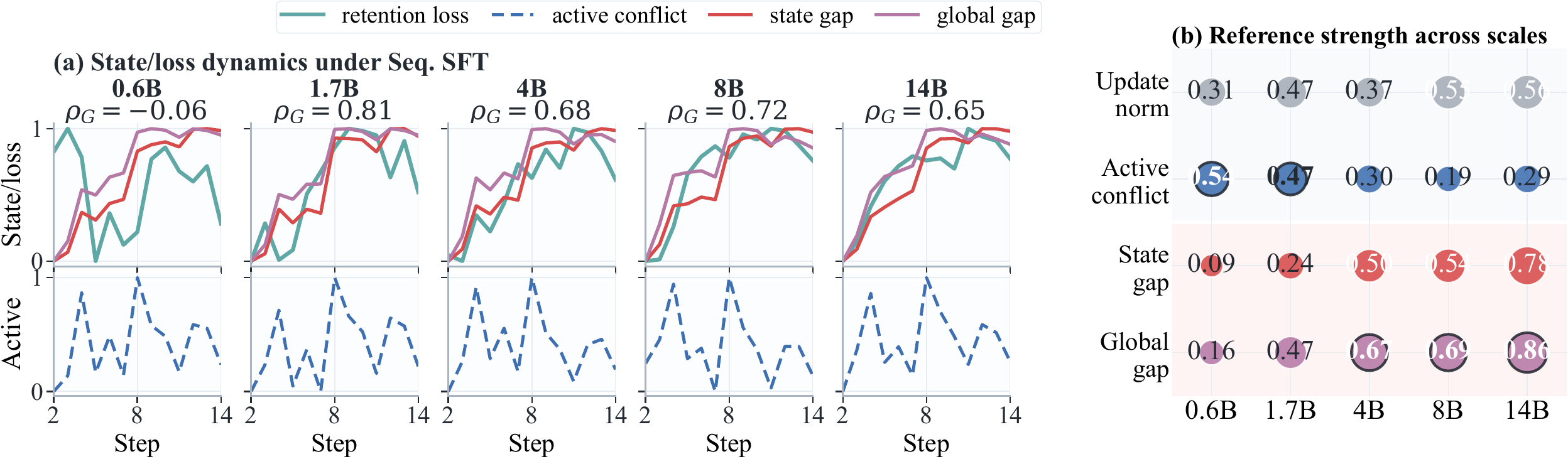}
  \vspace{-1.0em}
  \caption{
  {State-relative geometry tracks forgetting across continual steps and scales.}
  Panel (a) shows normalized SFT dynamics within each scale; panel (b) reports \(|\rho_s|\) between each signal and loss.
  }
  \label{fig:state_reference_core}
  \vspace{-0.8em}
\end{figure}

\vspace{-0.8em}
\subsection{Geometry Conflict Refines Subspace Compatibility}
\label{sec:sar_geometry}
\vspace{-0.4em}

\begin{wrapfigure}[20]{r}{0.34\linewidth}
  \centering
  \vspace{-2.5em}
  \includegraphics[width=\linewidth]{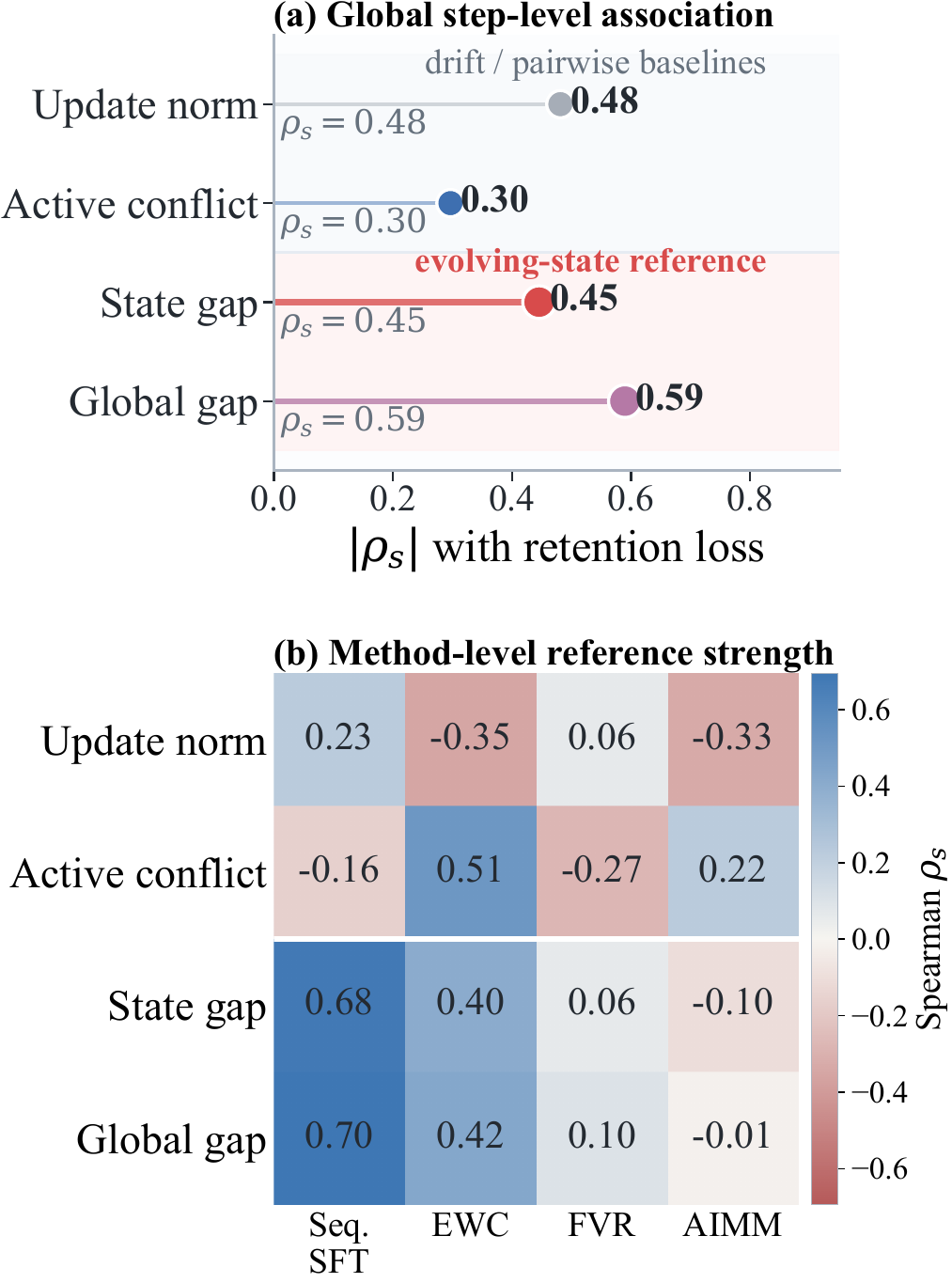}
  \vspace{-1.2em}
  \caption{\small
  {Global and method-level associations.}
  Top: global \(|\rho_s|\). Bottom: signed method-level \(\rho_s\); FVR denotes FOREVER.
  }
  \label{fig:state_reference_support}
  \vspace{-2em}
\end{wrapfigure}
Subspace overlap is a natural compatibility proxy: if two updates act on similar directions, they may be easier to integrate.
We therefore compare SAR with geometry conflict (Sec.~\ref{sec:task_geo_comp_sig}).
As shown in Fig.~\ref{fig:pairwise_compatibility}(a), SAR and geometry conflict are related but non-redundant: their global rank association is moderate (\(\rho_s=0.27\)), and task pairs with similar SAR can still exhibit very different geometry conflict.
SAR captures \emph{where} updates overlap; geometry conflict captures whether their induced covariance geometry is compatible in that shared space.
\vspace{-0.4em}

Pairwise geometry is useful for regime diagnosis, but it is not a standalone predictor of forgetting.
In Fig.~\ref{fig:pairwise_compatibility}(a--c), SAR percentile ranks task-pair SAR values; GC-drop and GC-forget denote correlations between pairwise geometry conflict and immediate old-task score change or best-previous forgetting, respectively.
Fig.~\ref{fig:pairwise_compatibility}(b) shows that GC-drop stays near zero across methods and scales, while GC-forget is scale-sensitive: it is visible on 0.6B--4B (\(0.28/0.31/0.30\)) but weak on 8B and 14B (\(0.12/0.02\)).
Fig.~\ref{fig:pairwise_compatibility}(c) further illustrates this point: large drops, such as Math\(\rightarrow\)History (\(12.9\) pp) and Math\(\rightarrow\)Economics (\(12.3\) pp), do not form a single pairwise-conflict pattern.
Thus, pairwise compatibility is informative but insufficient, motivating the state-relative analysis in Sec.~\ref{sec:state_relative_geometry}.
Overall, \textbf{SAR and geometry conflict capture different levels of compatibility}. Pairwise confidence intervals, heatmaps, summaries, and harmful transitions are provided in Appendices~\ref{app:sec3_confidence} and \ref{app:pairwise_analysis}.

\vspace{-0.6em}
\subsection{State-Relative Geometry Conflict Tracks Continual Forgetting}
\label{sec:state_relative_geometry}
\vspace{-0.2em}

Sec.~\ref{sec:sar_geometry} shows that isolated task pairwise compatibility is incomplete.
In LLM continual post-training, each incoming update is applied to an evolving model state that already encodes previous updates.
The question is whether incoming task geometry remains compatible with the \emph{current} state.

Fig.~\ref{fig:state_reference_core}(a) tracks this effect under Seq.\ SFT.
Active-pair conflict fluctuates across steps, while state and global gaps more closely follow the growth of retention loss, especially from 1.7B to 14B.
The method-level heatmap in Fig.~\ref{fig:state_reference_support}(b) shows the same pattern is strongest under direct sequential updating: state/global signals reach \(0.68/0.70\) for Seq.\ SFT and remain substantial for EWC (\(0.40/0.42\)), but weaken when replay or merging compresses forgetting variance.
This identifies the evolving model state, rather than isolated task pairs, as the relevant reference point for geometry-based forgetting analysis. Full confidence intervals and method-stratified correlations are in Appendices~\ref{app:sec3_dashboard}--\ref{app:additional_step_analysis}.

\vspace{-0.4em}
\subsection{Geometry and Gradient Conflict Reveal Complementary Failure Modes}
\label{sec:geometry_gradient}
\vspace{-0.3em}

Finally, we ask whether geometry conflict simply duplicates gradient conflict.
The answer is no.
Here, \(q/k/v/o\) denote attention projections, \(gate/up/down\) denote MLP projections, top-layer share is the fraction of top-ranked conflict layers in each family, min grad-cos is the minimum gradient cosine, and neg-grad ratio is the fraction of negative-cosine pairs.
Fig.~\ref{fig:geometry_gradient}(d) shows a sharp module-level separation: top geometry-conflict layers concentrate in \(up\_\mathrm{proj}\), \(gate\_\mathrm{proj}\), \(v\_\mathrm{proj}\), and \(down\_\mathrm{proj}\), whereas top negative-gradient layers are dominated by \(k\_\mathrm{proj}\) and \(q\_\mathrm{proj}\).
Together, the four geometry-heavy families account for about \(0.89\) of top geometry-conflict layers, while query/key projections account for about \(0.86\) of top negative-gradient layers.
Fig.~\ref{fig:geometry_gradient}(e) further shows that the geometry-conflict locus changes with the update-integration strategy, while negative-gradient conflict remains consistently query/key-centric.
Fig.~\ref{fig:geometry_gradient}(f) complements this module view: the global geometry gap is the strongest forgetting-aligned signal among the plotted predictors, whereas gradient diagnostics are more aligned with old-task mean and overall performance.
Thus, \textbf{geometry conflict and gradient conflict are complementary diagnostics}.
Gradient conflict exposes optimization-level opposition, while geometry conflict captures update-integration mismatch.
This distinction is important for GCWM: geometry conflict is not used as a replacement for gradient diagnostics, but as a native signal for controlling how strongly sequential updates should be integrated. 
Confidence intervals for the geometry and gradient target comparison and decompositions are in Appendices~\ref{app:sec3_confidence} and \ref{app:family_mechanism}.

\begin{figure}[t]
  \centering
  \includegraphics[width=\linewidth]{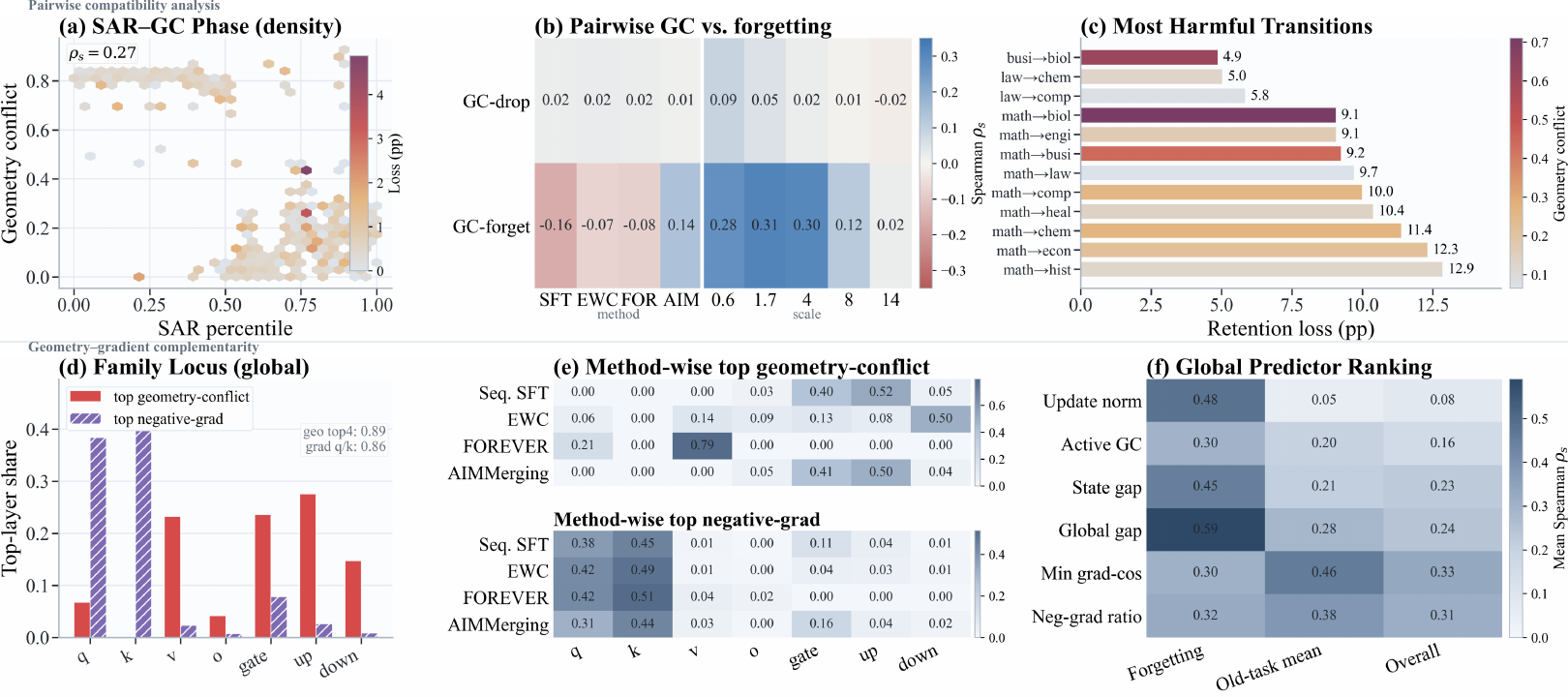}
  \caption{
\textbf{Pairwise compatibility and conflict complementarity.}
(a)--(c) SAR and geometry conflict stratify task-pair transfer regimes, while pairwise conflict alone weakly predicts forgetting. GC-drop is the signed association with the immediate old-task delta; GC-forget measures degradation from each old task's best prior score. 
(d)--(f) reveal complementary failure modes: top-layer share is the fraction of top-ranked layers within each projection family, and (f) reports global step-level $|\rho_s|$.
}
  \label{fig:pairwise_compatibility}
  \label{fig:geometry_gradient}
  \vspace{-1em}
\end{figure}

\vspace{-0.2em}
\section{Geometry Conflict Wasserstein Merging}
\label{sec:method}

We now turn the state-relative geometry findings in Sec.~\ref{sec:findings} into a data-free update-integration algorithm.
\textbf{G}eometry-\textbf{C}onflict \textbf{W}asserstein \textbf{M}erging (GCWM) operates on task vectors, estimates layer-wise geometry conflict, constructs a shared Wasserstein metric, and uses a conflict gate to control how strongly geometry-aware correction is applied.
At continual step \(t\), GCWM then applies only the incremental change of the update, yielding a compatibility-controlled continual post-training merge.

\vspace{-0.4em}
\subsection{Task Geometry and Conflict Gate}

GCWM represents each active task update by its layer-wise covariance geometry.
For an active update \(\Delta_i\) and target linear layer \(\ell\), let
\(\Delta_i^{(\ell)}\in\mathbb{R}^{d_{\mathrm{out}}\times d_{\mathrm{in}}}\).
We define
{\footnotesize
\begin{equation}
C_i^{(\ell)}
=
\big(\Delta_i^{(\ell)}\big)^\top \Delta_i^{(\ell)}+\lambda I,
\label{eq:covariance-geometry}
\end{equation}
}
which captures the dominant update subspaces and spectral energy while ensuring numerical stability.

To compare multiple active updates in a shared system, GCWM computes a truncated SVD
{\small
\[
\Delta_i^{(\ell)}
\approx
U_i^{(\ell)}\Sigma_i^{(\ell)}\big(V_i^{(\ell)}\big)^\top,
\]
}
retains the principal right-singular directions, and forms
{\small
\begin{equation}
Q^{(\ell)}
=
\mathrm{orth}\!\left(
\left[
V_1^{(\ell)},V_2^{(\ell)},\dots,V_m^{(\ell)}
\right]
\right),
\label{eq:shared-basis}
\end{equation}
}
where \(m\) is the number of active task updates. The projected geometry is
{\small
\begin{equation}
B_i^{(\ell)}
=
\big(Q^{(\ell)}\big)^\top C_i^{(\ell)}Q^{(\ell)} .
\label{eq:projected-geometry}
\end{equation}
}
The operators \(\{B_i^{(\ell)}\}_{i=1}^m\) are used for conflict estimation and shared-metric construction.

For two projected geometries, GCWM defines layer-wise geometry conflict by the normalized Bures--Wasserstein discrepancy
{\small
\begin{equation}
\gamma_{ij}^{(\ell)}
=
\frac{
d_{\mathrm{B}}^2\!\left(B_i^{(\ell)},B_j^{(\ell)}\right)
}{
\mathrm{tr}\!\left(B_i^{(\ell)}\right)
+
\mathrm{tr}\!\left(B_j^{(\ell)}\right)
+
\varepsilon
},
\qquad
d_{\mathrm{B}}^2(A,B)
=
\mathrm{tr}(A)+\mathrm{tr}(B)
-
2\,\mathrm{tr}\!\left(
\left(A^{1/2}BA^{1/2}\right)^{1/2}
\right),
\label{eq:pairwise-geometry-conflict}
\end{equation}
}
where \(\varepsilon>0\) is a stabilizer. Smaller \(\gamma_{ij}^{(\ell)}\) indicates more compatible task-induced geometries.
GCWM aggregates pairwise conflicts and converts the result into a layer-wise gate:
\begin{align}
g^{(\ell)}
&=
\sum_{i<j}w_{ij}\gamma_{ij}^{(\ell)},
\qquad
\sum_{i<j}w_{ij}=1,
\label{eq:layer-geometry-conflict}
\\
\alpha^{(\ell)}
&=
\alpha_{\min}
+
(\alpha_{\max}-\alpha_{\min})
\,\sigma\!\big(\kappa(g^{(\ell)}-\tau)\big).
\label{eq:conflict-gate}
\end{align}
Here \(w_{ij}\) are normalized task-pair weights, \(\tau\) is the conflict threshold, and \(\kappa\) controls gate sharpness.
Thus, geometry conflict becomes an actionable layer-wise control signal rather than a purely diagnostic score.

\subsection{Shared Wasserstein Metric and Gated Merge}

Given \(\{B_i^{(\ell)}\}_{i=1}^m\), GCWM constructs a shared merging metric through the Gaussian Wasserstein barycenter
\begin{equation}
\bar{B}^{(\ell)}
=
\arg\min_{B\succeq 0}
\sum_{i=1}^{m}
\omega_i
d_{\mathrm{B}}^2\!\left(B,B_i^{(\ell)}\right),
\qquad
\sum_i \omega_i=1 .
\label{eq:wasserstein-barycenter}
\end{equation}
The barycenter \(\bar{B}^{(\ell)}\) defines the local metric in which active updates are aligned before merging.

Let
\(\hat{\Delta}_i^{(\ell)}=\Delta_i^{(\ell)}Q^{(\ell)}\).
GCWM whitens the projected update, applies a base merge operator \(\mathcal{M}\), and recolors the result:
\begin{align}
\tilde{\Delta}_i^{(\ell)}
&=
\hat{\Delta}_i^{(\ell)}
\big(\bar{B}^{(\ell)}\big)^{-1/2},
\label{eq:whitened-update}
\\
\tilde{\Delta}_{\mathrm{geo}}^{(\ell)}
&=
\mathcal{M}\!\left(
\{\tilde{\Delta}_i^{(\ell)}\}_{i=1}^m;
\{\omega_i\}_{i=1}^m
\right),
\nonumber
\\
\Delta_{\mathrm{geo}}^{(\ell)}
&=
\tilde{\Delta}_{\mathrm{geo}}^{(\ell)}
\big(\bar{B}^{(\ell)}\big)^{1/2}
\big(Q^{(\ell)}\big)^\top .
\label{eq:recolor-update}
\end{align}
We instantiate \(\mathcal{M}\) with weighted WUDI~\cite{cheng2025whoever}. The geometry-aware branch is then blended with an ungated plain merge:
\begin{equation}
\Delta_{\mathrm{merge}}^{(\ell)}
=
\alpha^{(\ell)}\Delta_{\mathrm{geo}}^{(\ell)}
+
\big(1-\alpha^{(\ell)}\big)\Delta_{\mathrm{plain}}^{(\ell)} .
\label{eq:final-merge}
\end{equation}
For clarity, Eqs.~\eqref{eq:whitened-update}--\eqref{eq:recolor-update} present the projected form; the implementation uses the corresponding regularized full-space transform detailed in Appendix~\ref{app:algorithm}.

\subsection{Incremental Continual Update}

GCWM is applied incrementally. At step \(t\), let \(\mathcal{A}_t\) be the active set of task updates selected by the memory policy, containing the current update and optionally historical updates or the previous merged state.
For each target layer, GCWM computes \(\Delta_{\mathrm{merge},t}^{(\ell)}\) using Eqs.~\eqref{eq:pairwise-geometry-conflict}--\eqref{eq:final-merge}.
Instead of reapplying the full merged update, GCWM applies only its change relative to the previous merged state:
\begin{equation}
\Delta_{\mathrm{inc},t}^{(\ell)}
=
\Delta_{\mathrm{merge},t}^{(\ell)}
-
\Delta_{\mathrm{merge},t-1}^{(\ell)},
\qquad
\Delta_{\mathrm{merge},0}^{(\ell)}=0 .
\label{eq:incremental-merge}
\end{equation}
The model update is \quad
$\theta_t^{(\ell)}
=
\theta_{t-1}^{(\ell)}
+
\eta_t\Delta_{\mathrm{inc},t}^{(\ell)},$ \quad
where \(\eta_t\) is a step coefficient. This rule keeps continual post-training tied to newly induced compatibility-controlled changes.

\subsection{Theoretical Support for GCWM}

We provide a fixed-step proposal analysis of GCWM relative to the plain merge.
This analysis isolates the effect of geometry-aware correction before the incremental differencing in Eq.~\eqref{eq:incremental-merge}; the implementation applies the change between consecutive merged proposals.
Under local smoothness, projected-geometry adequacy, and layer-wise metric-curvature assumptions stated in Appendix~\ref{app:proof_conflict_controlled_integration}, the relative loss effect of the GCWM correction is controlled by geometry conflict and metric displacement.

Let
$\Theta_{\mathrm{plain},t}
=
\theta_{t-1}+\eta_t\Delta_{\mathrm{plain},t},
\qquad
\Theta_{\mathrm{gcwm},t}
=
\theta_{t-1}+\eta_t\Delta_{\mathrm{merge},t},$
where \(\Delta_{\mathrm{plain},t}\) and \(\Delta_{\mathrm{merge},t}\) are the plain and GCWM merge proposals at step \(t\).
Let \(\widetilde B_t^{(\ell)}\) denote the implementation-aligned full-space metric induced by the shared Wasserstein metric.

\begin{theorem}[Conflict-Controlled Integration]
\label{thm:conflict_controlled_integration}
Assume \(m_t\ge2\) and \(0\le \alpha_{\min}\le\alpha_{\max}\le1\).
Under the assumptions in Appendix~\ref{app:proof_conflict_controlled_integration}, the additional loss incurred on a previously acquired task \(u\) by the GCWM proposal relative to the plain proposal satisfies
\begin{equation*}
\mathcal L_u(\Theta_{\mathrm{gcwm},t})
-
\mathcal L_u(\Theta_{\mathrm{plain},t})
\le
\eta_t\sum_\ell c_{u,t}^{(\ell)}g_t^{(\ell)}
+
\frac{\eta_t^2}{2}
\sum_\ell d_{u,t}^{(\ell)}
\bigl\|
\Delta_{\mathrm{merge},t}^{(\ell)}
-
\Delta_{\mathrm{plain},t}^{(\ell)}
\bigr\|_{\widetilde B_t^{(\ell)}}^2,
\end{equation*}
where \(c_{u,t}^{(\ell)},d_{u,t}^{(\ell)}\ge0\) are local constants.
\end{theorem}

Theorem~\ref{thm:conflict_controlled_integration} shows that the relative loss effect of the geometry-aware proposal is bounded by two quantities: the shared geometry conflict \(g_t^{(\ell)}\) and the metric displacement from the plain merge. The next result shows how the conflict gate controls this displacement.

\begin{proposition}[Compatibility Regimes of Update Integration]
\label{prop:compatibility_regimes}
Assume \(0\le \alpha_{\min}\le\alpha_{\max}\le1\). For each layer \(\ell\), let
\[
D_t^{(\ell)}
=
\bigl\|
\Delta_{\mathrm{geo},t}^{(\ell)}
-
\Delta_{\mathrm{plain},t}^{(\ell)}
\bigr\|_{\widetilde B_t^{(\ell)}}.
\]
Then
$\bigl\|
\Delta_{\mathrm{merge},t}^{(\ell)}
-
\Delta_{\mathrm{plain},t}^{(\ell)}
\bigr\|_{\widetilde B_t^{(\ell)}}
=
\alpha_t^{(\ell)}D_t^{(\ell)},
\qquad
\bigl\|
\Delta_{\mathrm{merge},t}^{(\ell)}
-
\Delta_{\mathrm{geo},t}^{(\ell)}
\bigr\|_{\widetilde B_t^{(\ell)}}
=
(1-\alpha_t^{(\ell)})D_t^{(\ell)}.$
Moreover, Eq.~\eqref{eq:conflict-gate} implies
\(g_t^{(\ell)}\le\tau \Rightarrow
\alpha_t^{(\ell)}\le(\alpha_{\min}+\alpha_{\max})/2\),
whereas
\(g_t^{(\ell)}\ge\tau \Rightarrow
\alpha_t^{(\ell)}\ge(\alpha_{\min}+\alpha_{\max})/2\).
Thus, low-conflict layers receive weaker geometry-aware correction, while high-conflict layers receive stronger correction.
\end{proposition}

In summary, Theorem~\ref{thm:conflict_controlled_integration} and Proposition~\ref{prop:compatibility_regimes} show that GCWM is controlled by geometry conflict at both the loss and update levels. Proofs are provided in Appendices~\ref{app:proof_conflict_controlled_integration} and~\ref{app:proof_compatibility_regimes}.

\vspace{-0.4em}
\section{Experiments}
\label{sec:experiments}
\vspace{-0.3em}

We evaluate GCWM as a data-free update-integration method under domain and capability shifts. Our main comparisons focus on data-free merging baselines; sequential, regularized, and replay-based methods are included as reference continual-training pipelines.

\vspace{-0.5em}
\subsection{Setup}
\label{sec:exp_setup}
\vspace{-0.3em}

\textbf{Models:}
We use Qwen3 backbones at 0.6B, 1.7B, 4B, 8B, and 14B.
\\
\textbf{Training data:}
For domain-continual training, we use \cite{mmlu_pro_cot_train_labeled} and form a 14-task sequence with 1k samples per sub-domain.
For capability-continual training, we use 30k math samples from \cite{nemotron_post_training_dataset_v1} and 30k code samples from \cite{zheng2024opencodeinterpreter}.
More setup details are provided in Appendix~\ref{app:exp_setup_details} and \ref{app:evaluation_prompt}.
\\
\textbf{Baselines:}
Our main baselines are data-free update-integration methods, including Localize-and-Stitch \cite{he2024localize}, AIMMerging \cite{feng2025aimmerging}, and OPCM \cite{tang2025merging}.
Seq.\ SFT, EWC \cite{kirkpatrick2017overcoming}, and FOREVER \cite{feng2026forever} are reported as reference continual-training pipelines because they use additional regularization or replay.
Non-continual merging like TA \cite{ilharcoediting}, TIES \cite{yadav2023ties}, DARE \cite{yu2024language} are also reported in Appendix~\ref{app:non_continual_merge}.
\\
\textbf{Benchmarks:}
Domain-continual performance is evaluated on the 14 MMLU-Pro sub-categories \cite{wang2024mmlu}.
Capability-continual performance is evaluated on GSM8K \cite{cobbe2021training}, MATH-500 \cite{hendrycks2measuring}, MBPP \cite{austin2021program}, HumanEval \cite{chen2021evaluating}, GPQA-Diamond \cite{reingpqa}, and MMLU-Pro \cite{wang2024mmlu}.
\\
\textbf{Remark 1:}
All reported performance scores are averaged over five independent evaluation runs.
\\
\textbf{Remark 2:}
The evaluation code we employ strictly adheres to the Qwen3 Technical Report \cite{yang2025qwen3}

\definecolor{rcwmBlue}{HTML}{DDEEFF}
\definecolor{mergeBlue}{HTML}{F1F7FD}
\definecolor{refGray}{HTML}{F7F7F7}
\definecolor{mtlGold}{HTML}{FFF4D6}
\definecolor{sizeBand}{HTML}{EAF2FA}
\definecolor{ruleNavy}{HTML}{244B73}
\providecommand{\best}[1]{\textbf{#1}}
\providecommand{\ub}[1]{\underline{#1}}

\begin{table*}[t]
\centering
\caption{Domain-continual MMLU-Pro results on Qwen3-1.7B, 8B, and 14B. Scores are accuracies (\%). Underlined MTL is a joint-training upper-bound reference; bold marks the best non-MTL result in each block. Data-free update-integration methods are shaded in blue.}
\label{tab:mmlu_domain_main}
\scriptsize
\setlength{\tabcolsep}{2.15pt}
\renewcommand{\arraystretch}{1.06}

\arrayrulecolor{ruleNavy}
\resizebox{\textwidth}{!}{%
\begin{tabular}{lrrrrrrrrrrrrrrr}
\toprule[0.9pt]
Method & Overall & Bio & Bus & Chem & CS & Econ & Eng & Health & Hist & Law & Math & Other & Phil & Phys & Psych \\
\midrule
\rowcolor{sizeBand}
\multicolumn{16}{l}{\emph{Qwen3-1.7B}} \\
\rowcolor{mtlGold} MTL & \ub{44.4} & \ub{65.1} & \ub{51.6} & \ub{43.0} & \ub{47.3} & \ub{52.3} & \ub{37.1} & \ub{41.9} & \ub{30.2} & \ub{23.3} & \ub{53.2} & \ub{35.8} & \ub{39.5} & \ub{46.4} & \ub{52.9} \\
\rowcolor{refGray} Seq.\ SFT & 36.8 & 55.6 & 39.7 & 36.3 & 37.6 & 48.0 & 29.0 & 31.9 & 24.9 & 17.7 & 47.1 & 31.2 & 32.1 & 36.6 & 47.4 \\
\rowcolor{refGray} EWC & 40.0 & 58.2 & 44.6 & 39.6 & 40.0 & 48.6 & 32.3 & 35.8 & 24.1 & 17.7 & 55.4 & 31.7 & 31.5 & 42.9 & 46.6 \\
\rowcolor{refGray} FOREVER & 38.5 & 58.2 & 42.2 & 37.6 & 41.2 & 46.3 & 28.0 & 35.4 & 27.3 & 16.2 & 50.8 & 33.1 & 30.9 & 40.5 & 47.4 \\
\cmidrule(lr){1-16}
\rowcolor{mergeBlue} L\&S & 41.1 & 61.4 & 48.2 & 39.9 & 44.0 & 48.8 & 34.2 & 38.8 & 27.4 & 20.6 & 49.8 & 32.9 & 36.4 & 43.2 & 49.5 \\
\rowcolor{mergeBlue} AIMMerging & 41.8 & 60.7 & 49.7 & 41.4 & 45.9 & 50.8 & \best{36.6} & 36.9 & 24.4 & 18.4 & \best{56.4} & 34.4 & 32.3 & 42.3 & 47.5 \\
\rowcolor{mergeBlue} OPCM & 41.7 & 62.7 & 49.0 & 40.4 & 44.8 & 49.7 & 34.5 & 39.3 & 27.5 & 20.5 & 50.7 & 33.1 & 36.8 & 43.8 & 50.4 \\
\rowcolor{rcwmBlue} GCWM & \best{43.5} & \best{64.9} & \best{51.0} & \best{42.2} & \best{46.6} & \best{51.7} & 36.1 & \best{41.1} & \best{29.0} & \best{21.9} & 52.7 & \best{34.8} & \best{38.6} & \best{45.7} & \best{52.3} \\
\midrule[0.55pt]

\rowcolor{sizeBand}
\multicolumn{16}{l}{\emph{Qwen3-8B}} \\
\rowcolor{mtlGold} MTL & \ub{65.3} & \ub{83.3} & \ub{70.6} & \ub{65.9} & \ub{66.3} & \ub{74.4} & \ub{54.3} & \ub{65.4} & \ub{58.0} & \ub{40.2} & \ub{76.1} & \ub{58.2} & \ub{60.7} & \ub{67.5} & \ub{73.9} \\
\rowcolor{refGray} Seq.\ SFT & 55.2 & 75.2 & 59.2 & 53.7 & 51.7 & 67.1 & 48.4 & 54.4 & 49.9 & 27.3 & 59.1 & 51.5 & 51.7 & 57.0 & 67.2 \\
\rowcolor{refGray} EWC & 60.4 & 78.8 & 66.2 & 63.0 & 61.0 & 71.1 & 51.8 & 59.4 & 49.9 & 29.9 & 72.3 & 52.9 & 55.3 & 63.2 & 68.2 \\
\rowcolor{refGray} FOREVER & 59.6 & 79.4 & 63.5 & 59.6 & 61.5 & 69.8 & 48.9 & 62.0 & 50.4 & 29.9 & 71.3 & 54.6 & 51.9 & 62.5 & 68.5 \\
\cmidrule(lr){1-16}
\rowcolor{mergeBlue} L\&S & 62.4 & 79.4 & 68.3 & 67.0 & \best{65.4} & 70.5 & 52.6 & 61.5 & 51.7 & 32.3 & 77.3 & 55.5 & 52.1 & 66.0 & 67.4 \\
\rowcolor{mergeBlue} AIMMerging & 62.9 & 78.1 & \best{69.1} & \best{67.9} & 65.1 & 71.3 & \best{53.0} & 62.4 & 52.0 & 32.1 & \best{78.7} & 55.5 & 51.9 & \best{67.4} & 67.9 \\
\rowcolor{mergeBlue} OPCM & 61.9 & 78.8 & 66.8 & 62.4 & 62.8 & 70.5 & 51.4 & 61.9 & 54.9 & 38.1 & 72.0 & 55.1 & 57.5 & 63.9 & 70.0 \\
\rowcolor{rcwmBlue} GCWM & \best{63.7} & \best{81.4} & 68.9 & 64.2 & 64.7 & \best{72.7} & 52.7 & \best{63.7} & \best{56.4} & \best{38.8} & 74.3 & \best{56.6} & \best{59.1} & 65.8 & \best{72.2} \\
\midrule[0.55pt]

\rowcolor{sizeBand}
\multicolumn{16}{l}{\emph{Qwen3-14B}} \\
\rowcolor{mtlGold} MTL & \ub{68.6} & \ub{86.2} & \ub{74.0} & \ub{72.4} & \ub{67.8} & \ub{77.0} & \ub{56.8} & \ub{67.3} & \ub{61.9} & \ub{39.9} & \ub{79.4} & \ub{64.6} & \ub{61.7} & \ub{72.2} & \ub{77.6} \\
\rowcolor{refGray} Seq.\ SFT & 60.4 & 79.6 & 63.4 & 59.5 & 63.4 & 73.3 & 48.1 & 61.9 & 53.8 & 34.4 & 68.6 & 58.2 & 56.1 & 58.6 & 72.6 \\
\rowcolor{refGray} EWC & 65.3 & 84.2 & 70.0 & 67.0 & 63.4 & 74.1 & 52.6 & 66.3 & 59.3 & 33.5 & 80.9 & 61.2 & 55.9 & 69.1 & 71.7 \\
\rowcolor{refGray} FOREVER & 66.5 & \best{84.4} & 72.1 & 67.1 & 69.5 & 74.6 & \best{56.4} & \best{66.6} & 57.5 & 35.4 & 80.5 & 61.8 & 58.7 & 69.7 & 74.3 \\
\cmidrule(lr){1-16}
\rowcolor{mergeBlue} L\&S & 65.6 & 82.7 & 70.8 & 69.2 & 64.8 & 73.7 & 54.1 & 64.3 & 59.1 & 37.7 & 76.0 & 61.7 & 58.9 & 69.1 & 74.3 \\
\rowcolor{mergeBlue} AIMMerging & 66.4 & 83.9 & 71.7 & 70.1 & 65.6 & 74.7 & 54.6 & 65.0 & 59.7 & 37.7 & 77.1 & 62.4 & 59.5 & 70.0 & \best{75.3} \\
\rowcolor{mergeBlue} OPCM & 66.6 & 83.7 & 71.8 & 70.2 & 65.8 & 74.7 & 55.1 & 65.3 & \best{60.1} & \best{38.7} & 77.0 & \best{62.7} & \best{59.9} & 70.1 & 75.3 \\
\rowcolor{rcwmBlue} GCWM & \best{67.8} & 83.8 & \best{73.1} & \best{72.3} & \best{72.7} & \best{76.2} & 55.8 & 66.1 & 59.6 & 36.1 & \best{83.4} & 61.5 & 59.5 & \best{73.3} & 72.3 \\
\bottomrule[0.9pt]
\end{tabular}%
}
\arrayrulecolor{black}
\vspace{-1.5em}
\end{table*}

\vspace{-0.4em}
\subsection{Domain-Continual Post-Training}
\label{sec:domain_continual}
\vspace{-0.2em}

We first evaluate domain-continual post-training on the 14-domain MMLU-Pro sequence.
This setting tests whether GCWM can integrate domain-specific updates without accessing training data during the merge stage.
Table~\ref{tab:mmlu_domain_main} reports full results on Qwen3-1.7B, 8B, and 14B.
MTL is included as a joint-training upper-bound reference, while the main comparison is among data-free update-integration methods.
GCWM gives the strongest non-MTL overall performance at all three scales.
It improves over the best data-free baseline by \(+1.61\), \(+0.74\), and \(+1.23\) points on Qwen3-1.7B, 8B, and 14B, respectively.
The gains are broad rather than driven by a single domain: GCWM improves over AIMMerging on 12/14 domains at 1.7B, 10/14 domains at 8B, and 9/14 domains at 14B.
These results support the role of geometry conflict as a practical control signal for data-free continual update integration.
Complete results for all model scales are reported in Appendix~\ref{app:domain_continual_results}.

\vspace{-0.4em}
\subsection{Capability-Continual Post-Training}
\label{sec:capability_continual}
\vspace{-0.2em}
We next evaluate capability-continual post-training with sequential math and code updates.
Table~\ref{tab:capability_main} reports the performance of Qwen3-1.7B and 14B across three key domains: knowledge (GPQA-Diamond, MMLU-Pro), math (GSM8K, MATH-500), and code (HumanEval, MBPP).
At 1.7B, GCWM gives the best data-free average (62.6), improving over the strongest data-free baseline (OPCM, 56.8) by +5.78 points and leading on all six benchmarks.
At 14B, GCWM remains the strongest data-free method on average (74.3 vs. 72.9 for OPCM) and leads on GPQA-Diamond, GSM8K, HumanEval, and MMLU-Pro.
FOREVER can be higher in some settings because it revisits data through replay and additional sequential optimization; we include it as a replay-based reference, while the primary comparison is among data-free update-integration methods.
These results show that geometry-conflict-controlled integration extends beyond domain transfer to heterogeneous capability updates.
Full five-scale results are provided in Appendix~\ref{app:capability_continual_results}.

\begin{table*}[t]
\centering
\caption{Capability-continual results on Qwen3-1.7B and 14B. Scores are accuracies or pass@1 (\%). MTL is a joint-training reference; Seq.\ SFT, EWC, and FOREVER are training-pipeline references. Bold marks the best data-free method.}
\label{tab:capability_main}
\scriptsize
\setlength{\tabcolsep}{3.5pt}
\renewcommand{\arraystretch}{1.03}
\arrayrulecolor{ruleNavy}
\resizebox{\textwidth}{!}{%
\begin{tabular}{lrrrrrrrrrrrrrr}
\toprule[0.9pt]
Method & \multicolumn{2}{c}{Avg.} & \multicolumn{2}{c}{GPQA-D.} & \multicolumn{2}{c}{GSM8K} & \multicolumn{2}{c}{HumanEval} & \multicolumn{2}{c}{MATH-500} & \multicolumn{2}{c}{MBPP} & \multicolumn{2}{c}{MMLU-Pro} \\
\cmidrule(lr){2-3} \cmidrule(lr){4-5} \cmidrule(lr){6-7} \cmidrule(lr){8-9} \cmidrule(lr){10-11} \cmidrule(lr){12-13} \cmidrule(lr){14-15}
 & 1.7B & 14B & 1.7B & 14B & 1.7B & 14B & 1.7B & 14B & 1.7B & 14B & 1.7B & 14B & 1.7B & 14B \\
\midrule
\rowcolor{mtlGold} MTL & \ub{57.1} & \ub{74.6} & \ub{26.7} & \ub{33.3} & \ub{76.1} & \ub{92.1} & \ub{61.6} & \ub{84.2} & \ub{63.6} & \ub{87.8} & \ub{56.8} & \ub{80.5} & \ub{57.5} & \ub{69.8} \\
\rowcolor{refGray} Seq.\ SFT & 51.9 & 70.4 & 18.2 & 43.4 & 70.6 & 95.8 & 64.0 & 86.6 & 64.2 & 66.2 & 59.1 & 63.4 & 35.3 & 67.2 \\
\rowcolor{refGray} EWC & 54.7 & 73.5 & 24.8 & 43.4 & 75.7 & 95.4 & 61.6 & 86.0 & 67.8 & 68.0 & 57.6 & 78.6 & 40.5 & 69.4 \\
\rowcolor{refGray} FOREVER & 58.3 & 75.9 & 27.3 & 54.0 & 76.7 & 96.4 & 62.2 & 87.2 & 69.6 & 69.2 & 66.9 & 75.9 & 47.3 & 72.8 \\
\rowcolor{mergeBlue} L\&S & 52.4 & 71.3 & 21.3 & 38.4 & 71.0 & 94.1 & 62.8 & 83.7 & 57.5 & 76.5 & 58.4 & 78.5 & 43.5 & 56.8 \\
\rowcolor{mergeBlue} AIMMerging & 53.4 & 72.2 & 21.7 & 38.8 & 72.4 & 95.2 & 64.0 & 84.7 & 58.6 & 77.4 & 59.5 & \best{79.4} & 44.3 & 57.5 \\
\rowcolor{mergeBlue} OPCM & 56.8 & 72.9 & 23.2 & 38.0 & 73.0 & 94.5 & 65.2 & 80.5 & \best{67.6} & \best{79.7} & 59.9 & 78.6 & 51.9 & 66.3 \\
\rowcolor{rcwmBlue} GCWM & \best{58.3} & \best{74.3} & \best{26.3} & \best{39.9} & \best{79.0} & \best{95.8} & \best{67.4} & \best{86.6} & 63.4 & 78.2 & \best{61.5} & 76.7 & \best{52.0} & \best{68.8} \\
\bottomrule[0.9pt]
\end{tabular}%
}

\arrayrulecolor{black}
\vspace{-2em}
\end{table*}

\vspace{-0.5em}
\subsection{Ablations and Analysis}
\label{sec:ablations}
\vspace{-0.5em}

\begin{wrapfigure}{r}{0.5\textwidth}
  \centering
  \vspace{-3em}
  \includegraphics[width=0.99\linewidth]{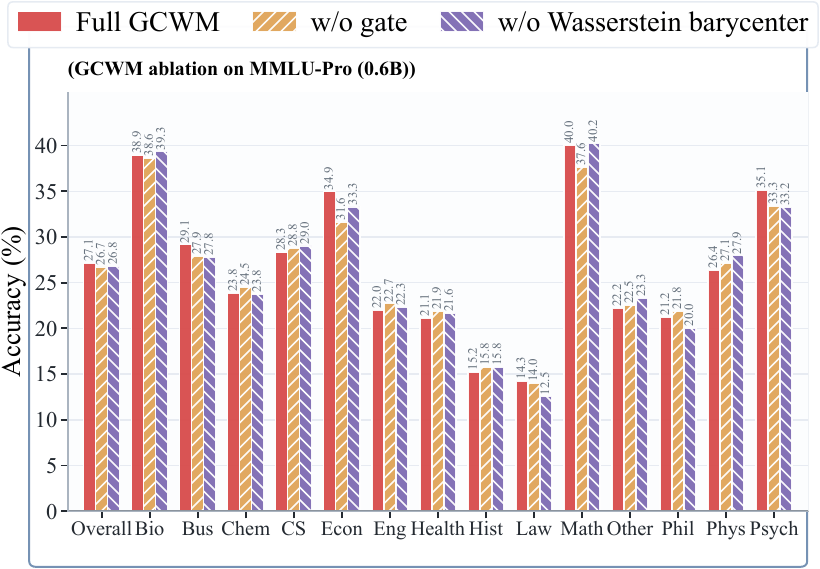}
  \caption{GCWM ablation on MMLU-Pro.}
  \label{fig:rcwm_ablation}
  \vspace{-1.0em}
\end{wrapfigure}

We ablate two merge-time components of GCWM: the conflict gate and the shared Wasserstein metric. 
All variants use the same Qwen3-0.6B domain-continual task experts and evaluation protocol, differing only in the integration rule. 
The \emph{w/o gate} variant removes conflict-conditioned gating and applies the geometry-aware branch uniformly, while \emph{w/o Wasserstein barycenter} replaces the shared Wasserstein barycenter with a mean covariance metric. 
Fig.~\ref{fig:rcwm_ablation} shows that full GCWM achieves the best aggregate score, improving overall accuracy from \(26.7/26.8\%\) to \(27.1\%\). 
The domain breakdown shows a trade-off rather than uniform dominance: removing the gate notably hurts economics, math, and psychology, whereas replacing the Wasserstein metric weakens business, law, and psychology. 
These results support the gate as a layer-wise control signal and the Wasserstein metric as a shared geometry for data-free update integration. 
Additional ablations on Qwen3-1.7B and 8B capability-continual settings, together with detailed breakdowns, are reported in Appendix~\ref{app:ablations}.

\vspace{-0.4em}
\textbf{Remark 3:} We also profile GCWM runtime and memory on Qwen3-8B and 14B in Appendix~\ref{app:runtime_memory_profile}, since these costs are the practical bottleneck for data-free update integration.
\\
\textbf{Remark 4:} Appendix~\ref{app:hyperparameter_sensitivity} reports Qwen3-8B rank and gate-parameter sensitivity.
\vspace{-0.5em}

\vspace{-0.3em}
\section{Conclusion}
\vspace{-0.4em}

We studied LLM continual post-training through task geometry, analyzing update norm, subspace alignment, gradient conflict, and geometry conflict across model scales and continual strategies. 
Our main finding is that forgetting is a state-relative update-integration failure: harmful steps occur when task-induced covariance geometries become incompatible with the geometry of the evolving model state. This explains why raw drift and isolated pairwise compatibility are insufficient, and why geometry conflict serves as both an explanatory signal for forgetting and a control signal for sequential update integration. 
\textbf{G}eometry-\textbf{C}onflict \textbf{W}asserstein \textbf{M}erging (GCWM) operationalizes this insight by constructing a Wasserstein shared metric from task-induced covariance geometry and gating data-free update integration by geometry conflict. Across domain-continual and capability-continual settings, GCWM improves retention and final performance over data-free baselines without replay data.

% \begin{ack}
% Use unnumbered first level headings for the acknowledgments. All acknowledgments

% Do {\bf not} include this section in the anonymized submission, only in the final paper. You can use the \texttt{ack} environment provided in the style file to automatically hide this section in the anonymized submission.
% \end{ack}

\newpage
\bibliographystyle{unsrt}
\bibliography{ref}

%%%%%%%%%%%%%%%%%%%%%%%%%%%%%%%%%%%%%%%%%%%%%%%%%%%%%%%%%%%%
\newpage
\appendix

\section*{Limitations}
\label{app:limitations}
Our analysis and experiments focus on Qwen3-scale open LLMs and on domain and capability continual post-training tasks built from public reasoning, knowledge, math, and code benchmarks.
Although the state-relative geometry signal is consistent across scales and methods, it should be viewed as an explanatory and control signal rather than a proof of causal necessity for all forms of forgetting.
GCWM is data-free at merge time and is therefore most relevant when historical data are unavailable or replay is undesirable; replay-based training can still be stronger when it is allowed to repeatedly revisit past data.
The method also assumes access to task-specific updates or checkpoints and incurs additional CPU cost for geometry construction and Wasserstein metric computation, though this cost is offline and does not affect inference.

\section{Discussion and Broader impacts}
\label{app:availability_compute_impact}

% Code, configuration files, analysis scripts, and processed result tables are available at \url{https://anonymous.4open.science/r/GCWM-D8B4}.
% Code, configuration files, analysis scripts, and processed result tables are available at \url{https://github.com/wyy-code/GCWM}.

GCWM may make continual LLM adaptation more practical by reducing reliance on replay data and by providing diagnostics for harmful update integration.
At the same time, easier post-training can also be used to adapt models toward unsafe or misleading behaviors if deployed without appropriate safety evaluation.

GCWM operates in parameter-update space and is designed for data-free continual update integration.
This differs from recent distribution- \cite{wang2025infigfusion} or preference-level \cite{gu2025infifpo} fusion methods, which primarily target cross-model knowledge transfer rather than state-relative continual forgetting.

This work does not introduce new safety guarantees; model developers should combine compatibility-controlled merging with standard data governance, red-teaming, and downstream safety checks before deployment.

\section{Implementation Details of Algorithm}
\label{app:algorithm}

This section summarizes the implementation-aligned version of GCWM. All task updates are constructed relative to the same pretrained model,
\[
\Delta_i = \theta_i - \theta_{\mathrm{pre}}.
\]
At continual step \(t\), GCWM forms an active set of updates according to a memory policy: either a history-aware policy that retains previous task updates, or an anchor-based policy that merges the current task against the previously merged state. For each target layer, GCWM then computes a shared task geometry, estimates geometry conflict, constructs a shared Wasserstein metric, performs geometry-aware merging, and applies only the incremental change relative to the previous merged state.

For clarity, the main text presents a projected formulation of whitening and recoloring. The implementation uses the corresponding regularized full-space transform, which preserves the projected geometry while regularizing the orthogonal complement.

The full continual update and merge process is summarized in Algorithm~\ref{alg:rcwm_impl}.

\section{Proof of Theorem~\ref{thm:conflict_controlled_integration}}
\label{app:proof_conflict_controlled_integration}

We prove a relative form of Theorem~\ref{thm:conflict_controlled_integration}, comparing GCWM against the plain merge at the same continual step. This form directly captures the additional effect introduced by the geometry-aware correction. For simplicity, we state the result for the \emph{projected shared-metric branch} of GCWM, which is the main analytical form studied in the paper; the dense-local implementation variant is treated as an efficiency-oriented special case.

\paragraph{Notation.}
At continual step \(t\), define
\[
\Theta_{\mathrm{plain},t}
:=
\theta_{t-1}+\eta_t \Delta_{\mathrm{plain},t},
\qquad
\Theta_{\mathrm{GCWM},t}
:=
\theta_{t-1}+\eta_t \Delta_{\mathrm{merge},t},
\]

\begin{algorithm}[H]
\caption{Implementation-aligned GCWM at continual step \(t\)}
\label{alg:rcwm_impl}
\footnotesize
\DontPrintSemicolon
\SetAlgoLined
\SetInd{0.3em}{0.8em}
\KwIn{
pretrained parameters \(\theta_{\mathrm{pre}}\), expert models \(\{\theta_i\}\), previous merged state \(\Delta_{\mathrm{merge},t-1}\), memory policy, task weights \(\{\omega_i\}\), hyperparameters \((r,\rho,\lambda,\tau,\kappa,\alpha_{\min},\alpha_{\max},\eta_t)\).
}
\KwOut{updated model parameters \(\theta_t\).}

Construct task vectors relative to the base model: \(\Delta_i=\theta_i-\theta_{\mathrm{pre}}\)\;

Build the active set \(\mathcal{A}_t=\{\Delta_i\}_{i=1}^m\) according to the memory policy:
history-aware mode uses historical task vectors (or a recent subset) plus the current task; anchor-based mode uses \(\Delta_{\mathrm{merge},t-1}\) and the current task\;

\For{each target layer \(\ell\)}{
Stack active layer updates \(V^{(\ell)}=\{\Delta_i^{(\ell)}\}_{i=1}^m\)\;

\textbf{Shared metric construction:}\;
\Indp
\If{dense local metric is used at layer \(\ell\)}{
Compute \(\Sigma_i^{(\ell)}=(\Delta_i^{(\ell)})^\top\Delta_i^{(\ell)}+\lambda I\), set \(\Sigma_{\mathrm{shared}}^{(\ell)}=\sum_i \omega_i \Sigma_i^{(\ell)}\), and compute the layer conflict score from normalized pairwise Bures distances between \(\{\Sigma_i^{(\ell)}\}\)\;
}
\Else{
For each \(\Delta_i^{(\ell)}\), compute a truncated right SVD and retain \((V_i^{(\ell)},S_i^{(\ell)})\)\;
Form \(Q^{(\ell)}=\mathrm{orth}\!\left([V_1^{(\ell)},\dots,V_m^{(\ell)}]\right)\)\;
Construct \(B_i^{(\ell)}=(Q^{(\ell)})^\top V_i^{(\ell)} \mathrm{diag}\!\big((S_i^{(\ell)})^2\big) (V_i^{(\ell)})^\top Q^{(\ell)}+\lambda I\)\;
\If{Wasserstein barycenter metric is used}{
compute \(B_{\mathrm{shared}}^{(\ell)}\) as the Gaussian Wasserstein barycenter of \(\{B_i^{(\ell)}\}\)\;
}
\Else{
compute \(B_{\mathrm{shared}}^{(\ell)}=\sum_i\omega_i B_i^{(\ell)}\)\;
}
Compute the layer conflict score from normalized pairwise Bures distances between \(\{B_i^{(\ell)}\}\)\;
}
\Indm

\textbf{Gate computation:}\;
\[
\alpha^{(\ell)}
=
\alpha_{\min}
+
(\alpha_{\max}-\alpha_{\min})
\,\sigma\!\big(\kappa(g^{(\ell)}-\tau)\big).
\]

\textbf{Merge branch selection:}\;
\Indp
\If{\(\alpha^{(\ell)}\) is below the skip threshold}{
Set \(\Delta_{\mathrm{merge},t}^{(\ell)}=\Delta_{\mathrm{plain}}^{(\ell)}\), where \(\Delta_{\mathrm{plain}}^{(\ell)}\) is the plain merge operator on \(V^{(\ell)}\)\;
}
\Else{
Whiten the active updates under the shared metric:\;
\Indp
dense branch: \(\tilde{\Delta}_i^{(\ell)}=\Delta_i^{(\ell)}(\Sigma_{\mathrm{shared}}^{(\ell)})^{-1/2}\)\;
projected branch: \(\tilde{\Delta}_i^{(\ell)}=\mathcal{T}_{Q^{(\ell)},B_{\mathrm{shared}}^{(\ell)},\lambda}^{-}(\Delta_i^{(\ell)})\)\;
\Indm

Apply merge operator to \(\{\tilde{\Delta}_i^{(\ell)}\}\), obtaining \(\tilde{\Delta}_{\mathrm{geo}}^{(\ell)}\)\;

Recolor the geometry-aware merge:\;
\Indp
dense branch: \(\Delta_{\mathrm{geo}}^{(\ell)}=\tilde{\Delta}_{\mathrm{geo}}^{(\ell)}(\Sigma_{\mathrm{shared}}^{(\ell)})^{1/2}\)\;
projected branch: \(\Delta_{\mathrm{geo}}^{(\ell)}=\mathcal{T}_{Q^{(\ell)},B_{\mathrm{shared}}^{(\ell)},\lambda}^{+}(\tilde{\Delta}_{\mathrm{geo}}^{(\ell)})\)\;
\Indm

\If{\(\alpha^{(\ell)}\) is above \(1\) minus the skip threshold}{
Set \(\Delta_{\mathrm{merge},t}^{(\ell)}=\Delta_{\mathrm{geo}}^{(\ell)}\)\;
}
\Else{
Compute \(\Delta_{\mathrm{plain}}^{(\ell)}\) and blend
\[
\Delta_{\mathrm{merge},t}^{(\ell)}
=
\alpha^{(\ell)}\Delta_{\mathrm{geo}}^{(\ell)}
+
(1-\alpha^{(\ell)})\Delta_{\mathrm{plain}}^{(\ell)}.
\]
}
}
\Indm
}

Set \(\Delta_{\mathrm{inc},t}^{(\ell)}=\Delta_{\mathrm{merge},t}^{(\ell)}-\Delta_{\mathrm{merge},t-1}^{(\ell)}\) with \(\Delta_{\mathrm{merge},0}^{(\ell)}=0\), and update \(\theta_t^{(\ell)}=\theta_{t-1}^{(\ell)}+\eta_t\,\Delta_{\mathrm{inc},t}^{(\ell)}\)\;

\Return{\(\theta_t\)}\;
\end{algorithm}

where \(\eta_t>0\) is the outer step coefficient, \(\Delta_{\mathrm{plain},t}\) is the ungated merge, and \(\Delta_{\mathrm{merge},t}\) is the GCWM update. For each target layer \(\ell\), define the GCWM correction relative to the plain merge as
\begin{equation}
\delta_t^{(\ell)}
:=
\Delta_{\mathrm{merge},t}^{(\ell)}-\Delta_{\mathrm{plain},t}^{(\ell)}.
\label{eq:delta_def_appendix_final}
\end{equation}
By the GCWM blending rule in Eq.~\eqref{eq:final-merge},
\begin{equation}
\delta_t^{(\ell)}
=
\alpha_t^{(\ell)}
\Bigl(
\Delta_{\mathrm{geo},t}^{(\ell)}-\Delta_{\mathrm{plain},t}^{(\ell)}
\Bigr),
\label{eq:delta_gate_identity_appendix_final}
\end{equation}
where \(\alpha_t^{(\ell)}\in[\alpha_{\min},\alpha_{\max}]\) is the layer-wise gate.

For a layer-wise collection \(X=\{X^{(\ell)}\}_{\ell=1}^L\), define the Frobenius inner product
\[
\langle X,Y\rangle
:=
\sum_{\ell=1}^L \mathrm{tr}\!\big((X^{(\ell)})^\top Y^{(\ell)}\big).
\]
For a positive semidefinite matrix \(M\), define the metric-induced norm
\begin{equation}
\|X\|_{M}^2
:=
\mathrm{tr}(X M X^\top).
\label{eq:metric_norm_appendix_final}
\end{equation}

\paragraph{Implementation-aligned shared metric.}
For each layer \(\ell\) and step \(t\), let \(Q_t^{(\ell)}\in\mathbb{R}^{d_{\mathrm{in}}\times r_t^{(\ell)}}\) be the shared basis and \(\bar B_t^{(\ell)}\in\mathbb{R}^{r_t^{(\ell)}\times r_t^{(\ell)}}\) the projected shared metric. To align the theorem with the implementation, define the corresponding full-space regularized metric
\begin{equation}
\widetilde B_t^{(\ell)}
:=
Q_t^{(\ell)} \bar B_t^{(\ell)} \bigl(Q_t^{(\ell)}\bigr)^\top
+
\lambda \Bigl(I - Q_t^{(\ell)}\bigl(Q_t^{(\ell)}\bigr)^\top\Bigr),
\label{eq:full_space_metric_appendix}
\end{equation}
where \(\lambda>0\) is the regularization parameter and \(I\) is the identity matrix of size \(d_{\mathrm{in}}\times d_{\mathrm{in}}\). The norm used in the proof is
\[
\|X\|_{\widetilde B_t^{(\ell)}}^2
=
\mathrm{tr}\!\bigl(X \widetilde B_t^{(\ell)} X^\top\bigr).
\]
This is the metric induced by the full-space whitening/recoloring operators used in the implementation.

\paragraph{Center mismatch quantity.}
Let \(\{B_{i,t}^{(\ell)}\}_{i=1}^{m_t}\) denote the projected task geometries at layer \(\ell\) and step \(t\), and let \(\bar B_t^{(\ell)}\) be their shared Wasserstein barycenter. Define the weighted center mismatch
\begin{equation}
r_t^{(\ell)}
:=
\sum_{i=1}^{m_t}\omega_{i,t}\,
d_{\mathrm B}^2\!\bigl(B_{i,t}^{(\ell)},\bar B_t^{(\ell)}\bigr),
\label{eq:center_mismatch_appendix_final}
\end{equation}
where \(\omega_{i,t}\ge 0\), \(\sum_i \omega_{i,t}=1\), and \(d_{\mathrm B}(\cdot,\cdot)\) is the Bures distance from \cite{bhatia2019bures} This quantity measures how far the active task geometries lie from the shared center.

We use the following noraml assumptions.

\begin{assumption}[Local smoothness]
\label{ass:local_smoothness_appendix_final}
For every previously acquired task \(u\), the loss \(\mathcal L_u\) is twice continuously differentiable in a neighborhood of the line segment
\[
\bigl\{
\Theta_{\mathrm{plain},t}+s(\Theta_{\mathrm{GCWM},t}-\Theta_{\mathrm{plain},t})
: s\in[0,1]
\bigr\}.
\]
\end{assumption}

\begin{assumption}[Projected-geometry adequacy]
\label{ass:geometry_adequacy_appendix_final}
For every previously acquired task \(u\), target layer \(\ell\), and step \(t\), there exists a nonnegative constant \(a_{u,t}^{(\ell)}\) such that
\begin{equation}
\bigl|
\langle
\nabla_{\ell}\mathcal L_u(\Theta_{\mathrm{plain},t}),
\delta_t^{(\ell)}
\rangle
\bigr|
\;\le\;
a_{u,t}^{(\ell)}\, r_t^{(\ell)}.
\label{eq:geometry_adequacy_appendix_final}
\end{equation}
This assumption formalizes that the projected geometry captures the dominant first-order directions relevant to cross-task interference.
\end{assumption}

\begin{assumption}[Layer-separable metric curvature bound]
\label{ass:metric_curvature_appendix_final}
For every previously acquired task \(u\), step \(t\), and \(s\in[0,1]\), there exist nonnegative constants \(d_{u,t}^{(\ell)}\) such that
\begin{equation}
\Bigl\langle
\delta_t,\,
\nabla^2\mathcal L_u\!\bigl(\Theta_{\mathrm{plain},t}+s\eta_t\delta_t\bigr)\,
\delta_t
\Bigr\rangle
\;\le\;
\sum_{\ell=1}^L
d_{u,t}^{(\ell)}
\,
\|\delta_t^{(\ell)}\|_{\widetilde B_t^{(\ell)}}^2.
\label{eq:metric_curvature_appendix_final}
\end{equation}
This assumption allows cross-layer couplings in the full Hessian to be absorbed into layer-wise constants while measuring displacement in the implementation-aligned full-space metric \(\widetilde B_t^{(\ell)}\).
\end{assumption}

The next lemma shows that the center mismatch \(r_t^{(\ell)}\) can be controlled by the pairwise geometry conflict \(g_t^{(\ell)}\) used by GCWM.

\begin{lemma}[Center mismatch is controlled by pairwise geometry conflict]
\label{lem:center_mismatch_to_conflict}
Fix a layer \(\ell\) and step \(t\), and assume \(m_t \ge 2\). Assume that the pairwise weights in Eq.~\eqref{eq:layer-geometry-conflict} are chosen as
\[
w_{ij,t}
=
\frac{\omega_{i,t}\omega_{j,t}}{Z_t},
\qquad
Z_t:=\sum_{1\le i<j\le m_t}\omega_{i,t}\omega_{j,t}.
\]
Let
\begin{equation}
M_t^{(\ell)}
:=
\max_{1\le i<j\le m_t}
\Bigl(
\mathrm{tr}(B_{i,t}^{(\ell)})
+
\mathrm{tr}(B_{j,t}^{(\ell)})
+\varepsilon
\Bigr).
\label{eq:trace_scale_constant}
\end{equation}
Then
\begin{equation}
r_t^{(\ell)}
\;\le\;
M_t^{(\ell)}\, g_t^{(\ell)}.
\label{eq:center_mismatch_bound}
\end{equation}
When \(m_t=1\), both \(r_t^{(\ell)}\) and \(g_t^{(\ell)}\) vanish, so the bound is trivial.
\end{lemma}

\begin{proof}
For fixed \(\ell,t\), define the Fr\'echet functional
\[
F(B)
:=
\sum_{i=1}^{m_t}\omega_{i,t}\,
d_{\mathrm B}^2(B_{i,t}^{(\ell)},B).
\]
By construction, \(\bar B_t^{(\ell)}\) is a minimizer of \(F\). Therefore, for every \(j\in\{1,\dots,m_t\}\),
\[
F(\bar B_t^{(\ell)})
\le
F(B_{j,t}^{(\ell)}).
\]
Multiplying both sides by \(\omega_{j,t}\) and summing over \(j\) yields
\begin{align}
r_t^{(\ell)}
=
F(\bar B_t^{(\ell)})
&\le
\sum_{j=1}^{m_t}\omega_{j,t}
\sum_{i=1}^{m_t}\omega_{i,t}\,
d_{\mathrm B}^2(B_{i,t}^{(\ell)},B_{j,t}^{(\ell)})
\nonumber\\
&=
\sum_{i=1}^{m_t}\sum_{j=1}^{m_t}
\omega_{i,t}\omega_{j,t}\,
d_{\mathrm B}^2(B_{i,t}^{(\ell)},B_{j,t}^{(\ell)}).
\label{eq:barycenter_variance_bound}
\end{align}
Because the diagonal terms vanish, this becomes
\begin{equation}
r_t^{(\ell)}
\le
2\sum_{1\le i<j\le m_t}
\omega_{i,t}\omega_{j,t}\,
d_{\mathrm B}^2(B_{i,t}^{(\ell)},B_{j,t}^{(\ell)}).
\label{eq:pairwise_bures_sum}
\end{equation}

Next, by the definition of the normalized pairwise conflict in Eq.~\eqref{eq:pairwise-geometry-conflict},
\[
d_{\mathrm B}^2(B_{i,t}^{(\ell)},B_{j,t}^{(\ell)})
=
\gamma_{ij,t}^{(\ell)}
\Bigl(
\mathrm{tr}(B_{i,t}^{(\ell)})+\mathrm{tr}(B_{j,t}^{(\ell)})+\varepsilon
\Bigr)
\le
M_t^{(\ell)}\gamma_{ij,t}^{(\ell)}.
\]
Substituting this bound into Eq.~\eqref{eq:pairwise_bures_sum} gives
\begin{equation}
r_t^{(\ell)}
\le
2M_t^{(\ell)}
\sum_{1\le i<j\le m_t}
\omega_{i,t}\omega_{j,t}\,
\gamma_{ij,t}^{(\ell)}.
\label{eq:pairwise_gamma_bound}
\end{equation}
Using the definition of \(w_{ij,t}\), we obtain
\[
\sum_{1\le i<j\le m_t}
\omega_{i,t}\omega_{j,t}\,
\gamma_{ij,t}^{(\ell)}
=
Z_t \sum_{1\le i<j\le m_t}
w_{ij,t}\,\gamma_{ij,t}^{(\ell)}
=
Z_t\, g_t^{(\ell)}.
\]
Hence
\[
r_t^{(\ell)}
\le
2M_t^{(\ell)} Z_t\, g_t^{(\ell)}.
\]
Finally, since \(\sum_i \omega_{i,t}=1\),
\[
Z_t
=
\sum_{i<j}\omega_{i,t}\omega_{j,t}
=
\frac{1-\sum_i \omega_{i,t}^2}{2}
\le
\frac12.
\]
Therefore
\[
r_t^{(\ell)}
\le
M_t^{(\ell)}\, g_t^{(\ell)},
\]
which proves Eq.~\eqref{eq:center_mismatch_bound}.
\end{proof}

We are now ready to prove the main result.

The case \(m_t=1\) is trivial, since then the pairwise geometry conflict vanishes and GCWM reduces to the plain merge. We therefore state the theorem for \(m_t \ge 2\).

\begin{theoreml}[Conflict-Controlled Integration]
\label{thm:conflict_controlled_integration_appendix_final}
Assume \(m_t \ge 2\) for every analyzed layer \(\ell\). Under Assumptions~\ref{ass:local_smoothness_appendix_final}, \ref{ass:geometry_adequacy_appendix_final}, and \ref{ass:metric_curvature_appendix_final}, and with pairwise weights chosen as in Lemma~\ref{lem:center_mismatch_to_conflict}, the additional loss incurred on a previously acquired task \(u\) by GCWM relative to the plain merge satisfies
\begin{equation}
\mathcal L_u(\Theta_{\mathrm{GCWM},t})
-
\mathcal L_u(\Theta_{\mathrm{plain},t})
\;\le\;
\eta_t \sum_{\ell=1}^L c_{u,t}^{(\ell)}\, g_t^{(\ell)}
\;+\;
\frac{\eta_t^2}{2}
\sum_{\ell=1}^L
d_{u,t}^{(\ell)}
\bigl\|
\Delta_{\mathrm{merge},t}^{(\ell)}-\Delta_{\mathrm{plain},t}^{(\ell)}
\bigr\|_{\widetilde B_t^{(\ell)}}^2,
\label{eq:theorem1_final_appendix_revised}
\end{equation}
where
\[
c_{u,t}^{(\ell)} := a_{u,t}^{(\ell)} M_t^{(\ell)}.
\]
Hence, the additional loss induced by GCWM relative to the plain merge is controlled by shared geometry conflict and gated merge displacement.
\end{theoreml}

\begin{proof}
By Assumption~\ref{ass:local_smoothness_appendix_final}, the exact second-order Taylor formula in integral form yields
\begin{align}
&\mathcal L_u(\Theta_{\mathrm{GCWM},t})
-
\mathcal L_u(\Theta_{\mathrm{plain},t})
=
\Bigl\langle
\nabla \mathcal L_u(\Theta_{\mathrm{plain},t}),
\Theta_{\mathrm{GCWM},t}-\Theta_{\mathrm{plain},t}
\Bigr\rangle
\nonumber\\
&\quad
+
\int_0^1 (1-s)\,
\Bigl\langle
\Theta_{\mathrm{GCWM},t}-\Theta_{\mathrm{plain},t},\,
\nabla^2\mathcal L_u\!\bigl(
\Theta_{\mathrm{plain},t}
+s(\Theta_{\mathrm{GCWM},t}-\Theta_{\mathrm{plain},t})
\bigr)
\nonumber\\
&\qquad\qquad
(\Theta_{\mathrm{GCWM},t}-\Theta_{\mathrm{plain},t})
\Bigr\rangle ds.
\label{eq:taylor_integral_exact_revised}
\end{align}
Since
\[
\Theta_{\mathrm{GCWM},t}-\Theta_{\mathrm{plain},t}
=
\eta_t \delta_t,
\]
Eq.~\eqref{eq:taylor_integral_exact_revised} becomes
\begin{align}
\mathcal L_u(\Theta_{\mathrm{GCWM},t})
-
\mathcal L_u(\Theta_{\mathrm{plain},t})
&=
\eta_t
\langle
\nabla \mathcal L_u(\Theta_{\mathrm{plain},t}),
\delta_t
\rangle
\nonumber\\
&\quad
+
\eta_t^2
\int_0^1 (1-s)\,
\Bigl\langle
\delta_t,\,
\nabla^2\mathcal L_u\!\bigl(\Theta_{\mathrm{plain},t}+s\eta_t\delta_t\bigr)\,
\delta_t
\Bigr\rangle ds.
\label{eq:taylor_substituted_revised}
\end{align}

We first bound the linear term. By layer-wise decomposition,
\begin{equation}
\langle
\nabla \mathcal L_u(\Theta_{\mathrm{plain},t}),
\delta_t
\rangle
=
\sum_{\ell=1}^L
\langle
\nabla_{\ell}\mathcal L_u(\Theta_{\mathrm{plain},t}),
\delta_t^{(\ell)}
\rangle.
\label{eq:linear_layer_decomp_revised}
\end{equation}
Applying Assumption~\ref{ass:geometry_adequacy_appendix_final},
\begin{equation}
\bigl|
\langle
\nabla \mathcal L_u(\Theta_{\mathrm{plain},t}),
\delta_t
\rangle
\bigr|
\le
\sum_{\ell=1}^L
a_{u,t}^{(\ell)}\, r_t^{(\ell)}.
\label{eq:linear_bound_via_r_revised}
\end{equation}
By Lemma~\ref{lem:center_mismatch_to_conflict},
\[
r_t^{(\ell)}
\le
M_t^{(\ell)} g_t^{(\ell)}.
\]
Hence
\begin{equation}
\bigl|
\langle
\nabla \mathcal L_u(\Theta_{\mathrm{plain},t}),
\delta_t
\rangle
\bigr|
\le
\sum_{\ell=1}^L
a_{u,t}^{(\ell)}M_t^{(\ell)} g_t^{(\ell)}
=
\sum_{\ell=1}^L
c_{u,t}^{(\ell)} g_t^{(\ell)}.
\label{eq:linear_bound_via_g_revised}
\end{equation}
Therefore
\begin{equation}
\eta_t
\langle
\nabla \mathcal L_u(\Theta_{\mathrm{plain},t}),
\delta_t
\rangle
\le
\eta_t
\sum_{\ell=1}^L
c_{u,t}^{(\ell)} g_t^{(\ell)}.
\label{eq:linear_term_final_revised}
\end{equation}

We next bound the second-order term. By Assumption~\ref{ass:metric_curvature_appendix_final}, for every \(s\in[0,1]\),
\begin{equation}
\Bigl\langle
\delta_t,\,
\nabla^2\mathcal L_u\!\bigl(\Theta_{\mathrm{plain},t}+s\eta_t\delta_t\bigr)\,
\delta_t
\Bigr\rangle
\le
\sum_{\ell=1}^L
d_{u,t}^{(\ell)}
\,
\|\delta_t^{(\ell)}\|_{\widetilde B_t^{(\ell)}}^2.
\label{eq:quadratic_pointwise_bound_revised}
\end{equation}
Substituting Eq.~\eqref{eq:quadratic_pointwise_bound_revised} into the integral term in Eq.~\eqref{eq:taylor_substituted_revised} yields
\begin{align}
&\eta_t^2
\int_0^1 (1-s)\,
\Bigl\langle
\delta_t,\,
\nabla^2\mathcal L_u\!\bigl(\Theta_{\mathrm{plain},t}+s\eta_t\delta_t\bigr)\,
\delta_t
\Bigr\rangle ds
\nonumber\\
&\le
\eta_t^2
\int_0^1 (1-s)\,ds
\sum_{\ell=1}^L
d_{u,t}^{(\ell)}
\,
\|\delta_t^{(\ell)}\|_{\widetilde B_t^{(\ell)}}^2
\nonumber\\
&=
\frac{\eta_t^2}{2}
\sum_{\ell=1}^L
d_{u,t}^{(\ell)}
\,
\|\delta_t^{(\ell)}\|_{\widetilde B_t^{(\ell)}}^2.
\label{eq:quadratic_integrated_bound_revised}
\end{align}

Combining Eqs.~\eqref{eq:taylor_substituted_revised}, \eqref{eq:linear_term_final_revised}, and \eqref{eq:quadratic_integrated_bound_revised}, we obtain
\begin{equation}
\mathcal L_u(\Theta_{\mathrm{GCWM},t})
-
\mathcal L_u(\Theta_{\mathrm{plain},t})
\le
\eta_t \sum_{\ell=1}^L c_{u,t}^{(\ell)} g_t^{(\ell)}
+
\frac{\eta_t^2}{2}
\sum_{\ell=1}^L
d_{u,t}^{(\ell)}
\,
\|\delta_t^{(\ell)}\|_{\widetilde B_t^{(\ell)}}^2.
\label{eq:combined_bound_revised}
\end{equation}
Finally, substituting
\[
\delta_t^{(\ell)}
=
\Delta_{\mathrm{merge},t}^{(\ell)}-\Delta_{\mathrm{plain},t}^{(\ell)}
\]
from Eq.~\eqref{eq:delta_def_appendix_final} into Eq.~\eqref{eq:combined_bound_revised} gives Eq.~\eqref{eq:theorem1_final_appendix_revised}. This completes the proof.
\end{proof}

\begin{remark}[Gate-scaled displacement]
\label{rem:gate_scaled_displacement_revised}
By Eq.~\eqref{eq:delta_gate_identity_appendix_final},
\begin{equation}
\|\delta_t^{(\ell)}\|_{\widetilde B_t^{(\ell)}}^2
=
\bigl(\alpha_t^{(\ell)}\bigr)^2
\bigl\|
\Delta_{\mathrm{geo},t}^{(\ell)}-\Delta_{\mathrm{plain},t}^{(\ell)}
\bigr\|_{\widetilde B_t^{(\ell)}}^2.
\label{eq:gate_scaled_displacement_revised}
\end{equation}
Thus, the GCWM gate directly scales the second-order displacement term in Theorem~\ref{thm:conflict_controlled_integration_appendix_final}. This is the precise sense in which geometry conflict becomes an actionable control signal.
\end{remark}

% =========================
% Appendix
% =========================

\section{Proof of Proposition~\ref{prop:compatibility_regimes}}
\label{app:proof_compatibility_regimes}

For convenience, we restate the main definitions. At continual step \(t\) and layer \(\ell\), GCWM forms the merged update
\begin{equation}
\Delta_{\mathrm{merge},t}^{(\ell)}
=
\alpha_t^{(\ell)} \Delta_{\mathrm{geo},t}^{(\ell)}
+
\bigl(1-\alpha_t^{(\ell)}\bigr)\Delta_{\mathrm{plain},t}^{(\ell)},
\label{eq:prop1_blend_appendix}
\end{equation}
where \(\Delta_{\mathrm{geo},t}^{(\ell)}\) is the geometry-aware merge, \(\Delta_{\mathrm{plain},t}^{(\ell)}\) is the plain merge, and \(\alpha_t^{(\ell)} \in [\alpha_{\min},\alpha_{\max}]\) is the layer-wise gate. We also define
\begin{equation}
D_t^{(\ell)}
:=
\bigl\|
\Delta_{\mathrm{geo},t}^{(\ell)}-\Delta_{\mathrm{plain},t}^{(\ell)}
\bigr\|_{\widetilde B_t^{(\ell)}},
\label{eq:prop1_D_appendix}
\end{equation}
where \(\widetilde B_t^{(\ell)}\) is the implementation-aligned shared metric defined in Eq.~\eqref{eq:full_space_metric_appendix}.

The gate is given by Eq.~\eqref{eq:conflict-gate}, namely
\begin{equation}
\alpha_t^{(\ell)}
=
\alpha_{\min}
+
(\alpha_{\max}-\alpha_{\min})
\,\sigma\!\bigl(\kappa(g_t^{(\ell)}-\tau)\bigr),
\label{eq:prop1_gate_appendix}
\end{equation}
where \(\sigma(z)=1/(1+e^{-z})\) is the sigmoid function, \(\kappa>0\) is the sharpness parameter, and \(\tau\) is the conflict threshold.

\begin{proof}
We first prove the two norm identities. Subtracting \(\Delta_{\mathrm{plain},t}^{(\ell)}\) from both sides of Eq.~\eqref{eq:prop1_blend_appendix} gives
\begin{align}
\Delta_{\mathrm{merge},t}^{(\ell)}-\Delta_{\mathrm{plain},t}^{(\ell)}
&=
\alpha_t^{(\ell)} \Delta_{\mathrm{geo},t}^{(\ell)}
+
\bigl(1-\alpha_t^{(\ell)}\bigr)\Delta_{\mathrm{plain},t}^{(\ell)}
-
\Delta_{\mathrm{plain},t}^{(\ell)}
\nonumber\\
&=
\alpha_t^{(\ell)}
\Bigl(
\Delta_{\mathrm{geo},t}^{(\ell)}-\Delta_{\mathrm{plain},t}^{(\ell)}
\Bigr).
\label{eq:prop1_plain_diff}
\end{align}
Taking the norm induced by \(\widetilde B_t^{(\ell)}\) and using positive homogeneity,
\begin{align}
\bigl\|
\Delta_{\mathrm{merge},t}^{(\ell)}-\Delta_{\mathrm{plain},t}^{(\ell)}
\bigr\|_{\widetilde B_t^{(\ell)}}
&=
\bigl\|
\alpha_t^{(\ell)}
\bigl(
\Delta_{\mathrm{geo},t}^{(\ell)}-\Delta_{\mathrm{plain},t}^{(\ell)}
\bigr)
\bigr\|_{\widetilde B_t^{(\ell)}}
\nonumber\\
&=
\alpha_t^{(\ell)}
\bigl\|
\Delta_{\mathrm{geo},t}^{(\ell)}-\Delta_{\mathrm{plain},t}^{(\ell)}
\bigr\|_{\widetilde B_t^{(\ell)}}
\nonumber\\
&=
\alpha_t^{(\ell)} D_t^{(\ell)}.
\label{eq:prop1_first_identity}
\end{align}

Similarly, subtracting \(\Delta_{\mathrm{geo},t}^{(\ell)}\) from both sides of Eq.~\eqref{eq:prop1_blend_appendix} gives
\begin{align}
\Delta_{\mathrm{merge},t}^{(\ell)}-\Delta_{\mathrm{geo},t}^{(\ell)}
&=
\alpha_t^{(\ell)} \Delta_{\mathrm{geo},t}^{(\ell)}
+
\bigl(1-\alpha_t^{(\ell)}\bigr)\Delta_{\mathrm{plain},t}^{(\ell)}
-
\Delta_{\mathrm{geo},t}^{(\ell)}
\nonumber\\
&=
\bigl(1-\alpha_t^{(\ell)}\bigr)
\Bigl(
\Delta_{\mathrm{plain},t}^{(\ell)}-\Delta_{\mathrm{geo},t}^{(\ell)}
\Bigr).
\label{eq:prop1_geo_diff}
\end{align}
Taking the same norm, and using the fact that \(\|X\|_{\widetilde B_t^{(\ell)}}=\|-X\|_{\widetilde B_t^{(\ell)}}\),
\begin{align}
\bigl\|
\Delta_{\mathrm{merge},t}^{(\ell)}-\Delta_{\mathrm{geo},t}^{(\ell)}
\bigr\|_{\widetilde B_t^{(\ell)}}
&=
\bigl(1-\alpha_t^{(\ell)}\bigr)
\bigl\|
\Delta_{\mathrm{plain},t}^{(\ell)}-\Delta_{\mathrm{geo},t}^{(\ell)}
\bigr\|_{\widetilde B_t^{(\ell)}}
\nonumber\\
&=
\bigl(1-\alpha_t^{(\ell)}\bigr)
\bigl\|
\Delta_{\mathrm{geo},t}^{(\ell)}-\Delta_{\mathrm{plain},t}^{(\ell)}
\bigr\|_{\widetilde B_t^{(\ell)}}
\nonumber\\
&=
\bigl(1-\alpha_t^{(\ell)}\bigr) D_t^{(\ell)}.
\label{eq:prop1_second_identity}
\end{align}

We next prove the threshold characterization of the gate. Since \(\kappa>0\) and the sigmoid function \(\sigma(\cdot)\) is monotone increasing with \(\sigma(0)=1/2\), Eq.~\eqref{eq:prop1_gate_appendix} implies
\[
g_t^{(\ell)} \le \tau
\quad\Longrightarrow\quad
\sigma\!\bigl(\kappa(g_t^{(\ell)}-\tau)\bigr)\le \frac12,
\]
and therefore
\begin{align}
\alpha_t^{(\ell)}
&=
\alpha_{\min}
+
(\alpha_{\max}-\alpha_{\min})
\,\sigma\!\bigl(\kappa(g_t^{(\ell)}-\tau)\bigr)
\nonumber\\
&\le
\alpha_{\min}
+
\frac{\alpha_{\max}-\alpha_{\min}}{2}
=
\frac{\alpha_{\min}+\alpha_{\max}}{2}.
\label{eq:prop1_low_conflict}
\end{align}
Likewise,
\[
g_t^{(\ell)} \ge \tau
\quad\Longrightarrow\quad
\sigma\!\bigl(\kappa(g_t^{(\ell)}-\tau)\bigr)\ge \frac12,
\]
which gives
\begin{equation}
\alpha_t^{(\ell)}
\ge
\frac{\alpha_{\min}+\alpha_{\max}}{2}.
\label{eq:prop1_high_conflict}
\end{equation}

Combining Eqs.~\eqref{eq:prop1_first_identity}, \eqref{eq:prop1_second_identity}, \eqref{eq:prop1_low_conflict}, and \eqref{eq:prop1_high_conflict} establishes the proposition. In particular, when \(g_t^{(\ell)}\) is below the threshold \(\tau\), GCWM applies a weaker geometry-aware correction, whereas when \(g_t^{(\ell)}\) exceeds \(\tau\), the geometry-aware branch receives a larger weight. This establishes the regime characterization claimed in Proposition~\ref{prop:compatibility_regimes}.
\end{proof}

\section{Analysis Metrics}
\label{app:analysis_metrics}

We use four families of signals to analyze continual post-training dynamics.
For a step update \(\Delta_t\), we measure parameter drift by the update norm
\[
n_t=\Big(\sum_{\ell\in\mathcal{L}}\|\Delta_t^{(\ell)}\|_F^2\Big)^{1/2}.
\]
Given the retained right-singular subspace \(V_j^{(\ell)}\) of task \(T_j\), we measure subspace overlap by the subspace alignment ratio (SAR)
\[
\mathrm{SAR}_{i\rightarrow j}^{(\ell)}
=
\frac{\|\Delta_i^{(\ell)}V_j^{(\ell)}\|_F}
{\|\Delta_i^{(\ell)}\|_F+\varepsilon},
\qquad
\mathrm{SAR}_{ij}^{(\ell)}
=
\frac{1}{2}
\big(
\mathrm{SAR}_{i\rightarrow j}^{(\ell)}
+
\mathrm{SAR}_{j\rightarrow i}^{(\ell)}
\big).
\]
For projected task geometries \(B_i^{(\ell)}\) and \(B_j^{(\ell)}\), we define geometry conflict by the normalized Bures--Wasserstein discrepancy
\[
\gamma_{ij}^{(\ell)}
=
\frac{
d_{\mathrm{B}}^2(B_i^{(\ell)},B_j^{(\ell)})
}{
\mathrm{tr}(B_i^{(\ell)})+\mathrm{tr}(B_j^{(\ell)})+\varepsilon
}.
\]
For gradient-based diagnostics, we compute gradient cosine similarity
\[
c_{ij}^{(\ell)}
=
\frac{\langle g_i^{(\ell)},g_j^{(\ell)}\rangle}
{\|g_i^{(\ell)}\|_2\|g_j^{(\ell)}\|_2+\varepsilon},
\]
and report both mean cosine and the fraction of negative-cosine pairs.
State-relative variants replace one task update by the current continual-training state.

\FloatBarrier
\section{Additional Empirical Analysis for Sec.~\ref{sec:findings}}
\label{app:sec3_additional}

This appendix mirrors the four empirical findings in Sec.~\ref{sec:findings}: diagnostic dashboards, step-level and state-relative analysis, pairwise compatibility, and module-level geometry--gradient complementarity.

\subsection{Statistical Confidence for Sec.~\ref{sec:findings}}
\label{app:sec3_confidence}

To quantify statistical uncertainty in Sec.~\ref{sec:findings}, we report run-cluster bootstrap confidence intervals and permutation-test significance for key Spearman associations.
For each statistic, we use 2{,}000 bootstrap resamples with clusters defined by run (model-size \(\times\) method sequence), which preserves within-run temporal dependence across continual steps.
We also report two-sided permutation \(p\)-values with 3{,}000 shuffles.

\begin{figure*}[!htbp]
  \centering
  \includegraphics[width=\textwidth]{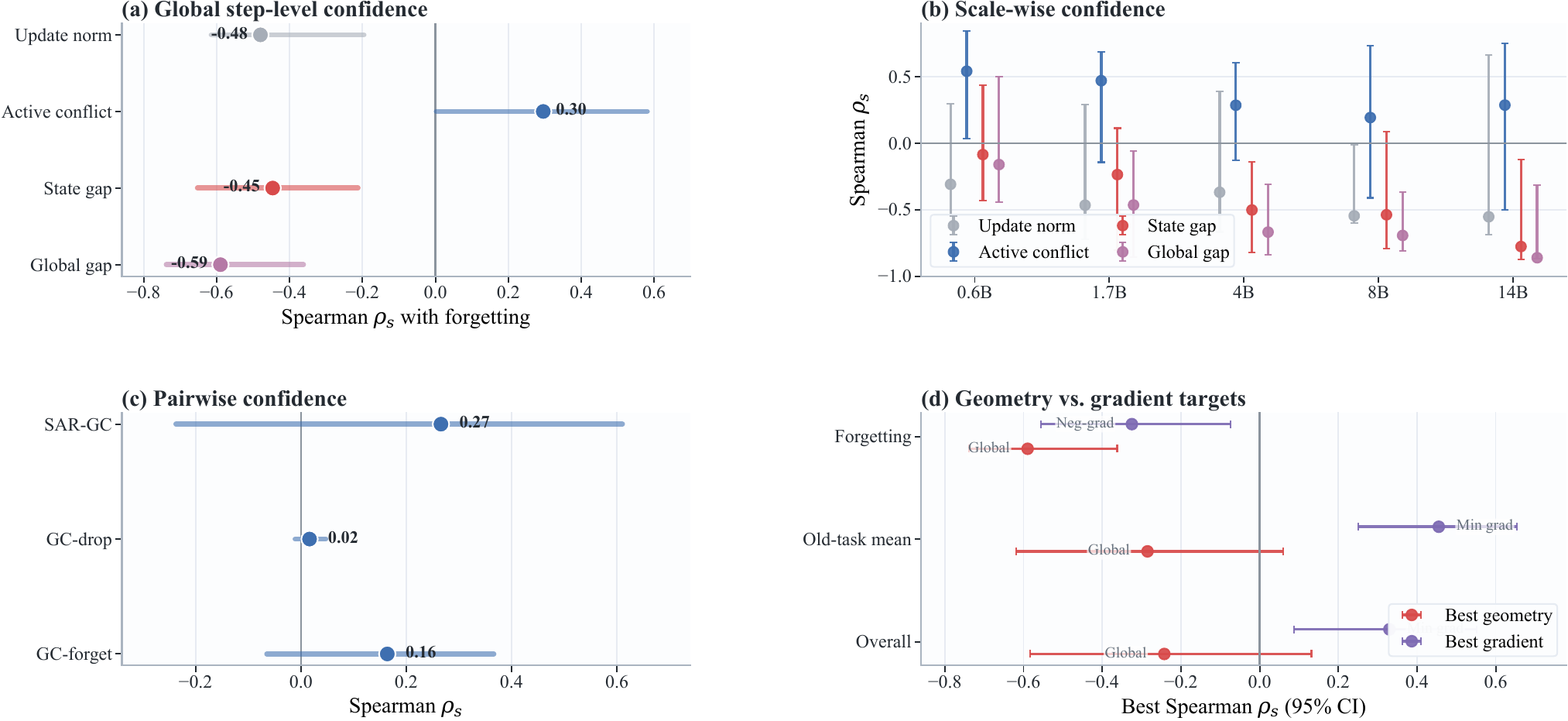}
  \caption{
  {Statistical confidence for Sec.~\ref{sec:findings}.}
  Error bars show run-cluster bootstrap 95\% confidence intervals for Spearman associations.
  }
  \label{fig:app_sec3_stat_confidence}
\end{figure*}

\begin{table}[!htbp]
\centering
\caption{{Global step-level confidence for forgetting associations.}}
\label{tab:app_sec3_conf_step_global}
\scriptsize
\setlength{\tabcolsep}{4pt}
\begin{tabular}{lccc}
\toprule
Signal & Spearman $\rho_s$ [95\% CI] & $p_{\mathrm{perm}}$ & $n$ \\
\midrule
Update norm & -0.48 [-0.61, -0.20] & <1e-3 & 250 \\
Active conflict & 0.30 [0.00, 0.58] & <1e-3 & 250 \\
State gap & -0.45 [-0.65, -0.21] & <1e-3 & 250 \\
Global gap & -0.59 [-0.74, -0.36] & <1e-3 & 250 \\
\bottomrule
\end{tabular}

\end{table}

\begin{table*}[!htbp]
\centering
\caption{{Step-level confidence by model scale.}}
\label{tab:app_sec3_conf_step_scale}
\scriptsize
\setlength{\tabcolsep}{4pt}
\begin{tabular}{lcccc}
\toprule
Scale & Update norm & Active conflict & State gap & Global gap \\
\midrule
0.6B & -0.31 [-0.64, 0.29] & 0.54 [0.04, 0.84] & -0.09 [-0.43, 0.44] & -0.16 [-0.44, 0.50] \\
1.7B & -0.47 [-0.72, 0.29] & 0.47 [-0.14, 0.69] & -0.24 [-0.80, 0.11] & -0.47 [-0.86, -0.06] \\
4B & -0.37 [-0.67, 0.39] & 0.28 [-0.13, 0.60] & -0.50 [-0.82, -0.14] & -0.67 [-0.84, -0.31] \\
8B & -0.55 [-0.60, -0.01] & 0.19 [-0.41, 0.73] & -0.54 [-0.79, 0.09] & -0.69 [-0.81, -0.37] \\
14B & -0.55 [-0.69, 0.66] & 0.29 [-0.50, 0.75] & -0.78 [-0.87, -0.12] & -0.86 [-0.87, -0.32] \\
\bottomrule
\end{tabular}

\end{table*}

\begin{table*}[!htbp]
\centering
\caption{{Step-level confidence by continual-training method.}}
\label{tab:app_sec3_conf_step_method}
\scriptsize
\setlength{\tabcolsep}{4pt}
\begin{tabular}{lcccc}
\toprule
Method & Update norm & Active conflict & State gap & Global gap \\
\midrule
Seq. SFT & -0.23 [-0.30, 0.18] & -0.16 [-0.49, 0.36] & -0.68 [-0.83, -0.30] & -0.70 [-0.83, -0.33] \\
EWC & 0.35 [0.04, 0.71] & 0.51 [0.08, 0.78] & -0.40 [-0.71, 0.04] & -0.42 [-0.77, 0.03] \\
FOREVER & -0.06 [-0.41, 0.18] & -0.27 [-0.50, 0.09] & -0.06 [-0.31, 0.34] & -0.10 [-0.33, 0.33] \\
AIMMerging & 0.33 [0.15, 0.52] & 0.22 [-0.22, 0.63] & 0.10 [-0.03, 0.30] & 0.01 [-0.20, 0.21] \\
\bottomrule
\end{tabular}

\end{table*}

\begin{table}[!htbp]
\centering
\caption{{Global pairwise confidence for Sec.~\ref{sec:sar_geometry}.}}
\label{tab:app_sec3_conf_pair_global}
\scriptsize
\setlength{\tabcolsep}{4pt}
\begin{tabular}{lccc}
\toprule
Pairwise relation & Spearman $\rho_s$ [95\% CI] & $p_{\mathrm{perm}}$ & $n$ \\
\midrule
SAR vs. geometry conflict & 0.27 [-0.24, 0.61] & <1e-3 & 1820 \\
GC vs. immediate old-task change & 0.02 [-0.01, 0.05] & 0.485 & 1820 \\
GC vs. forgetting-from-best & 0.16 [-0.06, 0.37] & <1e-3 & 1820 \\
\bottomrule
\end{tabular}

\end{table}

\begin{table*}[!htbp]
\centering
\caption{{Best geometry vs.\ gradient predictors by target.}}
\label{tab:app_sec3_conf_geo_grad}
\scriptsize
\setlength{\tabcolsep}{4pt}
\begin{tabular}{llll}
\toprule
Target & Best geometry predictor & Best gradient predictor & $\Delta|\rho_s|$ (Geom$-$Grad) \\
\midrule
Forgetting & Global gap: -0.59 [-0.74, -0.36] & Neg-grad ratio: -0.32 [-0.56, -0.07] & +0.26 \\
Old-task mean & Global gap: -0.28 [-0.62, 0.06] & Min grad-cos: 0.46 [0.25, 0.65] & -0.17 \\
Overall & Global gap: -0.24 [-0.58, 0.13] & Min grad-cos: 0.33 [0.09, 0.55] & -0.09 \\
\bottomrule
\end{tabular}

\end{table*}

Fig.~\ref{fig:app_sec3_stat_confidence} and Tables~\ref{tab:app_sec3_conf_step_global}--\ref{tab:app_sec3_conf_geo_grad} support the robustness of the main Sec.~\ref{sec:findings} claims.
At the global step level, global state-relative geometry gap remains the strongest forgetting-associated signal (\(\rho_s=-0.59\), 95\% CI \([-0.74,-0.36]\)), stronger than update norm (\(-0.48\), CI \([-0.61,-0.20]\)) and active conflict (\(0.30\), CI \([0.00,0.58]\)).
By scale, the state/global confidence intervals become more negative and better separated on larger models (4B--14B), while 0.6B intervals are substantially wider.
At the pairwise level, SAR--GC is non-redundant (\(\rho_s=0.27\)) but GC--drop remains near zero (\(\rho_s=0.02\), \(p=0.485\)).
For predictor families, geometry signals are strongest for forgetting, whereas gradient diagnostics are stronger for old-task mean and overall score, matching the complementarity claim in Sec.~\ref{sec:geometry_gradient}.

\subsection{Diagnostic Dashboards for Sec.~\ref{sec:findings}}
\label{app:sec3_dashboard}

\begin{figure*}[!htbp]
  \centering
  \includegraphics[width=\textwidth]{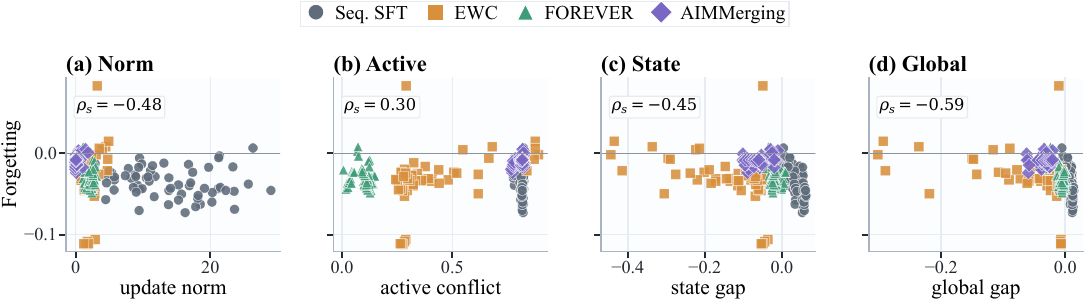}
  \caption{
  {Drift and geometry-discrepancy signals versus forgetting.}
  The four panels are arranged horizontally for readability. Each point is a continual post-training step; no smoothing or connecting curve is used.
  }
  \label{fig:app_sec3_drift_geometry}
\end{figure*}

\begin{figure*}[!htbp]
  \centering
  \includegraphics[width=\textwidth]{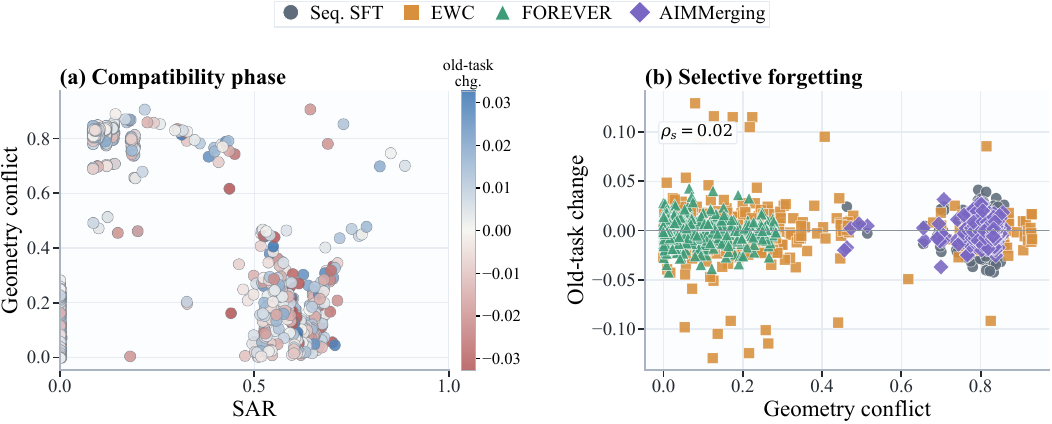}
  \caption{
  {Pairwise subspace and geometry compatibility.}
  The left panel compares SAR with geometry conflict, while the right panel relates pairwise geometry conflict to immediate old-task change.
  }
  \label{fig:app_sec3_compatibility_phase}
\end{figure*}

\begin{figure*}[!htbp]
  \centering
  \includegraphics[width=\textwidth]{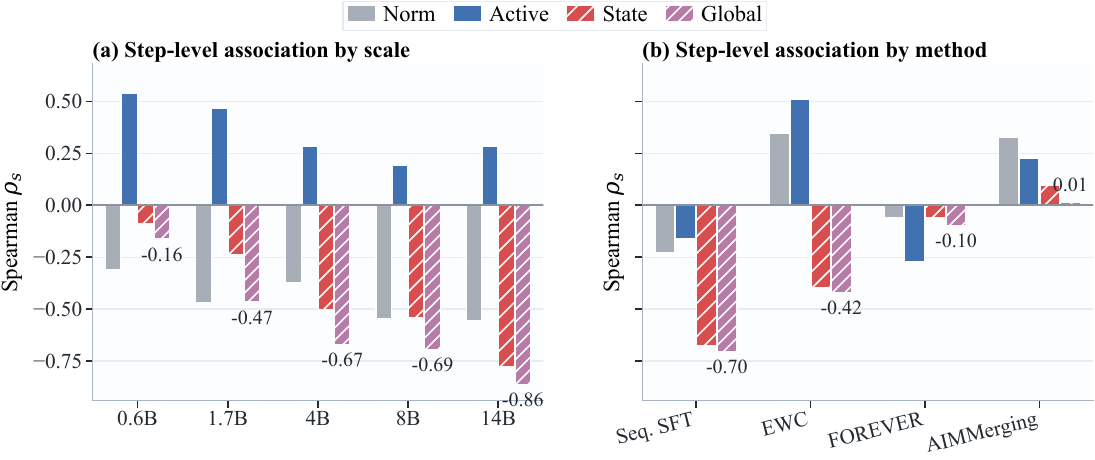}
  \caption{
  {Step-level correlation summary by scale and method.}
  Norm, active-pair conflict, state gap, and global gap are compared against forgetting using Spearman correlation.
  }
  \label{fig:app_sec3_step_corr_summary}
\end{figure*}

\begin{figure*}[!htbp]
  \centering
  \includegraphics[width=\textwidth]{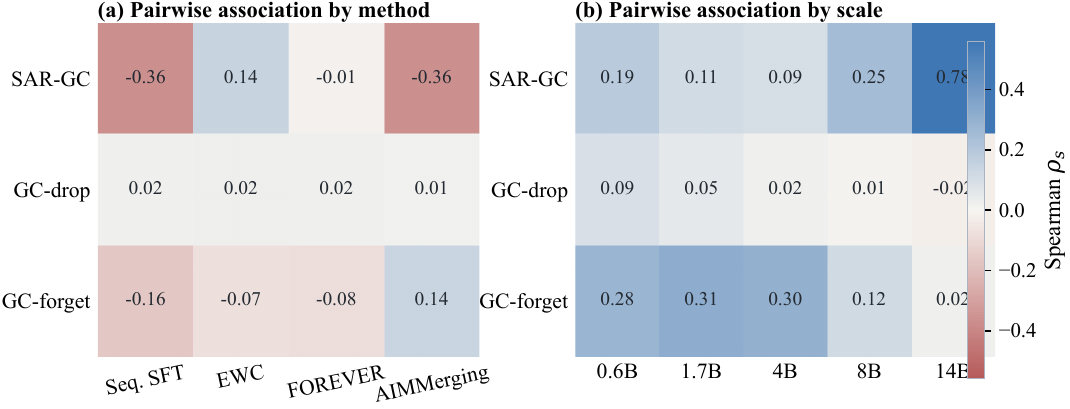}
  \caption{
  {Pairwise correlation summary by method and scale.}
  SAR-GC measures the relation between subspace overlap and geometry conflict; GC-drop and GC-forget measure pairwise geometry against immediate old-task change and best-previous forgetting.
  }
  \label{fig:app_sec3_pairwise_corr_summary}
\end{figure*}

\begin{figure*}[!htbp]
  \centering
  \includegraphics[width=\textwidth]{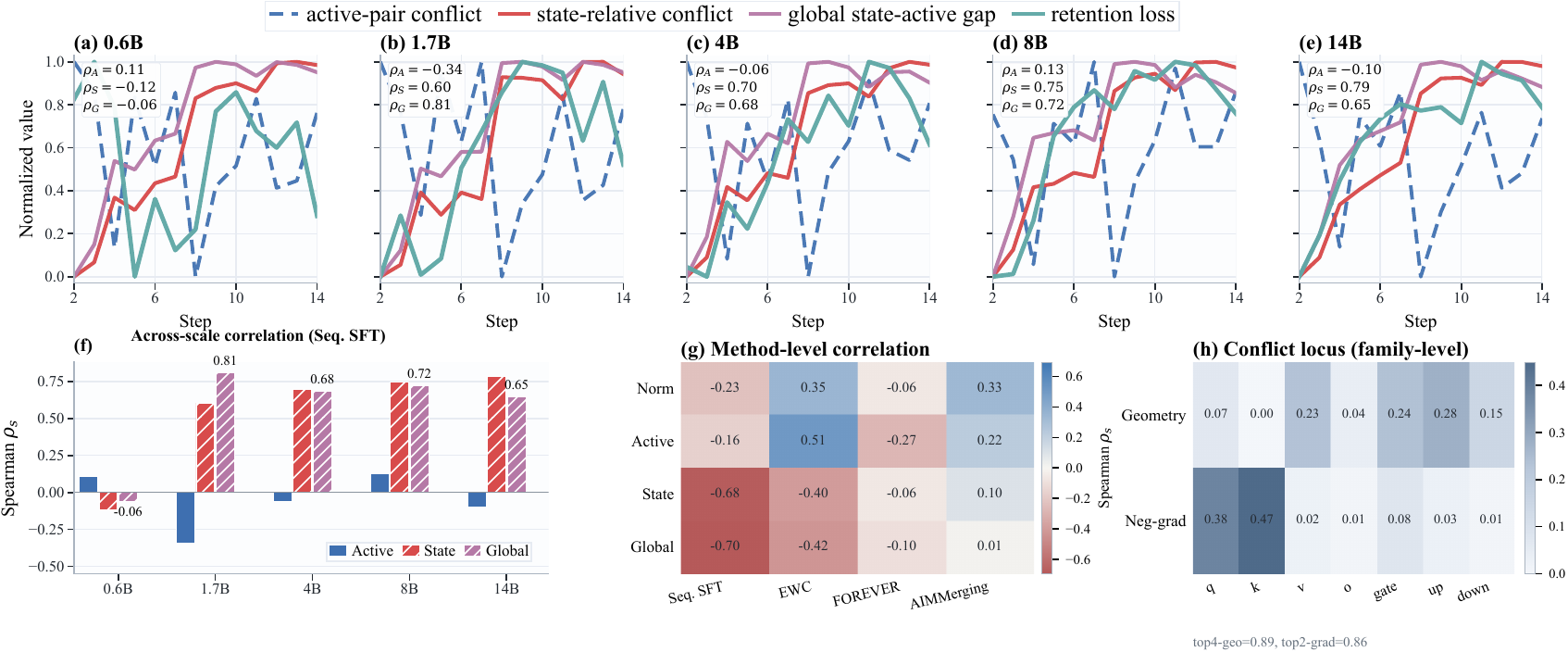}
  \caption{
  {Full state-relative geometry diagnostic.}
  This view expands the state-relative analysis with all Qwen3 scales, cross-scale correlations, method-level summaries, and module-level signal separation.
  }
  \label{fig:app_state_relative_focus_full}
\end{figure*}

Figs.~\ref{fig:app_sec3_drift_geometry}--\ref{fig:app_sec3_pairwise_corr_summary} split the former Sec.~\ref{sec:findings} overview dashboard into four readable diagnostics.
Fig.~\ref{fig:app_sec3_drift_geometry} confirms that update norm is only a coarse drift baseline: the global state-active gap has the strongest overall step-level association with forgetting (\(\rho_s=-0.59\)), compared with update norm (\(\rho_s=-0.48\)) and active-pair geometry conflict (\(\rho_s=0.30\)).
Fig.~\ref{fig:app_sec3_compatibility_phase} shows that SAR and geometry conflict are related but non-redundant, while pairwise geometry conflict alone remains weak for immediate selective forgetting.
Figs.~\ref{fig:app_sec3_step_corr_summary} and \ref{fig:app_sec3_pairwise_corr_summary} provide the corresponding scale/method summaries: the state-relative association strengthens on larger models, reaching \(-0.67\), \(-0.69\), and \(-0.86\) for Qwen3-4B, 8B, and 14B, whereas GC-drop remains near zero across methods and scales.

Fig.~\ref{fig:app_state_relative_focus_full} expands the state-relative tracking view.
Under Seq.\ SFT, active-pair conflict often changes abruptly with local task transitions, whereas state-relative and global geometry gaps evolve with the accumulated model state and more closely track retention loss.
For 1.7B/4B/8B/14B, the correlation between retention loss and state-relative conflict is \(0.60/0.70/0.75/0.79\), while the corresponding global-gap association is \(0.81/0.68/0.72/0.65\).
The 0.6B case is noisier, consistent with the weaker and less stable small-model correlations reported in Table~\ref{tab:app_step_corr}.
Together, these split diagnostics provide the complete evidence behind the focused main-text figures.

\FloatBarrier
\subsection{Additional Step-Level and State-Relative Analysis}
\label{app:additional_step_analysis}

\begin{figure*}[!htbp]
  \centering
  \includegraphics[width=\textwidth]{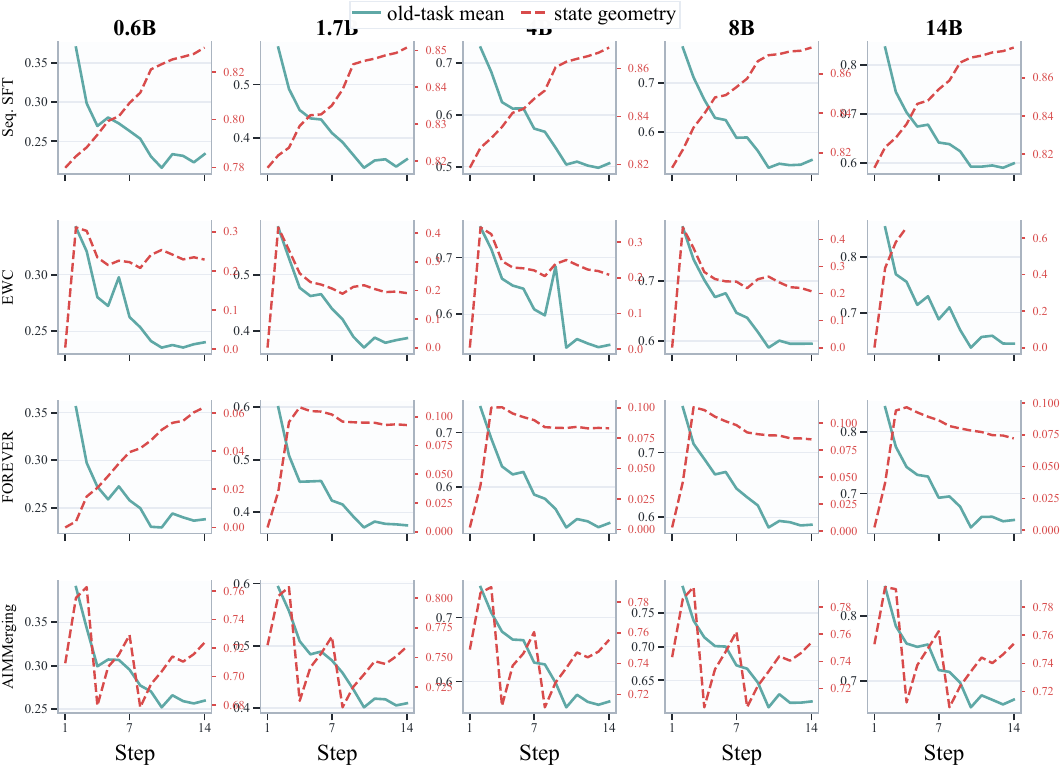}
  \caption{
  {Step-level continual post-training dynamics by method.}
  We show downstream retention and mechanism signals across steps for Seq. SFT, EWC, FOREVER, and AIMMerging.
  }
  \label{fig:app_step_dynamics}
\end{figure*}

\definecolor{rcwmBlue}{HTML}{DDEEFF}
\definecolor{mergeBlue}{HTML}{F1F7FD}
\definecolor{refGray}{HTML}{F7F7F7}
\definecolor{mtlGold}{HTML}{FFF4D6}
\definecolor{sizeBand}{HTML}{EAF2FA}
\definecolor{ruleNavy}{HTML}{244B73}
\providecommand{\best}[1]{\textbf{#1}}
\providecommand{\ub}[1]{\underline{#1}}

\begin{table}[!htbp]
\centering
\caption{
{Step-level Spearman correlations with forgetting.}
Norm denotes incremental update norm; Active denotes active-pair geometry conflict; State and Global denote state-relative geometry gaps.
}
\label{tab:app_step_corr}
\arrayrulecolor{ruleNavy}
\renewcommand{\arraystretch}{1.08}
\begin{tabular}{lrrrr}
\toprule[0.9pt]
\rowcolor{sizeBand}
Group & Norm & Active & State & Global \\
\midrule
\rowcolor{rcwmBlue} All & -0.48 & 0.30 & -0.45 & -0.59 \\
\rowcolor{refGray} 0.6B & -0.31 & 0.54 & -0.09 & -0.16 \\
\rowcolor{mergeBlue} 1.7B & -0.47 & 0.47 & -0.24 & -0.47 \\
\rowcolor{refGray} 4B & -0.37 & 0.28 & -0.50 & -0.67 \\
\rowcolor{mergeBlue} 8B & -0.55 & 0.19 & -0.54 & -0.69 \\
\rowcolor{refGray} 14B & -0.55 & 0.29 & -0.78 & -0.86 \\
\bottomrule[0.9pt]
\end{tabular}
\arrayrulecolor{black}
\end{table}

\begin{figure*}[!htbp]
  \centering
  \includegraphics[width=\textwidth]{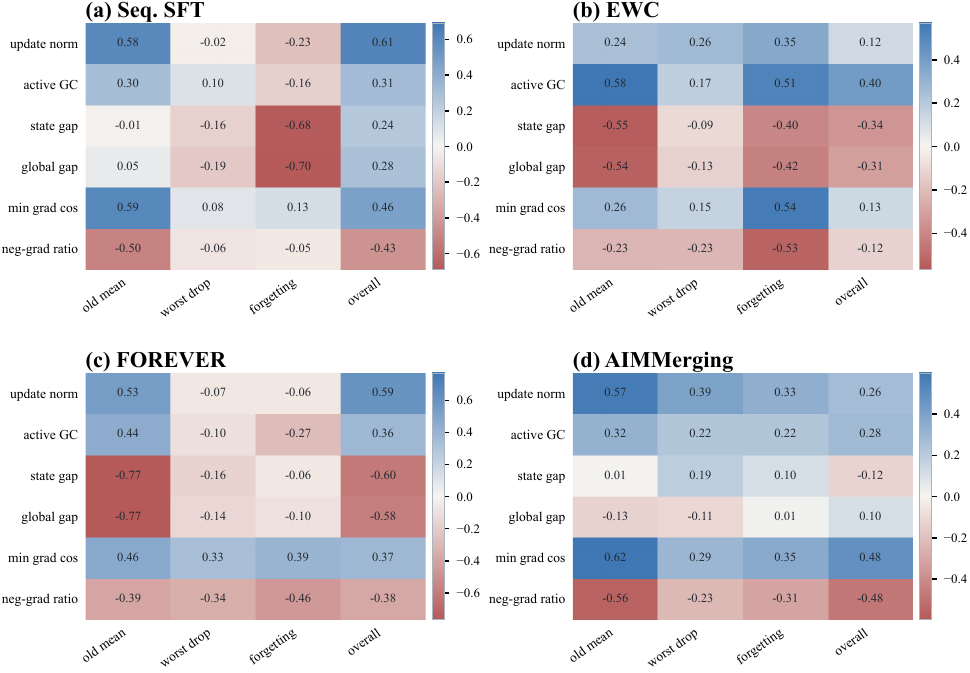}
  \caption{
  {Method-level correlation heatmaps.}
  We compare update norm, SAR, geometry conflict, and gradient conflict against retention, forgetting, and final performance.
  }
  \label{fig:app_method_corr}
\end{figure*}

\begin{table*}[!htbp]
  \centering
  \caption{
  {Step-level Spearman correlations with forgetting by method.}
  Norm denotes incremental update norm; Active denotes active-pair geometry conflict; State and Global denote state-relative geometry gaps.
  }
  \label{tab:app_step_corr_method}
  \scriptsize
  \setlength{\tabcolsep}{4pt}
  {\arrayrulecolor{ruleNavy}
\renewcommand{\arraystretch}{1.08}
\begin{tabular}{lrrrrrr}
\toprule[0.9pt]
\rowcolor{sizeBand}
Group & Update norm & Active GC & State gap & Global gap & Neg-grad ratio & Min grad cos \\
\midrule
\rowcolor{refGray} Seq. SFT & -0.23 & -0.16 & -0.68 & -0.70 & -0.05 & 0.13 \\
\rowcolor{refGray} EWC & 0.35 & 0.51 & -0.40 & -0.42 & -0.53 & 0.54 \\
\rowcolor{refGray} FOREVER & -0.06 & -0.27 & -0.06 & -0.10 & -0.46 & 0.39 \\
\rowcolor{mergeBlue} AIMMerging & 0.33 & 0.22 & 0.10 & 0.01 & -0.31 & 0.35 \\
\bottomrule[0.9pt]
\end{tabular}
\arrayrulecolor{black}}

\end{table*}

\begin{table*}[!htbp]
  \centering
  \caption{
  {Top step-level predictors for downstream targets.}
  We rank predictors by absolute Spearman correlation.
  }
  \label{tab:app_step_explanation_top}
  \scriptsize
  \setlength{\tabcolsep}{4pt}
  {\arrayrulecolor{ruleNavy}
\renewcommand{\arraystretch}{1.06}
\begin{tabular}{llrrr}
\toprule[0.9pt]
\rowcolor{sizeBand}
Target & Predictor & Pearson & Spearman & $|\rho_s|$ \\
\midrule
\rowcolor{rcwmBlue} Forgetting & Merged norm & -0.41 & -0.64 & 0.64 \\
\rowcolor{rcwmBlue} Forgetting & Global gap & -0.33 & -0.59 & 0.59 \\
\rowcolor{rcwmBlue} Forgetting & Update norm & -0.28 & -0.48 & 0.48 \\
\rowcolor{rcwmBlue} Forgetting & State gap & -0.31 & -0.45 & 0.45 \\
\rowcolor{rcwmBlue} Forgetting & Global active GC & 0.27 & 0.36 & 0.36 \\
\rowcolor{rcwmBlue} Forgetting & Neg-grad ratio & -0.16 & -0.32 & 0.32 \\
\rowcolor{rcwmBlue} Forgetting & Min grad cos & 0.25 & 0.30 & 0.30 \\
\rowcolor{rcwmBlue} Forgetting & Active GC & 0.25 & 0.30 & 0.30 \\
\midrule
\rowcolor{mergeBlue} Old-task mean & Min grad cos & 0.42 & 0.46 & 0.46 \\
\rowcolor{mergeBlue} Old-task mean & Neg-grad ratio & -0.16 & -0.38 & 0.38 \\
\rowcolor{mergeBlue} Old-task mean & Global gap & -0.23 & -0.28 & 0.28 \\
\rowcolor{mergeBlue} Old-task mean & Global active GC & 0.16 & 0.27 & 0.27 \\
\rowcolor{mergeBlue} Old-task mean & Active SAR & 0.06 & -0.26 & 0.26 \\
\rowcolor{mergeBlue} Old-task mean & Active SAR (clip) & 0.06 & -0.26 & 0.26 \\
\rowcolor{mergeBlue} Old-task mean & State gap & -0.22 & -0.21 & 0.21 \\
\rowcolor{mergeBlue} Old-task mean & Active GC & 0.12 & 0.20 & 0.20 \\
\midrule
\rowcolor{refGray} Overall & Min grad cos & 0.30 & 0.33 & 0.33 \\
\rowcolor{refGray} Overall & Neg-grad ratio & -0.18 & -0.31 & 0.31 \\
\rowcolor{refGray} Overall & Active SAR & 0.07 & -0.26 & 0.26 \\
\rowcolor{refGray} Overall & Active SAR (clip) & 0.08 & -0.26 & 0.26 \\
\rowcolor{refGray} Overall & Global gap & -0.09 & -0.24 & 0.24 \\
\rowcolor{refGray} Overall & State gap & -0.12 & -0.23 & 0.23 \\
\rowcolor{refGray} Overall & Active GC & 0.08 & 0.16 & 0.16 \\
\rowcolor{refGray} Overall & Global active GC & 0.10 & 0.15 & 0.15 \\
\bottomrule[0.9pt]
\end{tabular}
\arrayrulecolor{black}}

\end{table*}

\begin{figure}[!htbp]
  \centering
  \includegraphics[width=\linewidth]{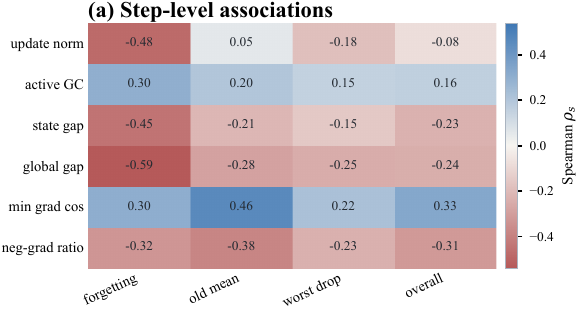}
  \caption{
  {Global step-level explanation heatmap.}
  Each cell reports the Spearman correlation between a mechanism signal and a downstream target.
  }
  \label{fig:app_step_explanation_heatmap}
\end{figure}

Fig.~\ref{fig:app_step_dynamics} expands the step-level analysis by separating continual post-training methods.
The main trend is that raw update magnitude captures a coarse drift effect but does not consistently explain retention across methods.
For example, the global Spearman correlation between incremental update norm and forgetting is \(-0.48\), whereas the global state-active geometry gap reaches \(-0.59\).
This gap becomes stronger at larger scales: its correlation with forgetting is \(-0.67\), \(-0.69\), and \(-0.86\) for Qwen3-4B, 8B, and 14B, respectively.
By contrast, active-pair geometry conflict is more method-dependent and remains weaker as a step-level forgetting predictor.

Fig.~\ref{fig:app_method_corr} summarizes method-level correlations across mechanism signals.
The heatmaps show that no single raw drift statistic consistently explains all outcomes.
State-relative geometry is more aligned with forgetting in Seq. SFT and larger models, while gradient conflict contributes more to old-task retention and overall performance.
This supports our use of geometry conflict as a compatibility signal, rather than treating update magnitude as a sufficient explanation.

Table~\ref{tab:app_step_corr_method} further decomposes the step-level correlations by method.
Seq. SFT shows the clearest state-relative pattern: the state gap and global gap correlate with forgetting at \(\rho_s=-0.68\) and \(-0.70\), respectively, while update norm is weaker (\(\rho_s=-0.23\)).
For EWC, active-pair geometry conflict and gradient conflict are more pronounced, with Active GC at \(\rho_s=0.51\) and minimum gradient cosine at \(\rho_s=0.54\).
FOREVER weakens the geometry-forgetting association, consistent with replay reducing direct sequential interference.
AIMMerging improves retention but also compresses the variance of forgetting, making correlations less directly interpretable.

Fig.~\ref{fig:app_step_explanation_heatmap} and Table~\ref{tab:app_step_explanation_top} summarize the same trends across targets.
For forgetting, the strongest geometry-based state signal is the global gap (\(\rho_s=-0.59\)), which is stronger than the incremental update norm (\(\rho_s=-0.48\)) and active-pair geometry conflict (\(\rho_s=0.30\)).
For old-task mean and overall performance, gradient conflict becomes more visible: minimum gradient cosine reaches \(\rho_s=0.46\) with old-task mean and \(\rho_s=0.33\) with overall score.
This supports the main-text interpretation that geometry conflict is the more direct update-integration signal, while gradient conflict exposes complementary optimization-level interference.

\FloatBarrier
\subsection{Additional Pairwise Compatibility Analysis}
\label{app:pairwise_analysis}

\begin{figure*}[!htbp]
  \centering
  \includegraphics[width=\textwidth]{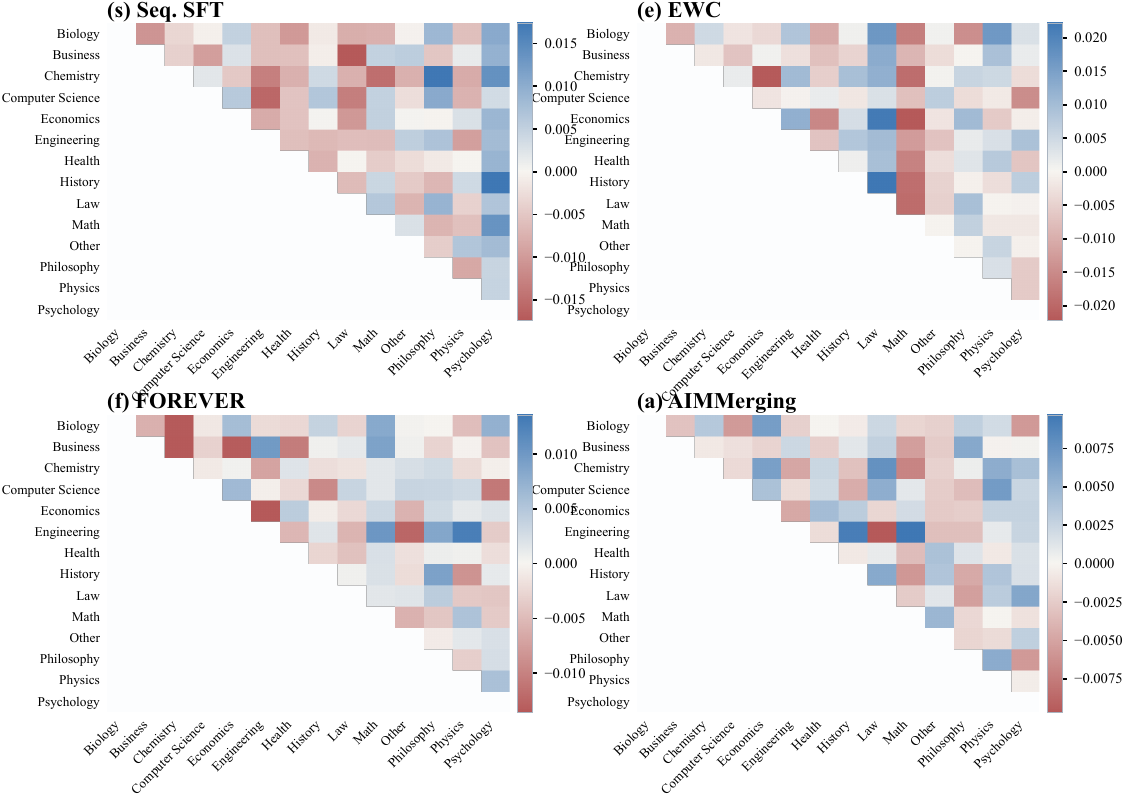}
  \caption{
  {Task-pair selective forgetting by method.}
  Each heatmap reports the old-task score change after introducing a new task.
  }
  \label{fig:app_pair_forgetting}
\end{figure*}

\begin{figure*}[!htbp]
  \centering
  \includegraphics[width=\textwidth]{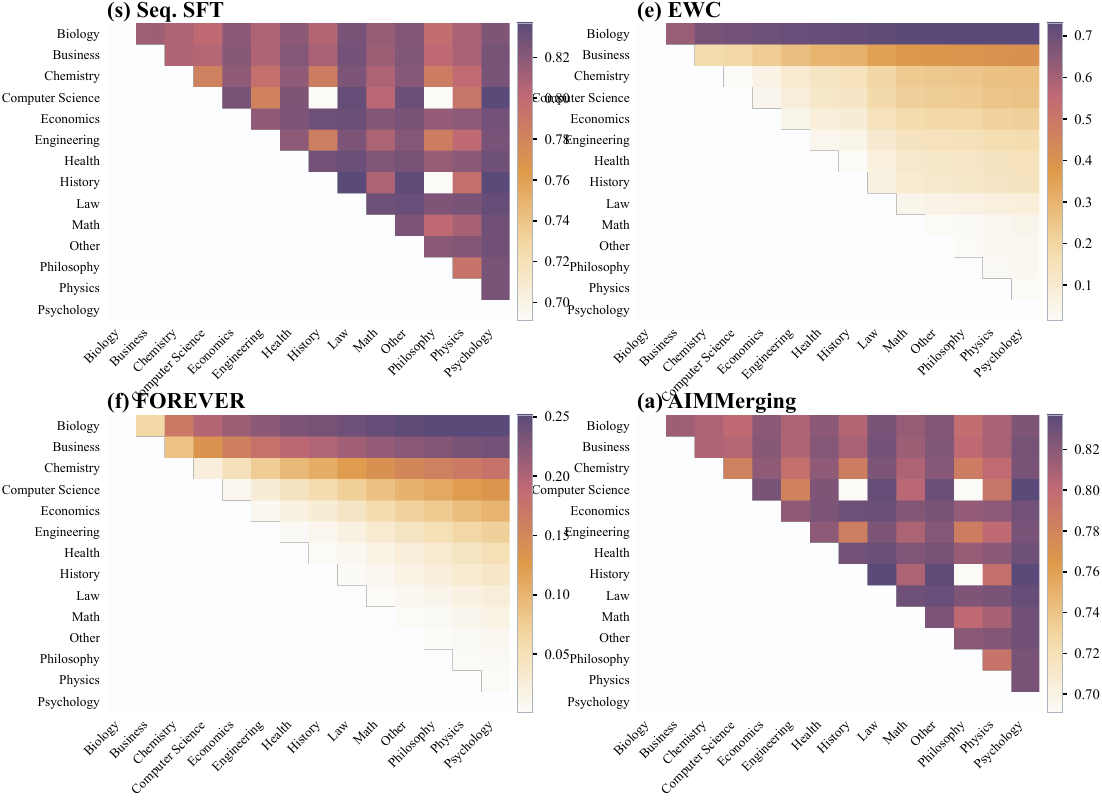}
  \caption{
  {Task-pair geometry conflict by method.}
  Pairwise geometry conflict reveals compatibility structure that is not captured by old-task score change alone.
  }
  \label{fig:app_pair_geometry}
\end{figure*}

\begin{table}[!htbp]
\centering
\caption{
{Pairwise compatibility summary by method.}
SAR-GC denotes Spearman correlation between SAR and geometry conflict; GC-drop denotes geometry conflict vs. immediate old-task change; GC-forget denotes geometry conflict vs. forgetting from best previous score.
}
\label{tab:app_pairwise_method_summary}
{\arrayrulecolor{ruleNavy}
\renewcommand{\arraystretch}{1.08}
\begin{tabular}{lccc}
\toprule[0.9pt]
\rowcolor{sizeBand}
Method & SAR-GC & GC-drop & GC-forget \\
\midrule
\rowcolor{refGray} Seq. SFT & -0.36 & 0.02 & -0.16 \\
\rowcolor{refGray} EWC & 0.14 & 0.02 & -0.07 \\
\rowcolor{refGray} FOREVER & -0.01 & 0.02 & -0.08 \\
\rowcolor{mergeBlue} AIMMerging & -0.36 & 0.01 & 0.14 \\
\bottomrule[0.9pt]
\end{tabular}
\arrayrulecolor{black}}

\end{table}

\begin{table}[!htbp]
\centering
\caption{
{Most harmful task transitions.}
Drop reports the old-task score change after adding the new task; GC is pairwise geometry conflict.
}
\label{tab:app_top_harmful_pairs}
{\arrayrulecolor{ruleNavy}
\renewcommand{\arraystretch}{1.08}
\begin{tabular}{llcc}
\toprule[0.9pt]
\rowcolor{sizeBand}
Transition & Method & Drop & GC \\
\midrule
\rowcolor{rcwmBlue} Math$\rightarrow$History & EWC & -0.129 & 0.12 \\
\rowcolor{rcwmBlue} Math$\rightarrow$Economics & EWC & -0.123 & 0.22 \\
\rowcolor{rcwmBlue} Math$\rightarrow$Chemistry & EWC & -0.114 & 0.26 \\
\rowcolor{mergeBlue} Math$\rightarrow$Health & EWC & -0.104 & 0.14 \\
\rowcolor{mergeBlue} Math$\rightarrow$Computer Science & EWC & -0.100 & 0.26 \\
\rowcolor{mergeBlue} Math$\rightarrow$Law & EWC & -0.097 & 0.05 \\
\rowcolor{refGray} Math$\rightarrow$Business & EWC & -0.093 & 0.44 \\
\rowcolor{refGray} Math$\rightarrow$Engineering & EWC & -0.091 & 0.17 \\
\bottomrule[0.9pt]
\end{tabular}
\arrayrulecolor{black}}

\end{table}

% ---------------------------------------------------------------------
% Insert near the end of "Additional Pairwise Compatibility Analysis".
% The first two tables can replace the compact pairwise_method_summary table
% if space is tight; otherwise keep both compact and full versions.
% ---------------------------------------------------------------------

\begin{table*}[!htbp]
  \centering
  \caption{
  {Full pairwise compatibility summary by method.}
  SAR-GC measures the relation between subspace overlap and geometry conflict; GC-drop and GC-forget measure pairwise geometry against immediate old-task change and forgetting from best previous score.
  }
  \label{tab:app_pairwise_method_full}
  \scriptsize
  \setlength{\tabcolsep}{4pt}
  {\arrayrulecolor{ruleNavy}
\renewcommand{\arraystretch}{1.08}
\begin{tabular}{lrrrrrrr}
\toprule[0.9pt]
\rowcolor{sizeBand}
Method & $N$ & SAR-GC & SAR-drop & GC-drop & GC-forget & Med. SAR & Med. GC \\
\midrule
\rowcolor{refGray} Seq. SFT & 455 & -0.36 & 0.04 & 0.02 & -0.16 & 0.12 & 0.82 \\
\rowcolor{refGray} EWC & 455 & 0.14 & 0.01 & 0.02 & -0.07 & 0.57 & 0.16 \\
\rowcolor{refGray} FOREVER & 455 & -0.01 & -0.04 & 0.02 & -0.08 & 0.00 & 0.06 \\
\rowcolor{mergeBlue} AIMMerging & 455 & -0.36 & 0.00 & 0.01 & 0.14 & 0.12 & 0.82 \\
\bottomrule[0.9pt]
\end{tabular}
\arrayrulecolor{black}}

\end{table*}

\begin{table*}[!htbp]
  \centering
  \caption{
  {Full pairwise compatibility summary by model scale.}
  Pairwise compatibility is informative but does not fully explain continual forgetting.
  }
  \label{tab:app_pairwise_size_full}
  \scriptsize
  \setlength{\tabcolsep}{4pt}
  {\arrayrulecolor{ruleNavy}
\renewcommand{\arraystretch}{1.08}
\begin{tabular}{lrrrrrrr}
\toprule[0.9pt]
\rowcolor{sizeBand}
Size & $N$ & SAR-GC & SAR-drop & GC-drop & GC-forget & Med. SAR & Med. GC \\
\midrule
\rowcolor{refGray} 0.6B & 364 & 0.19 & 0.03 & 0.09 & 0.28 & 0.18 & 0.74 \\
\rowcolor{mergeBlue} 1.7B & 364 & 0.11 & 0.02 & 0.05 & 0.31 & 0.13 & 0.80 \\
\rowcolor{refGray} 4B & 364 & 0.09 & 0.04 & 0.02 & 0.30 & 0.12 & 0.78 \\
\rowcolor{mergeBlue} 8B & 364 & 0.25 & -0.01 & 0.01 & 0.12 & 0.10 & 0.80 \\
\rowcolor{refGray} 14B & 364 & 0.78 & -0.03 & -0.02 & 0.02 & 0.04 & 0.37 \\
\bottomrule[0.9pt]
\end{tabular}
\arrayrulecolor{black}}

\end{table*}

\begin{table*}[!htbp]
  \centering
  \caption{
  {Global pair-level predictor ranking.}
  Pairwise metrics are ranked by their absolute Spearman correlation with immediate old-task change and forgetting.
  }
  \label{tab:app_pair_explanation_global}
  \scriptsize
  \setlength{\tabcolsep}{4pt}
  {\arrayrulecolor{ruleNavy}
\renewcommand{\arraystretch}{1.06}
\begin{tabular}{llrrr}
\toprule[0.9pt]
\rowcolor{sizeBand}
Target & Predictor & Pearson & Spearman & $|\rho_s|$ \\
\midrule
\rowcolor{mergeBlue} Old-task change & Min grad cos & -0.00 & 0.04 & 0.04 \\
\rowcolor{mergeBlue} Old-task change & Neg-grad ratio & 0.01 & -0.03 & 0.03 \\
\rowcolor{mergeBlue} Old-task change & Mean grad cos & 0.02 & 0.03 & 0.03 \\
\rowcolor{mergeBlue} Old-task change & GC & 0.01 & 0.02 & 0.02 \\
\rowcolor{mergeBlue} Old-task change & Max GC & 0.01 & 0.01 & 0.01 \\
\rowcolor{mergeBlue} Old-task change & Mean SAR & -0.01 & 0.01 & 0.01 \\
\rowcolor{mergeBlue} Old-task change & Mean SAR (clip) & -0.01 & 0.01 & 0.01 \\
\rowcolor{mergeBlue} Old-task change & Max SAR & 0.00 & 0.00 & 0.00 \\
\midrule
\rowcolor{rcwmBlue} Forgetting & Max GC & 0.12 & 0.17 & 0.17 \\
\rowcolor{rcwmBlue} Forgetting & Mean grad cos & 0.12 & 0.17 & 0.17 \\
\rowcolor{rcwmBlue} Forgetting & GC & 0.14 & 0.16 & 0.16 \\
\rowcolor{rcwmBlue} Forgetting & Mean SAR (clip) & 0.03 & -0.16 & 0.16 \\
\rowcolor{rcwmBlue} Forgetting & Mean SAR & 0.03 & -0.16 & 0.16 \\
\rowcolor{rcwmBlue} Forgetting & Min grad cos & 0.08 & 0.15 & 0.15 \\
\rowcolor{rcwmBlue} Forgetting & Max SAR & 0.02 & -0.14 & 0.14 \\
\rowcolor{rcwmBlue} Forgetting & Neg-grad ratio & -0.05 & -0.14 & 0.14 \\
\bottomrule[0.9pt]
\end{tabular}
\arrayrulecolor{black}}

\end{table*}

\begin{table*}[!htbp]
  \centering
  \caption{
  {Top harmful task transitions.}
  We report the largest old-task drops together with SAR, geometry conflict, and gradient cosine.
  }
  \label{tab:app_top_harmful_pairs_full}
  \scriptsize
  \setlength{\tabcolsep}{4pt}
  {\arrayrulecolor{ruleNavy}
\renewcommand{\arraystretch}{1.06}
\begin{tabular}{llrlrrrrr}
\toprule[0.9pt]
\rowcolor{sizeBand}
Size & Method & Step & Transition & Drop & Forget & SAR & GC & Grad cos \\
\midrule
\rowcolor{rcwmBlue} 4B & EWC & 10 & Math$\rightarrow$History & -0.129 & -0.129 & 0.596 & 0.125 & 0.424 \\
\rowcolor{rcwmBlue} 4B & EWC & 10 & Math$\rightarrow$Economics & -0.123 & -0.123 & 0.592 & 0.216 & 0.573 \\
\rowcolor{rcwmBlue} 4B & EWC & 10 & Math$\rightarrow$Chemistry & -0.114 & -0.114 & 0.574 & 0.264 & 0.700 \\
\rowcolor{mergeBlue} 4B & EWC & 10 & Math$\rightarrow$Health & -0.104 & -0.104 & 0.597 & 0.137 & 0.472 \\
\rowcolor{mergeBlue} 4B & EWC & 10 & Math$\rightarrow$Computer Science & -0.100 & -0.100 & 0.590 & 0.258 & 0.623 \\
\rowcolor{mergeBlue} 4B & EWC & 10 & Math$\rightarrow$Law & -0.097 & -0.097 & 0.612 & 0.054 & 0.507 \\
\rowcolor{refGray} 4B & EWC & 10 & Math$\rightarrow$Business & -0.092 & -0.092 & 0.550 & 0.440 & 0.706 \\
\rowcolor{refGray} 4B & EWC & 10 & Math$\rightarrow$Engineering & -0.091 & -0.091 & 0.596 & 0.170 & 0.651 \\
\rowcolor{refGray} 4B & EWC & 10 & Math$\rightarrow$Biology & -0.091 & -0.091 & 0.319 & 0.823 & 0.582 \\
\midrule
\rowcolor{mergeBlue} 14B & EWC & 9 & Law$\rightarrow$Computer Science & -0.058 & -0.058 & 0.002 & 0.073 & 0.650 \\
\rowcolor{mergeBlue} 14B & EWC & 9 & Law$\rightarrow$Chemistry & -0.050 & -0.050 & 0.003 & 0.127 & 0.507 \\
\rowcolor{refGray} 0.6B & EWC & 2 & Business$\rightarrow$Biology & -0.049 & -0.049 & 0.434 & 0.617 & 0.503 \\
\bottomrule[0.9pt]
\end{tabular}
\arrayrulecolor{black}}

\end{table*}

\begin{figure}[!htbp]
  \centering
  \includegraphics[width=\linewidth]{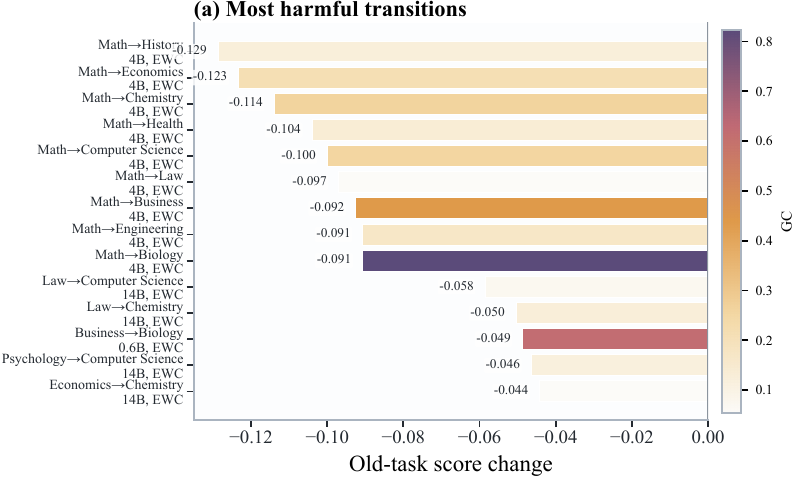}
  \caption{
  {Most harmful task transitions.}
  We visualize the largest old-task drops after introducing new tasks.
  }
  \label{fig:app_top_harmful_pairs_bar}
\end{figure}

Fig.~\ref{fig:app_pair_forgetting} and Fig.~\ref{fig:app_pair_geometry} provide task-pair views of selective forgetting and geometry conflict.
The pairwise results explain why task-task compatibility is informative but incomplete.
Across 1,820 task pairs, SAR and geometry conflict have moderate global association (\(\rho_s=0.27\)), but geometry conflict is only weakly correlated with immediate old-task change (\(\rho_s=0.02\)).
This means that pairwise conflict helps identify compatibility regimes, but forgetting is not determined by isolated task pairs alone.

Table~\ref{tab:app_top_harmful_pairs} lists the largest selective-forgetting transitions.
The most severe drops concentrate around adding Math under EWC for Qwen3-4B, where old-task scores drop by up to \(12.9\) points.
These cases are not explained by pairwise geometry conflict alone: some harmful transitions have moderate conflict, while others have high conflict.
This further motivates the state-relative analysis in Sec.~\ref{sec:state_relative_geometry}.

Tables~\ref{tab:app_pairwise_method_full} and \ref{tab:app_pairwise_size_full} expand the compact pairwise summary.
Across methods, GC-drop remains weak (\(\rho_s=0.01\)--\(0.02\)), confirming that immediate selective forgetting is not determined by pairwise geometry conflict alone.
However, SAR and geometry conflict are not redundant: Seq. SFT and AIMMerging show negative SAR-GC correlations (\(\rho_s=-0.36\)), while EWC shows a weak positive relation (\(\rho_s=0.14\)).
Across scales, GC-forget is more visible on 0.6B, 1.7B, and 4B (\(\rho_s=0.28\), \(0.31\), and \(0.30\)), but becomes weak on 8B and 14B.
This scale dependence motivates the state-relative analysis in the main text.

Table~\ref{tab:app_pair_explanation_global} shows that pairwise predictors have limited explanatory power for immediate old-task change: the largest absolute Spearman correlation is only \(0.04\).
For forgetting from best previous score, pairwise geometry conflict is more useful but still modest (\(\rho_s=0.16\) for mean GC and \(0.17\) for max GC).
Table~\ref{tab:app_top_harmful_pairs_full} gives concrete failure cases.
The largest drops concentrate around the Math step for Qwen3-4B EWC, including Math\(\rightarrow\)History (\(-0.129\)), Math\(\rightarrow\)Economics (\(-0.123\)), and Math\(\rightarrow\)Chemistry (\(-0.114\)).
Their geometry conflict values vary widely, from \(0.054\) for Math\(\rightarrow\)Law to \(0.823\) for Math\(\rightarrow\)Biology, which again shows that harmful forgetting is not captured by a single pairwise metric.
Fig.~\ref{fig:app_top_harmful_pairs_bar} provides the same cases as a compact visual summary.

% ---------------------------------------------------------------------
% Add this as a new subsection after Pairwise Compatibility Analysis.
% ---------------------------------------------------------------------

\FloatBarrier
\subsection{Geometry and Gradient Conflict by Module Family}
\label{app:family_mechanism}

\begin{table}[!htbp]
  \centering
  \caption{
  {Family-level metric summary.}
  We report median SAR, mean geometry conflict, mean gradient cosine, and negative-gradient ratio by module family.
  }
  \label{tab:app_family_metric_summary}
  \scriptsize
  \setlength{\tabcolsep}{3pt}
  {\arrayrulecolor{ruleNavy}
\renewcommand{\arraystretch}{1.08}
\begin{tabular}{lrrrr}
\toprule[0.9pt]
\rowcolor{sizeBand}
Family & Med. SAR & Mean GC & Mean grad cos & Neg-grad ratio \\
\midrule
\rowcolor{mergeBlue} q\_proj & 0.20 & 0.56 & 0.59 & 0.03 \\
\rowcolor{mergeBlue} k\_proj & 0.12 & 0.40 & 0.57 & 0.03 \\
\rowcolor{mergeBlue} v\_proj & 0.21 & 0.56 & 0.64 & 0.01 \\
\rowcolor{mergeBlue} o\_proj & 0.13 & 0.47 & 0.62 & 0.01 \\
\midrule
\rowcolor{rcwmBlue} gate\_proj & 0.11 & 0.49 & 0.63 & 0.02 \\
\rowcolor{rcwmBlue} up\_proj & 0.11 & 0.50 & 0.65 & 0.02 \\
\rowcolor{rcwmBlue} down\_proj & 0.08 & 0.45 & 0.64 & 0.02 \\
\bottomrule[0.9pt]
\end{tabular}
\arrayrulecolor{black}}

\end{table}

\begin{table}[!htbp]
  \centering
  \caption{
  {Top-layer family distribution.}
  We compare the module families most frequently appearing among top geometry-conflict layers and top negative-gradient layers.
  }
  \label{tab:app_top_layer_family_distribution}
  \scriptsize
  \setlength{\tabcolsep}{3pt}
  {\arrayrulecolor{ruleNavy}
\renewcommand{\arraystretch}{1.08}
\begin{tabular}{lrrrrrrr}
\toprule[0.9pt]
\rowcolor{sizeBand}
Metric & q\_proj & k\_proj & v\_proj & o\_proj & gate\_proj & up\_proj & down\_proj \\
\midrule
\rowcolor{rcwmBlue} Top GC layers & 0.07 & 0.00 & 0.23 & 0.04 & 0.24 & 0.28 & 0.15 \\
\rowcolor{mergeBlue} Top neg-grad layers & 0.38 & 0.47 & 0.02 & 0.01 & 0.08 & 0.03 & 0.01 \\
\bottomrule[0.9pt]
\end{tabular}
\arrayrulecolor{black}}

\end{table}

\begin{figure*}[!htbp]
  \centering
  \includegraphics[width=\textwidth]{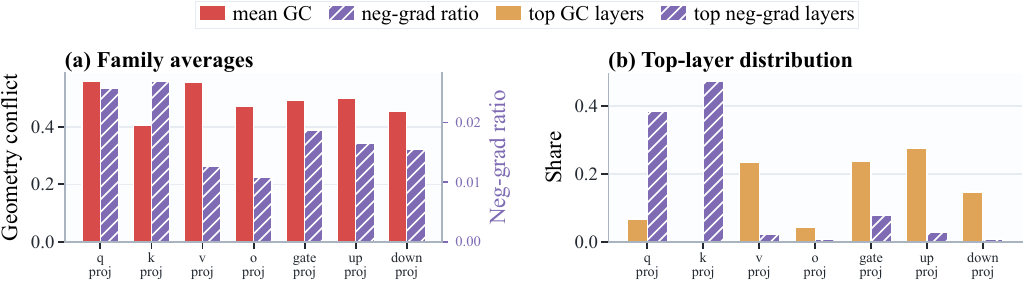}
  \caption{
  {Family-level mechanism profile.}
  Geometry conflict and gradient conflict emphasize different module families.
  }
  \label{fig:app_family_mechanism}
\end{figure*}

\begin{figure*}[!htbp]
  \centering
  \includegraphics[width=\textwidth]{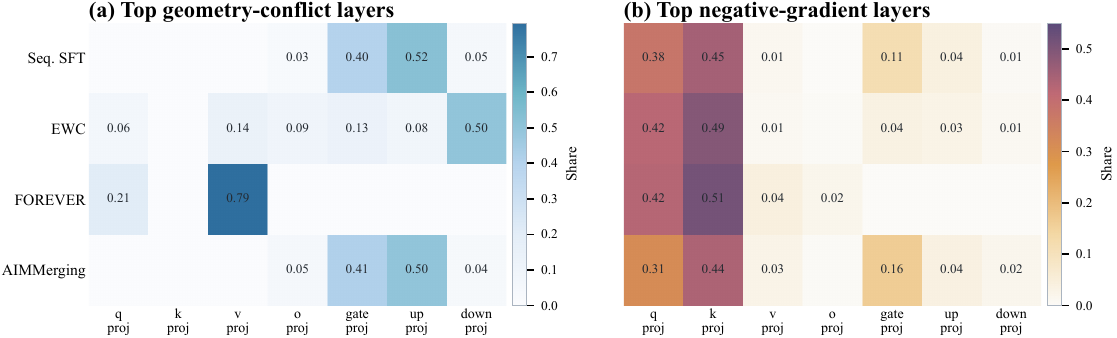}
  \caption{
  {Method-wise top-layer family distribution.}
  We decompose top geometry-conflict layers and top negative-gradient layers by method and module family.
  }
  \label{fig:app_method_family_distribution}
\end{figure*}

Fig.~\ref{fig:app_family_mechanism} and Tables~\ref{tab:app_family_metric_summary}--\ref{tab:app_top_layer_family_distribution} provide the module-level decomposition behind Sec.~\ref{sec:geometry_gradient}.
Geometry conflict and gradient conflict are not redundant: they emphasize different module families and correspond to different failure modes.

At the family-average level, geometry conflict is high across both attention and MLP projections.
The largest mean geometry-conflict values appear in \(q\_\mathrm{proj}\) and \(v\_\mathrm{proj}\) (\(0.56\) for both), while MLP projections also remain high: \(up\_\mathrm{proj}\), \(gate\_\mathrm{proj}\), and \(down\_\mathrm{proj}\) reach \(0.50\), \(0.49\), and \(0.45\), respectively.
This indicates that geometry conflict is not merely a sparse gradient-opposition event; it reflects broad mismatch in task-induced update structure.

The top-layer distribution shows a sharper separation.
Top geometry-conflict layers concentrate in value and MLP-related modules:
\(up\_\mathrm{proj}\) accounts for \(0.28\), \(gate\_\mathrm{proj}\) for \(0.24\), \(v\_\mathrm{proj}\) for \(0.23\), and \(down\_\mathrm{proj}\) for \(0.15\).
Together, these four families explain \(0.90\) of top geometry-conflict layers.
By contrast, top negative-gradient layers concentrate in attention query/key modules:
\(k\_\mathrm{proj}\) accounts for \(0.47\) and \(q\_\mathrm{proj}\) for \(0.38\), together explaining \(0.85\) of top negative-gradient layers.

Fig.~\ref{fig:app_method_family_distribution} further decomposes this pattern by continual post-training method.
The separation between geometry conflict and gradient conflict is stable, but the geometry-conflict locus varies with the update-integration strategy.
For Seq. SFT, top geometry-conflict layers concentrate almost entirely in MLP projections, especially \(up\_\mathrm{proj}\) (\(0.52\)) and \(gate\_\mathrm{proj}\) (\(0.40\)).
AIMMerging shows a similar geometry pattern, with \(up\_\mathrm{proj}\) and \(gate\_\mathrm{proj}\) accounting for \(0.50\) and \(0.41\), respectively.
This suggests that direct sequential update accumulation and ungated merging both induce strong mismatch in MLP transformation pathways.

EWC shifts the geometry-conflict locus toward \(down\_\mathrm{proj}\), which accounts for \(0.50\) of top geometry-conflict layers, with smaller contributions from \(v\_\mathrm{proj}\) (\(0.14\)), \(gate\_\mathrm{proj}\) (\(0.13\)), and \(o\_\mathrm{proj}\) (\(0.09\)).
This is consistent with regularization changing where update mismatch appears rather than removing it entirely.
FOREVER shows a different pattern: top geometry-conflict layers concentrate in \(v\_\mathrm{proj}\) (\(0.79\)) and \(q\_\mathrm{proj}\) (\(0.21\)), suggesting that replay changes the geometry of retained and incoming updates, making attention-value pathways more prominent.

In contrast, top negative-gradient layers are consistently concentrated in query/key projections across all methods.
For Seq. SFT, \(q\_\mathrm{proj}\) and \(k\_\mathrm{proj}\) account for \(0.38\) and \(0.45\); for EWC, \(0.42\) and \(0.49\); for FOREVER, \(0.42\) and \(0.51\); and for AIMMerging, \(0.31\) and \(0.44\).
The only notable secondary contribution is \(gate\_\mathrm{proj}\) under AIMMerging (\(0.16\)) and Seq. SFT (\(0.11\)).
Thus, gradient opposition remains largely attention-centric even when the geometry-conflict locus changes across methods.

This module-level separation supports the main-text interpretation.
Geometry conflict captures mismatch in how task updates should be integrated in weight space, especially in value and MLP transformation pathways.
Gradient conflict captures direct optimization opposition, most visible in attention query/key projections.
The two signals therefore provide complementary views of continual post-training failure: geometry conflict is better suited as an update-integration control signal, while gradient conflict remains useful as a diagnostic for optimization-level interference.

The predictor-level results in Fig.~\ref{fig:app_step_explanation_heatmap} are consistent with this distinction.
For forgetting, the strongest state-relative geometry signal is the global gap (\(\rho_s=-0.59\)), exceeding update norm (\(\rho_s=-0.48\)) and active-pair geometry conflict (\(\rho_s=0.30\)).
For retained-task and overall performance, gradient diagnostics become more visible: minimum gradient cosine reaches \(\rho_s=0.46\) with old-task mean and \(\rho_s=0.33\) with overall score, while negative-gradient ratio reaches \(\rho_s=-0.38\) and \(-0.31\), respectively.
This further supports the complementary role of geometry and gradient conflict.

% ---------------------------------------------------------------------
% Add this as a new subsection after the mechanism analysis.
% ---------------------------------------------------------------------

\FloatBarrier
\section{Additional Experiments for Sec.~\ref{sec:experiments}}
\label{app:sec5_additional}

\subsection{Experimental Setup Details}
\label{app:exp_setup_details}

\paragraph{Model scales and task experts.}
All experiments start from the corresponding Qwen3 backbone and construct task experts relative to the same pretrained initialization.
GCWM and the data-free merging baselines operate only on these task updates at integration time; they do not access replay data during merging.
Seq.\ SFT, EWC, and FOREVER are included as reference continual-training pipelines because they respectively perform sequential optimization, regularized optimization, or replay-based training.

\paragraph{Training data.}
For the domain-continual setting, we use MMLU-Pro-CoT-Train-Labeled and sample 1k training examples for each of the 14 MMLU-Pro sub-domains.
For the capability-continual setting, we construct math and code capability experts from two instruction sources: 30k examples from the math split of Nemotron-Post-Training-Dataset-v1 and 30k examples from CodeFeedback-Filtered-Instruction.
This setup separates the data used to train task experts from the held-out benchmark suite used to measure capability retention and transfer.

\paragraph{Evaluation suite.}
The domain-continual setting evaluates all 14 MMLU-Pro sub-categories and reports both overall accuracy and per-domain accuracy.
The capability-continual setting uses a mixed reasoning and code suite: GPQA-Diamond for graduate-level science reasoning, GSM8K for grade-school mathematical reasoning, MATH-500 for competition-style mathematics, HumanEval for Python code generation, and MMLU-Pro for broad knowledge retention.
We use accuracy or exact-match style scoring for multiple-choice and math benchmarks, and pass@1 for HumanEval.
{The evaluation code we employ strictly adheres to the evaluation and inference configurations from the Qwen Technical Report \cite{yang2025qwen3}, and produces results on the original Qwen3-7B and Qwen3-14B models that are aligned with the Qwen Technical Report.}

\paragraph{Evaluation reliability.}
LLM benchmark scores can vary across repeated runs and execution environments \cite{yang2026inficoevalchain}.
Therefore, all reported capability-continual performance scores are averaged over five independent evaluation runs, using different decoding seeds where stochastic decoding is applied.
We report average accuracy and keep benchmark scripts, decoding settings, and task order fixed across data-free integration methods.

\paragraph{Compute resources.}
Additional task-expert training used an internal Slurm-managed cluster with up to 64 NVIDIA H800 GPUs.
The core GCWM merge, geometry construction, and scaling-law analyses are data-free and run on CPU nodes.
Our CPU runs use Slurm allocations from dual-socket Intel Xeon Platinum 8480CL machines (2 sockets, 56 cores per socket, 224 logical CPUs, 210 MiB L3 cache); typical analysis jobs request a subset of cores (e.g., 24 CPU cores), while large parallel sweeps can request larger CPU allocations.
GCWM introduces no additional inference-time cost after merging: the final model is evaluated with the same architecture as the corresponding merged checkpoint.

\paragraph{Remark:}
Multi-Task Learning (MTL) is treated as a joint-training upper-bound reference rather than a data-free method.
In tables, bold numbers mark the best non-MTL method, while underlined MTL numbers indicate the upper-bound reference.
All data-free update-integration methods are compared under the same trained-expert inputs, so differences reflect the merge-time integration rule rather than additional training data.

\subsection{Evaluation Prompt}
\label{app:evaluation_prompt}

\begin{table}[!htbp]
  \centering
  \label{tab:evaluation_prompt_template}
  \scriptsize
  \setlength{\tabcolsep}{3pt}
  \caption{Prompt Templates for Benchmarks}
\label{tab:full_benchmark_prompts}
\small
\begin{tabularx}{\textwidth}{l >{\raggedright\arraybackslash}X}
\toprule
\textbf{Benchmark} & \textbf{Prompt Template} \\
\midrule
GSM8K & 
\textbf{System:} \texttt{You are a helpful assistant.}\par
\textbf{User:} \texttt{Solve the following math problem step by step. The last line of your response should display the answer enclosed within ANSWER.}\par
\texttt{Example: [Example-Content]}\par
\texttt{question: [Question-Content]}\par
\texttt{Remember to put your answer on its own line at the end in the form ANSWER (without quotes), where \$ANSWER is replaced by the actual answer to the problem.} \\
\addlinespace
HumanEval & 
\textbf{System:} \texttt{You are a helpful assistant.}\par
\textbf{User:} \texttt{Read the following function signature and docstring, and fully implement the function described. Your response should only contain the code for this function.}\par
\texttt{[Code-Content]} \\
\addlinespace
Math500 & 
\textbf{System:} \texttt{You are a helpful assistant.}\par
\textbf{User:} \texttt{[Question-Content]} \\
\addlinespace
MMLU-Pro & 
\textbf{System:} \texttt{You are a helpful assistant.}\par
\textbf{User:} \texttt{Answer the following multiple choice question. The last line of your response should be of the following format: ANSWER: LETTER (without quotes) where LETTER is one of A,B,C,D,E,F,G,H,I,J. Think step by step before answering.}\par
\texttt{Question: [Question-Content]} \\
\addlinespace
MBPP & 
\textbf{System:} \texttt{You are a helpful assistant.}\par
\textbf{User:} \texttt{[Code-Content]} \\
\addlinespace
GPQA-Diamond & 
\textbf{System:} \texttt{You are a helpful assistant.}\par
\textbf{User:} \texttt{Answer the following multiple choice question. The last line of your response should be of the following format: ANSWER: LETTER (without quotes) where LETTER is one of A,B,C,D. Think step by step before answering.} \par
\texttt{[Question-Content]} \\
\bottomrule
\end{tabularx}
\end{table}

\subsection{Performance Context and Data Completeness}
\label{app:performance_context}

\begin{table}[!htbp]
  \centering
  \caption{
  {Final MMLU-Pro overall score by scale and method.}
  }
  \label{tab:app_final_overall_pivot}
  \footnotesize
  \setlength{\tabcolsep}{4pt}
  {\arrayrulecolor{ruleNavy}
\renewcommand{\arraystretch}{1.08}
\begin{tabular}{lrrrr}
\toprule[0.9pt]
\rowcolor{sizeBand}
Size & Seq. SFT & EWC & FOREVER & AIMMerging \\
\midrule
\rowcolor{refGray} 0.6B & 0.244 & 0.248 & 0.247 & 0.271 \\
\rowcolor{mergeBlue} 1.7B & 0.370 & 0.400 & 0.385 & 0.418 \\
\rowcolor{refGray} 4B & 0.516 & 0.554 & 0.547 & 0.581 \\
\rowcolor{mergeBlue} 8B & 0.549 & 0.604 & 0.596 & 0.629 \\
\rowcolor{refGray} 14B & 0.604 & 0.653 & 0.665 & 0.678 \\
\bottomrule[0.9pt]
\end{tabular}
\arrayrulecolor{black}}

\end{table}

\begin{table*}[!htbp]
  \centering
  \caption{
  {Final performance summary.}
  Old mean reports retained-task average at the final step; Forget reports average drop from best previous score.
  }
  \label{tab:app_final_performance_summary}
  \footnotesize
  \setlength{\tabcolsep}{4pt}
  {\arrayrulecolor{ruleNavy}
\renewcommand{\arraystretch}{1.05}
\begin{tabular}{llrrrr}
\toprule[0.9pt]
\rowcolor{sizeBand}
Size & Method & Overall & Old mean & Forget & Current \\
\midrule
\rowcolor{refGray} 0.6B & Seq. SFT & 0.244 & 0.234 & -0.016 & 0.343 \\
\rowcolor{refGray} 0.6B & EWC & 0.248 & 0.240 & -0.023 & 0.284 \\
\rowcolor{refGray} 0.6B & FOREVER & 0.247 & 0.238 & -0.023 & 0.296 \\
\rowcolor{mergeBlue} 0.6B & AIMMerging & 0.271 & 0.260 & -0.011 & 0.351 \\
\midrule
\rowcolor{refGray} 1.7B & Seq. SFT & 0.370 & 0.360 & -0.029 & 0.474 \\
\rowcolor{refGray} 1.7B & EWC & 0.400 & 0.386 & -0.031 & 0.466 \\
\rowcolor{refGray} 1.7B & FOREVER & 0.385 & 0.375 & -0.036 & 0.474 \\
\rowcolor{mergeBlue} 1.7B & AIMMerging & 0.418 & 0.408 & -0.009 & 0.475 \\
\midrule
\rowcolor{refGray} 4B & Seq. SFT & 0.516 & 0.507 & -0.035 & 0.630 \\
\rowcolor{refGray} 4B & EWC & 0.554 & 0.546 & -0.110 & 0.629 \\
\rowcolor{refGray} 4B & FOREVER & 0.547 & 0.534 & -0.034 & 0.635 \\
\rowcolor{mergeBlue} 4B & AIMMerging & 0.581 & 0.569 & -0.006 & 0.627 \\
\midrule
\rowcolor{refGray} 8B & Seq. SFT & 0.549 & 0.543 & -0.054 & 0.672 \\
\rowcolor{refGray} 8B & EWC & 0.604 & 0.596 & -0.028 & 0.682 \\
\rowcolor{refGray} 8B & FOREVER & 0.596 & 0.589 & -0.033 & 0.685 \\
\rowcolor{mergeBlue} 8B & AIMMerging & 0.629 & 0.619 & -0.011 & 0.679 \\
\midrule
\rowcolor{refGray} 14B & Seq. SFT & 0.604 & 0.599 & -0.045 & 0.726 \\
\rowcolor{refGray} 14B & EWC & 0.653 & 0.644 & -0.028 & 0.717 \\
\rowcolor{refGray} 14B & FOREVER & 0.665 & 0.657 & -0.022 & 0.743 \\
\rowcolor{mergeBlue} 14B & AIMMerging & 0.678 & 0.672 & -0.007 & 0.723 \\
\bottomrule[0.9pt]
\end{tabular}
\arrayrulecolor{black}}

\end{table*}

\begin{table*}[!htbp]
  \centering
  \caption{
  {Benchmark completeness and final overall scores.}
  All benchmark files used in the analysis contain 14 evaluated continual steps after the data completion pass.
  }
  \label{tab:app_benchmark_quality}
  \footnotesize
  \setlength{\tabcolsep}{4pt}
  {\arrayrulecolor{ruleNavy}
\renewcommand{\arraystretch}{1.05}
\begin{tabular}{llrrrr}
\toprule[0.9pt]
\rowcolor{sizeBand}
Size & Method & Steps & NaN rows & NaN values & Final overall \\
\midrule
\rowcolor{refGray} 0.6B & Seq. SFT & 14 & 0 & 0 & 0.244 \\
\rowcolor{refGray} 0.6B & EWC & 14 & 0 & 0 & 0.248 \\
\rowcolor{refGray} 0.6B & FOREVER & 14 & 0 & 0 & 0.247 \\
\rowcolor{mergeBlue} 0.6B & AIMMerging & 14 & 0 & 0 & 0.271 \\
\midrule
\rowcolor{refGray} 1.7B & Seq. SFT & 14 & 0 & 0 & 0.370 \\
\rowcolor{refGray} 1.7B & EWC & 14 & 0 & 0 & 0.400 \\
\rowcolor{refGray} 1.7B & FOREVER & 14 & 0 & 0 & 0.385 \\
\rowcolor{mergeBlue} 1.7B & AIMMerging & 14 & 0 & 0 & 0.418 \\
\midrule
\rowcolor{refGray} 4B & Seq. SFT & 14 & 0 & 0 & 0.516 \\
\rowcolor{refGray} 4B & EWC & 14 & 0 & 0 & 0.554 \\
\rowcolor{refGray} 4B & FOREVER & 14 & 0 & 0 & 0.547 \\
\rowcolor{mergeBlue} 4B & AIMMerging & 14 & 0 & 0 & 0.581 \\
\midrule
\rowcolor{refGray} 8B & Seq. SFT & 14 & 0 & 0 & 0.549 \\
\rowcolor{refGray} 8B & EWC & 14 & 0 & 0 & 0.604 \\
\rowcolor{refGray} 8B & FOREVER & 14 & 0 & 0 & 0.596 \\
\rowcolor{mergeBlue} 8B & AIMMerging & 14 & 0 & 0 & 0.629 \\
\midrule
\rowcolor{refGray} 14B & Seq. SFT & 14 & 0 & 0 & 0.604 \\
\rowcolor{refGray} 14B & EWC & 14 & 0 & 0 & 0.653 \\
\rowcolor{refGray} 14B & FOREVER & 14 & 0 & 0 & 0.665 \\
\rowcolor{mergeBlue} 14B & AIMMerging & 14 & 0 & 0 & 0.678 \\
\bottomrule[0.9pt]
\end{tabular}
\arrayrulecolor{black}}

\end{table*}

\begin{figure*}[!htbp]
  \centering
  \includegraphics[width=\textwidth]{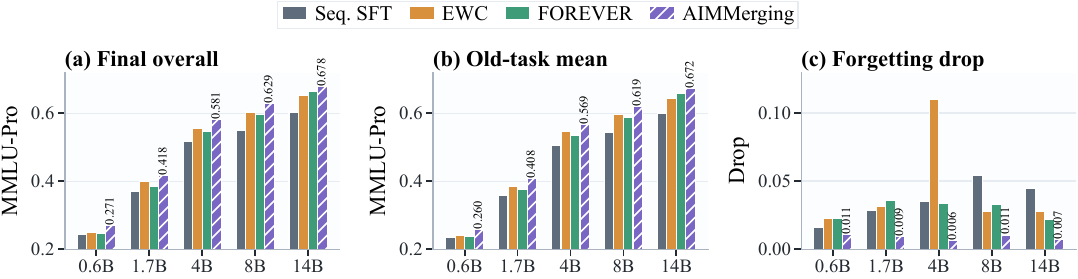}
  \caption{
  {Final performance across model scales and continual post-training methods.}
  }
  \label{fig:app_final_scaling}
\end{figure*}

Fig.~\ref{fig:app_final_scaling} and Tables~\ref{tab:app_final_overall_pivot}--\ref{tab:app_final_performance_summary} provide performance context for the analysis.
AIMMerging improves final overall performance at every scale, from \(0.271\) on Qwen3-0.6B to \(0.678\) on Qwen3-14B.
It also yields consistently small final forgetting, with final forgetting scores of \(-0.011\), \(-0.009\), \(-0.006\), \(-0.011\), and \(-0.007\) from 0.6B to 14B.
These results show that update integration can mitigate forgetting, but the analysis in Sec.~\ref{sec:findings} explains why an explicit geometry-based control signal is needed: existing strategies improve retention without directly modeling state-relative geometry conflict.

Table~\ref{tab:app_benchmark_quality} verifies that all benchmark runs used for the final analysis contain 14 steps and no missing values.
This includes the re-evaluated EWC and SFT files that previously had incomplete MMLU-Pro coverage.

\subsection{Comparison with Non-Continual Model Merging}
\label{app:non_continual_merge}

We additionally compare GCWM with non-continual model merging baselines that merge task experts without explicitly modeling the sequential state.
This comparison isolates a different question from the main continual setting: whether one-shot task-vector merging is sufficient when capability updates must be integrated into a shared LLM.
Table~\ref{tab:app_normal_merging_all_scales} reports the full benchmark results, and Table~\ref{tab:app_normal_merging_gain} summarizes the overall gain over the strongest non-continual merge at each scale.

The results show two clear patterns.
First, DARE is unstable in this setting: its average score collapses to \(15.1\%\) on Qwen3-0.6B, \(32.2\%\) on Qwen3-8B, and \(29.6\%\) on Qwen3-14B, with severe failures on individual benchmarks such as MBPP at 8B (\(0.8\%\)) and MATH-500 at 14B (\(0.4\%\)).
This suggests that sparsifying or rescaling task vectors without modeling continual compatibility can be brittle for heterogeneous math, code, and knowledge updates.
Second, GCWM is more effective beyond the smallest model scale.
Compared with the best non-continual merge, GCWM improves average performance by \(+0.21\), \(+3.24\), \(+5.71\), \(+2.64\), and \(+4.28\) points from 0.6B to 14B, respectively.
The gain is also broad: GCWM matches or exceeds the best non-continual merge on 2/6, 5/6, 6/6, 4/6, and 4/6 benchmarks across the five scales.
These results support the main claim that update integration benefits from an explicit geometry-aware compatibility signal rather than one-shot, state-agnostic task-vector merging.

\definecolor{rcwmBlue}{HTML}{DDEEFF}
\definecolor{mergeBlue}{HTML}{F1F7FD}
\definecolor{refGray}{HTML}{F7F7F7}
\definecolor{mtlGold}{HTML}{FFF4D6}
\definecolor{sizeBand}{HTML}{EAF2FA}
\definecolor{ruleNavy}{HTML}{244B73}
\providecommand{\best}[1]{\textbf{#1}}
\providecommand{\ub}[1]{\underline{#1}}

\begin{table}[!htbp]
  \centering
  \caption{
  {GCWM gains over the strongest non-continual merge.}
  Wins counts the number of benchmarks where GCWM matches or exceeds the best score among TA, TIES, and DARE.
  }
  \label{tab:app_normal_merging_gain}
  \small
  \setlength{\tabcolsep}{5pt}
  \renewcommand{\arraystretch}{1.08}
  \arrayrulecolor{ruleNavy}
  \begin{tabular}{lrrrr}
\toprule[0.9pt]
Scale & Best non-continual merge & GCWM & Gain & Wins \\
\midrule
\rowcolor{rcwmBlue} 0.6B & TIES 39.75 & 39.96 & +0.21 & 2/6 \\
\rowcolor{rcwmBlue} 1.7B & TIES 55.02 & 58.26 & +3.24 & 5/6 \\
\rowcolor{rcwmBlue} 4B & TIES 64.56 & 70.27 & +5.71 & 6/6 \\
\rowcolor{rcwmBlue} 8B & TIES 69.90 & 72.54 & +2.64 & 4/6 \\
\rowcolor{rcwmBlue} 14B & TIES 70.04 & 74.32 & +4.28 & 4/6 \\
\bottomrule[0.9pt]
\end{tabular}

  \arrayrulecolor{black}
\end{table}

\begin{table*}[!htbp]
  \centering
  \caption{
  {Non-continual merging on capability benchmarks.}
  Scores are accuracies or pass@1 (\%).
  Bold marks GCWM scores that match or exceed the best non-continual merging baseline among TA, TIES, and DARE.
  }
  \label{tab:app_normal_merging_all_scales}
  \scriptsize
  \setlength{\tabcolsep}{4pt}
  \renewcommand{\arraystretch}{1.06}
  \arrayrulecolor{ruleNavy}
  \resizebox{0.8\textwidth}{!}{%
\begin{tabular}{lrrrrrrr}
\toprule[0.9pt]
Method & Avg. & GPQA-D & GSM8K & HumanEval & MATH-500 & MBPP & MMLU-Pro \\
\midrule
\rowcolor{sizeBand}
\multicolumn{8}{l}{\emph{Qwen3-0.6B}} \\
\rowcolor{mergeBlue} TA & 39.2 & 24.8 & 52.4 & 34.8 & 41.8 & 51.4 & 29.9 \\
\rowcolor{mergeBlue} TIES & 39.8 & 22.7 & 55.9 & 34.8 & 44.6 & 49.0 & 31.5 \\
\rowcolor{mergeBlue} DARE & 15.1 & 5.6 & 7.2 & 16.5 & 6.8 & 33.5 & 21.1 \\
\rowcolor{rcwmBlue} GCWM & \best{40.0} & 23.2 & 51.2 & \best{36.6} & \best{50.8} & 50.2 & 27.8 \\
\midrule[0.55pt]
\rowcolor{sizeBand}
\multicolumn{8}{l}{\emph{Qwen3-1.7B}} \\
\rowcolor{mergeBlue} TA & 52.5 & 21.7 & 72.1 & 61.6 & 59.2 & 56.0 & 44.7 \\
\rowcolor{mergeBlue} TIES & 55.0 & 31.3 & 74.2 & 59.8 & 62.4 & 56.8 & 45.7 \\
\rowcolor{mergeBlue} DARE & 27.5 & 19.2 & 28.2 & 23.8 & 12.0 & 46.7 & 35.3 \\
\rowcolor{rcwmBlue} GCWM & \best{58.3} & 26.3 & \best{79.0} & \best{67.4} & \best{63.4} & \best{61.5} & \best{52.0} \\
\midrule[0.55pt]
\rowcolor{sizeBand}
\multicolumn{8}{l}{\emph{Qwen3-4B}} \\
\rowcolor{mergeBlue} TA & 63.3 & 32.3 & 86.7 & 70.7 & 64.2 & 70.4 & 55.7 \\
\rowcolor{mergeBlue} TIES & 64.6 & 32.8 & 86.2 & 72.0 & 69.2 & 69.7 & 57.5 \\
\rowcolor{mergeBlue} DARE & 46.6 & 26.8 & 67.4 & 54.3 & 27.6 & 58.8 & 44.8 \\
\rowcolor{rcwmBlue} GCWM & \best{70.3} & \best{35.4} & \best{93.7} & \best{84.8} & \best{71.2} & \best{72.8} & \best{63.9} \\
\midrule[0.55pt]
\rowcolor{sizeBand}
\multicolumn{8}{l}{\emph{Qwen3-8B}} \\
\rowcolor{mergeBlue} TA & 69.3 & 31.3 & 90.0 & 80.5 & 79.8 & 73.5 & 60.8 \\
\rowcolor{mergeBlue} TIES & 69.9 & 38.4 & 89.8 & 77.4 & 78.2 & 75.1 & 60.4 \\
\rowcolor{mergeBlue} DARE & 32.2 & 2.5 & 80.3 & 24.5 & 35.0 & 0.8 & 49.9 \\
\rowcolor{rcwmBlue} GCWM & \best{72.5} & \best{38.8} & \best{94.5} & \best{84.9} & 75.4 & 75.0 & \best{66.5} \\
\midrule[0.55pt]
\rowcolor{sizeBand}
\multicolumn{8}{l}{\emph{Qwen3-14B}} \\
\rowcolor{mergeBlue} TA & 69.1 & 36.4 & 91.2 & 61.0 & 80.4 & 80.5 & 64.9 \\
\rowcolor{mergeBlue} TIES & 70.0 & 36.9 & 89.5 & 84.8 & 66.2 & 79.4 & 63.6 \\
\rowcolor{mergeBlue} DARE & 29.6 & 7.1 & 41.9 & 15.8 & 0.4 & 56.4 & 55.8 \\
\rowcolor{rcwmBlue} GCWM & \best{74.3} & \best{39.9} & \best{95.8} & \best{86.6} & 78.2 & 76.7 & \best{66.2} \\
\bottomrule[0.9pt]
\end{tabular}%
}

  \arrayrulecolor{black}
\end{table*}

\FloatBarrier
\subsection{Full Domain-Continual Results}
\label{app:domain_continual_results}

We provide complete MMLU-Pro domain-continual results across all Qwen3 scales.
Table~\ref{tab:app_mmlu_op_all_scales} summarizes overall accuracy, Table~\ref{tab:app_mmlu_rcwm_gain} reports GCWM's gain over the strongest data-free update-integration baseline at each scale, and Table~\ref{tab:app_mmlu_domain_06b_4b} gives the full per-domain results for Qwen3-0.6B and Qwen3-4B omitted from the main text.
Together with Table~\ref{tab:mmlu_domain_main}, these tables cover all five model scales.

Across all five scales, GCWM achieves the best non-MTL overall accuracy.
It improves over the strongest data-free baseline by \(+0.30\), \(+1.61\), \(+1.19\), \(+0.74\), and \(+1.23\) points on Qwen3-0.6B, 1.7B, 4B, 8B, and 14B, respectively.
The average overall accuracy improves from \(51.1\%\) for AIMMerging and \(51.0\%\) for OPCM to \(52.3\%\) for GCWM.

The 0.6B and 4B results further support the same trend.
On Qwen3-0.6B, GCWM obtains the best non-MTL overall score (\(27.14\%\)), with clear gains on chemistry, engineering, math, and other.
On Qwen3-4B, GCWM improves over all data-free baselines overall (\(59.43\%\)) and matches or exceeds the best non-MTL result on 12 of 14 domains, including business, chemistry, computer science, economics, health, history, law, math, other, philosophy, and physics.
These results indicate that geometry-conflict-controlled integration is not only effective at larger scales, but also improves data-free update integration in small and mid-size LLM regimes.

\definecolor{rcwmBlue}{HTML}{DDEEFF}
\definecolor{mergeBlue}{HTML}{F1F7FD}
\definecolor{refGray}{HTML}{F7F7F7}
\definecolor{mtlGold}{HTML}{FFF4D6}
\definecolor{sizeBand}{HTML}{EAF2FA}
\definecolor{ruleNavy}{HTML}{244B73}
\providecommand{\best}[1]{\textbf{#1}}
\providecommand{\ub}[1]{\underline{#1}}

\begin{table*}[!htbp]
\centering
\caption{All-scale MMLU-Pro overall accuracy. MTL is an upper-bound reference; bold marks the best non-MTL result.}
\label{tab:app_mmlu_op_all_scales}
\footnotesize
\setlength{\tabcolsep}{5pt}
\renewcommand{\arraystretch}{1.08}
\arrayrulecolor{ruleNavy}
\begin{tabular}{lrrrrrr}
\toprule[0.9pt]
Method & 0.6B & 1.7B & 4B & 8B & 14B & Avg. \\
\midrule
\rowcolor{mtlGold} MTL & \ub{27.5} & \ub{44.4} & \ub{61.0} & \ub{65.3} & \ub{68.6} & \ub{53.4} \\
\rowcolor{refGray} Seq.\ SFT & 24.4 & 36.8 & 51.6 & 55.2 & 60.4 & 45.7 \\
\rowcolor{refGray} EWC & 24.9 & 40.0 & 55.4 & 60.4 & 65.3 & 49.2 \\
\rowcolor{refGray} FOREVER & 24.7 & 38.5 & 54.7 & 59.6 & 66.5 & 48.8 \\
\cmidrule(lr){1-7}
\rowcolor{mergeBlue} L\&S & 26.2 & 41.1 & 57.3 & 62.4 & 65.6 & 50.5 \\
\rowcolor{mergeBlue} AIMMerging & 26.5 & 41.8 & 58.1 & 62.9 & 66.4 & 51.1 \\
\rowcolor{mergeBlue} OPCM & 26.8 & 41.7 & 58.2 & 61.9 & 66.6 & 51.0 \\
\rowcolor{rcwmBlue} GCWM & \best{27.1} & \best{43.5} & \best{59.4} & \best{63.7} & \best{67.8} & \best{52.3} \\
\bottomrule[0.9pt]
\end{tabular}
\arrayrulecolor{black}
\end{table*}

\begin{table}[!htbp]
\centering
\caption{GCWM gains over data-free update-integration baselines on MMLU-Pro.}
\label{tab:app_mmlu_rcwm_gain}
\footnotesize
\setlength{\tabcolsep}{4.5pt}
\renewcommand{\arraystretch}{1.08}
\arrayrulecolor{ruleNavy}
\begin{tabular}{lrrrr}
\toprule[0.9pt]
Scale & Best data-free base & GCWM & Gain & Wins vs AIM \\
\midrule
\rowcolor{rcwmBlue} 0.6B & OPCM 26.84 & 27.14 & +0.30 & 7/14 \\
\rowcolor{rcwmBlue} 1.7B & AIMMerging 41.84 & 43.45 & +1.61 & 12/14 \\
\rowcolor{rcwmBlue} 4B & OPCM 58.24 & 59.43 & +1.19 & 14/14 \\
\rowcolor{rcwmBlue} 8B & AIMMerging 62.92 & 63.66 & +0.74 & 8/14 \\
\rowcolor{rcwmBlue} 14B & OPCM 66.61 & 67.84 & +1.23 & 9/14 \\
\bottomrule[0.9pt]
\end{tabular}
\arrayrulecolor{black}
\end{table}

\begin{table*}[!htbp]
\centering
\caption{Additional full MMLU-Pro domain-continual results on Qwen3-0.6B and Qwen3-4B. Scores are accuracies (\%). Underlined MTL is a joint-training upper-bound reference; bold marks the best non-MTL result in each block.}
\label{tab:app_mmlu_domain_06b_4b}
\scriptsize
\setlength{\tabcolsep}{2.15pt}
\renewcommand{\arraystretch}{1.06}
\arrayrulecolor{ruleNavy}
\resizebox{\textwidth}{!}{%
\begin{tabular}{lrrrrrrrrrrrrrrr}
\toprule[0.9pt]
Method & Overall & Bio & Bus & Chem & CS & Econ & Eng & Health & Hist & Law & Math & Other & Phil & Phys & Psych \\
\midrule
\rowcolor{sizeBand}
\multicolumn{16}{l}{\emph{Qwen3-0.6B}} \\
\rowcolor{mtlGold} MTL & \ub{27.5} & \ub{42.1} & \ub{33.3} & \ub{19.9} & \ub{30.0} & \ub{35.2} & \ub{21.1} & \ub{22.3} & \ub{18.6} & \ub{14.7} & \ub{36.1} & \ub{22.5} & \ub{22.7} & \ub{28.9} & \ub{37.1} \\
\rowcolor{refGray} Seq.\ SFT & 24.4 & 33.9 & 28.1 & 20.6 & 22.9 & 29.0 & 17.3 & \best{23.5} & \best{18.9} & \best{14.4} & 33.2 & 21.4 & 17.2 & 23.4 & 34.3 \\
\rowcolor{refGray} EWC & 24.9 & 36.7 & 30.9 & 22.4 & 23.9 & 27.8 & 19.2 & 21.0 & 16.5 & 12.9 & 35.9 & 21.1 & 18.0 & 25.9 & 28.4 \\
\rowcolor{refGray} FOREVER & 24.7 & 36.3 & 28.4 & 21.7 & 22.9 & 30.9 & 19.7 & 19.7 & 15.8 & 11.5 & 36.3 & 19.8 & 20.0 & 26.4 & 29.6 \\
\cmidrule(lr){1-16}
\rowcolor{mergeBlue} L\&S & 26.2 & 41.2 & 32.2 & 18.4 & 28.7 & 34.1 & 19.6 & 20.8 & 17.1 & 13.1 & 35.0 & 21.1 & 21.2 & 27.7 & 36.0 \\
\rowcolor{mergeBlue} AIMMerging & 26.5 & 42.2 & 32.7 & 18.3 & 29.2 & 34.7 & 19.6 & 20.9 & 17.0 & 12.8 & 35.7 & 21.1 & 21.3 & 28.0 & 36.8 \\
\rowcolor{mergeBlue} OPCM & 26.8 & \best{42.3} & \best{33.0} & 18.9 & \best{29.5} & \best{35.0} & 20.1 & 21.4 & 17.6 & 13.5 & 35.9 & 21.6 & \best{21.8} & \best{28.4} & \best{36.9} \\
\rowcolor{rcwmBlue} GCWM & \best{27.1} & 38.9 & 29.1 & \best{23.8} & 28.3 & 34.9 & \best{22.0} & 21.1 & 15.2 & 14.3 & \best{40.0} & \best{22.2} & 21.2 & 26.4 & 35.1 \\
\midrule[0.55pt]
\rowcolor{sizeBand}
\multicolumn{16}{l}{\emph{Qwen3-4B}} \\
\rowcolor{mtlGold} MTL & \ub{61.0} & \ub{81.6} & \ub{67.5} & \ub{67.5} & \ub{65.1} & \ub{70.3} & \ub{52.2} & \ub{57.2} & \ub{50.4} & \ub{35.4} & \ub{71.3} & \ub{51.3} & \ub{53.3} & \ub{65.7} & \ub{69.9} \\
\rowcolor{refGray} Seq.\ SFT & 51.6 & 73.8 & 57.0 & 49.3 & 51.9 & 61.4 & 44.7 & 50.0 & 41.5 & 24.6 & 62.7 & 43.4 & 44.5 & 53.7 & 63.0 \\
\rowcolor{refGray} EWC & 55.4 & \best{75.6} & 62.9 & 57.9 & 61.5 & 64.5 & \best{47.8} & 53.3 & 42.5 & 25.2 & 69.2 & 46.5 & 45.1 & 57.8 & 62.9 \\
\rowcolor{refGray} FOREVER & 54.7 & 72.1 & 59.4 & 57.4 & 58.8 & 64.5 & 41.2 & 53.3 & 40.4 & 24.7 & 71.0 & 46.7 & 45.7 & 59.3 & \best{63.5} \\
\cmidrule(lr){1-16}
\rowcolor{mergeBlue} L\&S & 57.3 & 72.7 & 64.7 & 62.3 & 62.2 & 66.3 & 46.1 & 56.3 & 41.2 & 26.4 & 75.0 & 46.2 & 45.1 & 61.4 & 61.4 \\
\rowcolor{mergeBlue} AIMMerging & 58.1 & 74.3 & 66.0 & 63.6 & 63.4 & 67.8 & 46.8 & 57.3 & 41.7 & 26.3 & 76.7 & 46.9 & 45.7 & 62.6 & 62.7 \\
\rowcolor{mergeBlue} OPCM & 58.3 & 75.2 & 65.7 & 63.8 & 63.5 & 67.8 & 47.0 & 57.5 & 42.1 & 27.0 & 76.5 & 47.2 & 46.0 & 62.6 & 63.1 \\
\rowcolor{rcwmBlue} GCWM & \best{59.4} & 75.1 & \best{66.8} & \best{64.4} & \best{64.2} & \best{68.6} & 47.6 & \best{58.2} & \best{42.6} & \best{27.3} & \best{77.5} & \best{47.7} & \best{46.6} & \best{63.4} & 63.5 \\
\bottomrule[0.9pt]
\end{tabular}%
}
\arrayrulecolor{black}
\end{table*}

\FloatBarrier
\subsection{Full Capability-Continual Results}
\label{app:capability_continual_results}

Table~\ref{tab:app_capability_all_scales} reports full capability-continual results across all five Qwen3 scales, and Table~\ref{tab:app_capability_gain} summarizes GCWM's gain over the strongest data-free baseline at each scale.
Together with Table~\ref{tab:capability_main}, these results cover both the compact main-text comparison and the full scale sweep.

Across all five scales, GCWM has the strongest average among data-free update-integration methods (63.94), ahead of OPCM (61.99), AIMMerging (60.28), and L\&S (59.53).
GCWM is best on overall average at 1.7B, 8B, and 14B, while remaining close to the best baseline at 0.6B and 4B.
The gain over the strongest data-free baseline is -0.11, +5.78, -0.18, +1.61, and +1.39 points on 0.6B, 1.7B, 4B, 8B, and 14B, respectively.

At larger scales, GCWM shows clearer capability transfer benefits.
On 14B, GCWM leads data-free baselines on GPQA-Diamond, GSM8K, HumanEval, and MMLU-Pro, while remaining competitive on MATH-500 and MBPP.
At smaller scales, capability interactions are noisier: for 0.6B, GCWM is strongest on HumanEval and MBPP but slightly behind AIMMerging in average score.
This pattern is consistent with the state-relative geometry findings in Sec.~\ref{sec:findings}: as model capacity increases, compatibility-aware update integration yields more stable gains.

\begin{table*}[!htbp]
\centering
\caption{All-scale capability-continual results. Scores are accuracies or pass@1 (\%). Underlined MTL is a joint-training upper-bound reference; bold marks the best data-free update-integration method in each size block.}
\label{tab:app_capability_all_scales}
\scriptsize
\setlength{\tabcolsep}{4.0pt}
\renewcommand{\arraystretch}{1.06}
\arrayrulecolor{ruleNavy}
\resizebox{0.8\textwidth}{!}{%
\begin{tabular}{lrrrrrrr}
\toprule[0.9pt]
Method & Avg. & GPQA-D. & GSM8K & HumanEval & MATH-500 & MBPP & MMLU-Pro \\
\midrule
\rowcolor{sizeBand}
\multicolumn{8}{l}{\emph{Qwen3-0.6B}} \\
\rowcolor{mtlGold} MTL & \ub{40.7} & \ub{24.8} & \ub{52.4} & \ub{31.7} & \ub{44.4} & \ub{45.9} & \ub{45.1} \\
\rowcolor{refGray} Seq.\ SFT & 37.4 & 19.7 & 51.7 & 36.6 & 49.8 & 42.4 & 24.2 \\
\rowcolor{refGray} EWC & 38.9 & 23.7 & 52.8 & 35.4 & 49.4 & 45.5 & 26.6 \\
\rowcolor{refGray} FOREVER & 40.0 & 18.2 & 61.7 & 37.2 & 51.6 & 42.4 & 29.1 \\
\rowcolor{mergeBlue} L\&S & 39.6 & 23.5 & 55.9 & 32.6 & 48.2 & 48.5 & 29.0 \\
\rowcolor{mergeBlue} AIMMerging & \best{40.1} & \best{23.7} & \best{56.6} & 32.9 & 48.8 & 49.0 & \best{29.4} \\
\rowcolor{mergeBlue} OPCM & 38.8 & 22.3 & 51.0 & 32.9 & \best{51.2} & 48.8 & 26.7 \\
\rowcolor{rcwmBlue} GCWM & 40.0 & 23.2 & 51.2 & \best{36.6} & 50.8 & \best{50.2} & 27.8 \\
\midrule[0.55pt]
\rowcolor{sizeBand}
\multicolumn{8}{l}{\emph{Qwen3-1.7B}} \\
\rowcolor{mtlGold} MTL & \ub{57.1} & \ub{26.7} & \ub{76.1} & \ub{61.6} & \ub{63.6} & \ub{56.8} & \ub{57.5} \\
\rowcolor{refGray} Seq.\ SFT & 51.9 & 18.2 & 70.6 & 64.0 & 64.2 & 59.1 & 35.3 \\
\rowcolor{refGray} EWC & 54.7 & 24.8 & 75.7 & 61.6 & 67.8 & 57.6 & 40.5 \\
\rowcolor{refGray} FOREVER & 58.3 & 27.3 & 76.7 & 62.2 & 69.6 & 66.9 & 47.3 \\
\rowcolor{mergeBlue} L\&S & 52.4 & 21.3 & 71.0 & 62.8 & 57.5 & 58.4 & 43.5 \\
\rowcolor{mergeBlue} AIMMerging & 53.4 & 21.7 & 72.4 & 64.0 & 58.6 & 59.5 & 44.3 \\
\rowcolor{mergeBlue} OPCM & 56.8 & 23.2 & 73.0 & 65.2 & \best{67.6} & 59.9 & 51.9 \\
\rowcolor{rcwmBlue} GCWM & \best{58.3} & \best{26.3} & \best{79.0} & \best{67.4} & 63.4 & \best{61.5} & \best{52.0} \\
\midrule[0.55pt]
\rowcolor{sizeBand}
\multicolumn{8}{l}{\emph{Qwen3-4B}} \\
\rowcolor{mtlGold} MTL & \ub{68.7} & \ub{32.8} & \ub{89.3} & \ub{74.4} & \ub{80.2} & \ub{70.4} & \ub{65.3} \\
\rowcolor{refGray} Seq.\ SFT & 71.6 & 41.4 & 93.9 & 84.8 & 69.4 & 76.6 & 63.6 \\
\rowcolor{refGray} EWC & 68.6 & 37.2 & 90.1 & 82.6 & 71.5 & 70.6 & 60.0 \\
\rowcolor{refGray} FOREVER & 72.5 & 42.4 & 92.7 & 88.4 & 72.6 & 75.9 & 62.8 \\
\rowcolor{mergeBlue} L\&S & 64.7 & 32.9 & 86.7 & 73.4 & 69.6 & 69.8 & 55.6 \\
\rowcolor{mergeBlue} AIMMerging & 65.6 & 33.3 & 87.9 & 74.4 & 70.6 & 70.8 & 56.4 \\
\rowcolor{mergeBlue} OPCM & \best{70.5} & \best{37.4} & \best{94.5} & 82.3 & 70.4 & \best{73.9} & \best{64.2} \\
\rowcolor{rcwmBlue} GCWM & 70.3 & 35.4 & 93.7 & \best{84.8} & \best{71.2} & 72.8 & 63.9 \\
\midrule[0.55pt]
\rowcolor{sizeBand}
\multicolumn{8}{l}{\emph{Qwen3-8B}} \\
\rowcolor{mtlGold} MTL & \ub{73.3} & \ub{35.9} & \ub{92.2} & \ub{82.9} & \ub{87.2} & \ub{76.3} & \ub{65.4} \\
\rowcolor{refGray} Seq.\ SFT & 70.3 & 35.9 & 89.3 & 87.2 & 72.6 & 76.6 & 60.4 \\
\rowcolor{refGray} EWC & 72.6 & 41.3 & 93.5 & 84.7 & 76.5 & 75.4 & 64.2 \\
\rowcolor{refGray} FOREVER & 75.6 & 47.0 & 94.7 & 92.7 & 69.0 & 82.5 & 68.1 \\
\rowcolor{mergeBlue} L\&S & 69.6 & 32.1 & 88.5 & 79.3 & 76.4 & 75.3 & 66.1 \\
\rowcolor{mergeBlue} AIMMerging & 70.2 & 32.3 & 89.2 & 79.9 & \best{77.0} & \best{75.9} & \best{66.7} \\
\rowcolor{mergeBlue} OPCM & 70.9 & 37.8 & \best{94.6} & 79.3 & 74.4 & 74.3 & 65.2 \\
\rowcolor{rcwmBlue} GCWM & \best{72.5} & \best{38.8} & 94.5 & \best{84.9} & 75.4 & 75.0 & 66.5 \\
\midrule[0.55pt]
\rowcolor{sizeBand}
\multicolumn{8}{l}{\emph{Qwen3-14B}} \\
\rowcolor{mtlGold} MTL & \ub{74.6} & \ub{33.3} & \ub{92.1} & \ub{84.2} & \ub{87.8} & \ub{80.5} & \ub{69.8} \\
\rowcolor{refGray} Seq.\ SFT & 70.4 & 43.4 & 95.8 & 86.6 & 66.2 & 63.4 & 67.2 \\
\rowcolor{refGray} EWC & 73.5 & 43.4 & 95.4 & 86.0 & 68.0 & 78.6 & 69.4 \\
\rowcolor{refGray} FOREVER & 75.9 & 54.0 & 96.4 & 87.2 & 69.2 & 75.9 & 72.8 \\
\rowcolor{mergeBlue} L\&S & 71.3 & 38.4 & 94.1 & 83.7 & 76.5 & 78.5 & 56.8 \\
\rowcolor{mergeBlue} AIMMerging & 72.2 & 38.8 & 95.2 & 84.7 & 77.4 & \best{79.4} & 57.5 \\
\rowcolor{mergeBlue} OPCM & 72.9 & 38.0 & 94.5 & 80.5 & \best{79.7} & 78.6 & 66.3 \\
\rowcolor{rcwmBlue} GCWM & \best{74.3} & \best{39.9} & \best{95.8} & \best{86.6} & 78.2 & 76.7 & \best{68.8} \\
\bottomrule[0.9pt]
\end{tabular}%
}

\arrayrulecolor{black}
\end{table*}

\begin{table}[!htbp]
\centering
\caption{GCWM gains over the strongest data-free capability baseline at each scale.}
\label{tab:app_capability_gain}
\footnotesize
\setlength{\tabcolsep}{4.2pt}
\renewcommand{\arraystretch}{1.08}
\arrayrulecolor{ruleNavy}
\begin{tabular}{lrrrr}
\toprule[0.9pt]
Scale & Best data-free base & GCWM-Diamond & Gain & Wins vs AIM \\
\midrule
\rowcolor{rcwmBlue} 0.6B & AIMMerging 40.07 & 39.96 & -0.11 & 3/6 \\
\rowcolor{rcwmBlue} 1.7B & OPCM 56.82 & 58.26 & +1.44 & 6/6 \\
\rowcolor{rcwmBlue} 4B & OPCM 70.45 & 70.27 & -0.18 & 6/6 \\
\rowcolor{rcwmBlue} 8B & OPCM 70.93 & 72.54 & +1.61 & 3/6 \\
\rowcolor{rcwmBlue} 14B & OPCM 72.93 & 74.32 & +1.39 & 5/6 \\
\bottomrule[0.9pt]
\end{tabular}

\arrayrulecolor{black}
\end{table}

\begin{figure}[!htbp]
  \centering
  \includegraphics[width=\linewidth]{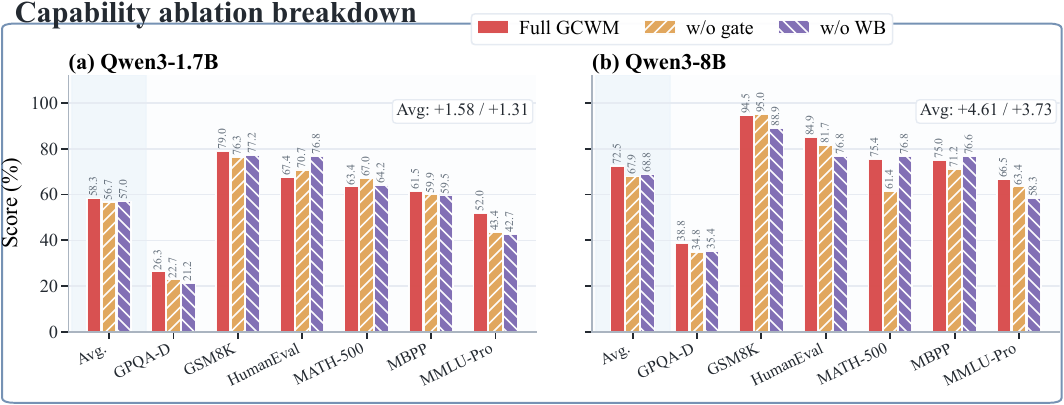}
  \caption{
  {Capability ablation breakdown.}
  Full GCWM is compared with variants that remove conflict gating or replace the Wasserstein barycenter.
  }
  \label{fig:app_capability_ablation_breakdown}
\end{figure}

\section{Additional Ablation Study}
\label{app:ablations}

We provide additional capability-continual ablations on Qwen3-1.7B and Qwen3-8B.
The variants isolate the two merge-time components used by GCWM: the geometry-conflict gate and the Wasserstein shared metric.
The \emph{w/o gate} variant removes conflict-conditioned interpolation and applies the geometry-aware branch uniformly, while \emph{w/o WB} replaces the Wasserstein barycenter with a simpler shared metric.
All variants use the same task experts and evaluation protocol; only the merge rule changes.

Fig.~\ref{fig:app_capability_ablation_breakdown} visualizes the benchmark-level ablations, and Table~\ref{tab:app_capability_ablation_summary} summarizes the aggregate effect.
On Qwen3-1.7B, full GCWM improves the average score from \(56.68\%\) without the gate and \(56.95\%\) without the Wasserstein barycenter to \(58.26\%\).
On Qwen3-8B, the gap is larger: full GCWM reaches \(72.54\%\), compared with \(67.93\%\) without the gate and \(68.81\%\) without the Wasserstein barycenter.
This indicates that both components matter more strongly in the larger-scale capability setting, where math, code, and knowledge updates interact more sharply.

Table~\ref{tab:app_capability_ablation_full} gives the full benchmark breakdown.
The improvements are not uniform across all benchmarks, which is expected for heterogeneous capability updates.
For Qwen3-1.7B, removing the gate can improve MATH-500 and removing the Wasserstein barycenter can improve HumanEval, but both ablations substantially reduce MMLU-Pro and GPQA-Diamond, leading to weaker aggregate performance.
For Qwen3-8B, the gate is especially important for MATH-500 (\(+14.0\) points over w/o gate), while the Wasserstein barycenter is important for GSM8K, HumanEval, MMLU-Pro, and GPQA-Diamond.
Overall, these ablations support the design of GCWM as a compatibility-controlled merge: the gate decides how strongly to trust geometry-aware integration, and the Wasserstein shared metric provides the common space in which heterogeneous task updates can be combined.

\definecolor{rcwmBlue}{HTML}{DDEEFF}
\definecolor{mergeBlue}{HTML}{F1F7FD}
\definecolor{refGray}{HTML}{F7F7F7}
\definecolor{mtlGold}{HTML}{FFF4D6}
\definecolor{sizeBand}{HTML}{EAF2FA}
\definecolor{ruleNavy}{HTML}{244B73}
\providecommand{\best}[1]{\textbf{#1}}
\providecommand{\ub}[1]{\underline{#1}}

\begin{table}[!htbp]
  \centering
  \caption{
  {Capability ablation summary.}
  Drop is the average-score difference between full GCWM and the corresponding ablated variant.
  }
  \label{tab:app_capability_ablation_summary}
  \footnotesize
  \setlength{\tabcolsep}{5pt}
  \renewcommand{\arraystretch}{1.08}
  \arrayrulecolor{ruleNavy}
  \begin{tabular}{lrrrrr}
\toprule[0.9pt]
Scale & Full GCWM & w/o gate & Drop & w/o WB & Drop \\
\midrule
\rowcolor{rcwmBlue} 1.7B & 58.26 & 56.68 & +1.58 & 56.95 & +1.31 \\
\rowcolor{rcwmBlue} 8B & 72.54 & 67.93 & +4.61 & 68.81 & +3.73 \\
\bottomrule[0.9pt]
\end{tabular}

  \arrayrulecolor{black}
\end{table}

\begin{table*}[!htbp]
  \centering
  \caption{
  {Capability ablation breakdown on Qwen3-1.7B and 8B.}
  Scores are accuracies or pass@1 (\%).
  Bold marks the best variant within each scale and metric.
  }
  \label{tab:app_capability_ablation_full}
  \scriptsize
  \setlength{\tabcolsep}{2.5pt}
  \renewcommand{\arraystretch}{1.06}
  \arrayrulecolor{ruleNavy}
  \resizebox{0.75\textwidth}{!}{%
\begin{tabular}{lrrrrrrr}
\toprule[0.9pt]
Variant & Avg. & GPQA-D & GSM8K & HumanEval & MATH-500 & MBPP & MMLU-Pro \\
\midrule
\rowcolor{sizeBand}
\multicolumn{8}{l}{\emph{Qwen3-1.7B}} \\
\rowcolor{rcwmBlue} GCWM & \best{58.3} & \best{26.3} & \best{79.0} & 67.4 & 63.4 & \best{61.5} & \best{52.0} \\
\rowcolor{mergeBlue} w/o gate & 56.7 & 22.7 & 76.3 & 70.7 & \best{67.0} & 59.9 & 43.4 \\
\rowcolor{mergeBlue} w/o WB & 57.0 & 21.2 & 77.2 & \best{76.8} & 64.2 & 59.5 & 42.7 \\
\midrule[0.55pt]
\rowcolor{sizeBand}
\multicolumn{8}{l}{\emph{Qwen3-8B}} \\
\rowcolor{rcwmBlue} GCWM & \best{72.5} & \best{38.8} & 94.5 & \best{84.9} & 75.4 & 75.0 & \best{66.5} \\
\rowcolor{mergeBlue} w/o gate & 67.9 & 34.8 & \best{95.0} & 81.7 & 61.4 & 71.2 & 63.4 \\
\rowcolor{mergeBlue} w/o WB & 68.8 & 35.4 & 88.9 & 76.8 & \best{76.8} & \best{76.6} & 58.3 \\
\bottomrule[0.9pt]
\end{tabular}%
}

  \arrayrulecolor{black}
\end{table*}

\FloatBarrier

\section{Runtime and Memory Profiling}
\label{app:runtime_memory_profile}

We profile GCWM merge-time cost on Qwen3-8B and Qwen3-14B under the capability-continual setting.
The purpose is to measure practical feasibility rather than training efficiency: GCWM is a data-free merge-time method, so its overhead is incurred once during model integration and adds no inference-time cost.
Profiling is performed on a Slurm job using one visible GPU on node \texttt{kb3-a1-nv-dgx02}; each continual step is timed separately and peak GPU memory is recorded.
The active update count \(m\) increases with the number of merged updates, while the retained rank is fixed at \(r=16\). The 8B profile covers the full three-step sequence (MMLU, math, code), including the \(m=3\) code step. For 14B, we report the two profiled steps (MMLU and math), which provide a matched \(m=1,2\) comparison against 8B and isolate the scale effect without treating 14B as a full three-step average.

Fig.~\ref{fig:app_runtime_memory_profile} visualizes the merge-time decomposition and peak memory trend, and Tables~\ref{tab:app_runtime_memory_summary}--\ref{tab:app_runtime_memory_steps} report the exact summary and per-step values.
On Qwen3-8B, the average merge step takes \(40.5\pm19.7\) minutes with \(7.8\pm3.4\) GB peak GPU memory across the three profiled steps.
On Qwen3-14B, the two matched profiled steps take \(76.2\pm34.8\) minutes with \(11.7\pm4.7\) GB peak GPU memory.
The cost grows with active updates because the union rank increases with \(m\), but the expensive SPD operations remain in the projected rank-\(r\) or union-rank space rather than in the full input dimension.
In practice, the dominant wall-clock terms are SVD/metric preparation and inner merge optimization; Bures--Wasserstein barycenter and matrix square-root operations are comparatively small in the profiled runs.
For example, averaged over the available steps, barycenter plus matrix square-root/inverse-square-root takes \(52.1\) seconds on 8B and \(30.5\) seconds on 14B.

The implementation follows the projected low-rank formulation in Sec.~\ref{sec:method}.
For an update matrix with input dimension \(d_{\mathrm{in}}\), output dimension \(d_{\mathrm{out}}\), active updates \(m\), retained rank \(r\), and barycenter iterations \(I\), the dominant per-layer cost can be summarized as
\[
O(m d_{\mathrm{out}} d_{\mathrm{in}} r)
+ O(I r^3)
+ O(m d_{\mathrm{out}} r^2),
\]
up to implementation constants and merge-optimizer iterations.
This highlights why GCWM is feasible at 8B/14B scales: the Bures--Wasserstein and square-root operations are applied to low-rank projected SPD matrices rather than dense \(d_{\mathrm{in}}\times d_{\mathrm{in}}\) operators.

\begin{figure}[!htbp]
  \centering
  \includegraphics[width=\linewidth]{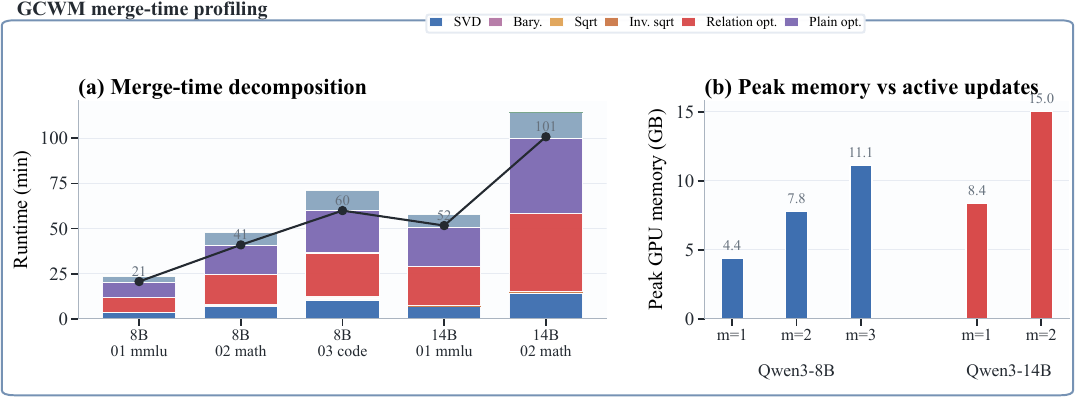}
  \caption{
  {GCWM merge-time runtime and memory profiling.}
  Runtime is decomposed by major merge-time components, and memory is reported as peak allocated GPU memory.
  }
  \label{fig:app_runtime_memory_profile}
\end{figure}

\begin{table}[!htbp]
  \centering
  \caption{
  {GCWM runtime and memory summary.}
  Values are mean \(\pm\) standard deviation across the profiled continual steps; 14B uses the two matched \(m=1,2\) steps.
  }
  \label{tab:app_runtime_memory_summary}
  \scriptsize
  \setlength{\tabcolsep}{3.3pt}
  \renewcommand{\arraystretch}{1.08}
  \arrayrulecolor{ruleNavy}
  \resizebox{0.8\linewidth}{!}{\begin{tabular}{lrrrrrrrr}
\toprule[0.9pt]
Model & Steps & Layers & \(m\) & \(r\) & Union rank & Wall time & Peak mem. & Bary.+sqrt \\
 &  &  &  &  &  & (min/step) & (GB) & (s) \\
\midrule
\rowcolor{rcwmBlue} 8B & 3 & 252 & 2.0 & 16 & 32 & 40.5 $\pm$ 19.7 & 7.8 $\pm$ 3.4 & 52.1 \\
\rowcolor{rcwmBlue} 14B & 2 & 280 & 1.5 & 16 & 24 & 76.2 $\pm$ 34.8 & 11.7 $\pm$ 4.7 & 30.5 \\
\bottomrule[0.9pt]
\end{tabular}
}
  \arrayrulecolor{black}
\end{table}

\begin{table}[!htbp]
  \centering
  \caption{
  {Per-step GCWM profiling.}
  \(m\) denotes the number of active updates in the merge step.
  }
  \label{tab:app_runtime_memory_steps}
  \scriptsize
  \setlength{\tabcolsep}{3.4pt}
  \renewcommand{\arraystretch}{1.08}
  \arrayrulecolor{ruleNavy}
  \resizebox{0.8\linewidth}{!}{\begin{tabular}{llrrrrr}
\toprule[0.9pt]
Model & Step & \(m\) & Union rank & Wall time (min) & Peak mem. (GB) & Bary. (s) \\
\midrule
\rowcolor{mergeBlue} 8B & 01\_mmlu & 1 & 16 & 20.6 & 4.4 & 0.6 \\
\rowcolor{mergeBlue} 8B & 02\_math & 2 & 32 & 40.9 & 7.8 & 25.2 \\
\rowcolor{mergeBlue} 8B & 03\_code & 3 & 48 & 59.9 & 11.1 & 55.0 \\
\rowcolor{mergeBlue} 14B & 01\_mmlu & 1 & 16 & 51.6 & 8.4 & 3.9 \\
\rowcolor{mergeBlue} 14B & 02\_math & 2 & 32 & 100.8 & 15.0 & 28.4 \\
\bottomrule[0.9pt]
\end{tabular}
}
  \arrayrulecolor{black}
\end{table}

\FloatBarrier

\section{Hyperparameter Sensitivity}
\label{app:hyperparameter_sensitivity}

We evaluate one-dimensional hyperparameter sweeps on Qwen3-8B under the capability-continual setting.
Each sweep changes one parameter while keeping the remaining GCWM configuration fixed, which is intended to test robustness rather than tune on the evaluation set.
We vary the retained geometry energy, the conflict threshold \(\tau\), the retained SVD rank \(r\), the gate sharpness \(\kappa\), and the outer merge coefficient \(\eta_t\).
The default configuration corresponds to energy \(0.90\), \(\tau=0.12\), \(r=16\), \(\kappa=10\), and \(\eta_t=0.1\).

Fig.~\ref{fig:app_hyperparam_sensitivity} summarizes the sensitivity profile, while Tables~\ref{tab:app_hyperparam_sensitivity_summary}--\ref{tab:app_hyperparam_sensitivity_full} provide exact values.
GCWM is stable under moderate choices of the geometry and gate parameters.
Changing \(\tau\) over \(\{0.08,0.12,0.18\}\), \(r\) over \(\{8,16,32,64\}\), and \(\kappa\) over \(\{5,10,20\}\) changes the average score by only \(1.6\), \(0.9\), and \(1.4\) points, respectively.
The energy threshold has a larger but interpretable trade-off: energy \(0.95\) gives the best average score (\(72.2\%\)), while the default energy \(0.90\) is more conservative on MMLU-Pro and GPQA-Diamond.
The dominant sensitivity is the outer merge coefficient \(\eta_t\): moderate values remain competitive, with \(\eta_t=0.3\) achieving \(70.8\%\), but overly aggressive integration collapses performance at \(\eta_t=1.0\) (\(34.3\%\)).
This supports the use of a conservative validation-free coefficient schedule for data-free continual update integration.

\begin{figure}[!htbp]
  \centering
  \includegraphics[width=\linewidth]{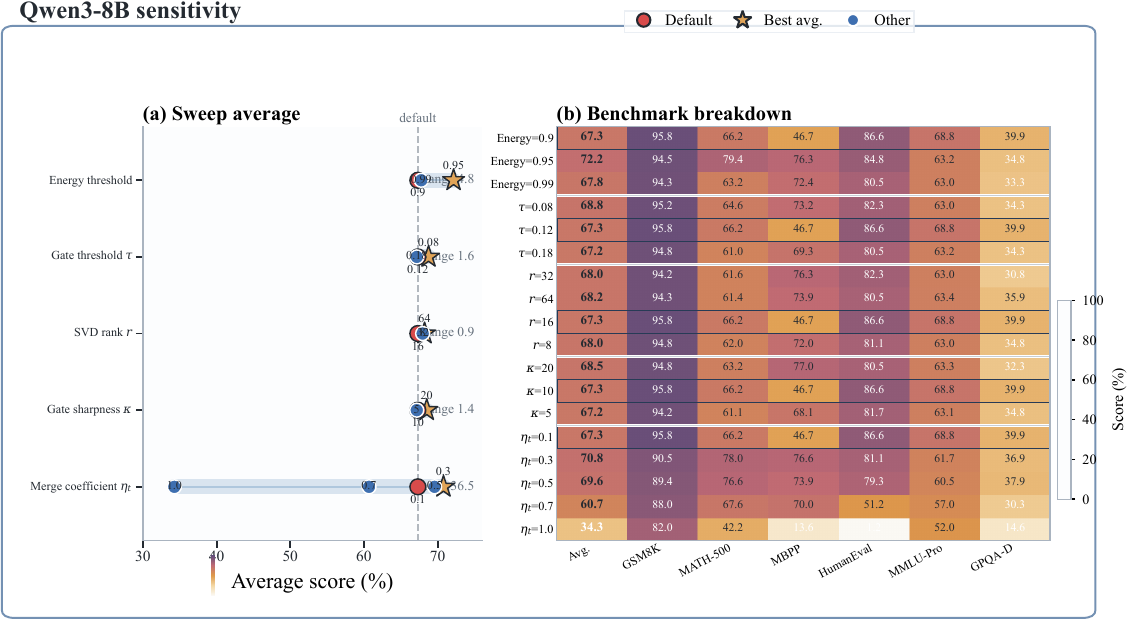}
  \caption{
  {GCWM hyperparameter sensitivity on Qwen3-8B.}
  Each sweep changes one parameter while keeping the others fixed; default rows are outlined and best average settings are marked by stars.
  }
  \label{fig:app_hyperparam_sensitivity}
\end{figure}

\begin{table}[!htbp]
  \centering
  \caption{
  {GCWM hyperparameter sensitivity summary.}
  Scores are average capability-continual performance (\%); range is best minus worst within each one-dimensional sweep.
  }
  \label{tab:app_hyperparam_sensitivity_summary}
  \footnotesize
  \setlength{\tabcolsep}{5pt}
  \renewcommand{\arraystretch}{1.08}
  \arrayrulecolor{ruleNavy}
  \begin{tabular}{llccc}
\toprule
Sweep & Default & Best avg. & Worst avg. & Range \\
\midrule
Energy threshold & 0.9 (67.3) & \best{0.95 (72.2)} & 0.9 (67.3) & 4.8 \\
Gate threshold $\tau$ & 0.12 (67.3) & \best{0.08 (68.8)} & 0.18 (67.2) & 1.6 \\
SVD rank $r$ & 16 (67.3) & \best{64 (68.2)} & 16 (67.3) & 0.9 \\
Gate sharpness $\kappa$ & 10 (67.3) & \best{20 (68.5)} & 5 (67.2) & 1.4 \\
Merge coefficient $\eta_t$ & 0.1 (67.3) & \best{0.3 (70.8)} & 1.0 (34.3) & 36.5 \\
\bottomrule
\end{tabular}

  \arrayrulecolor{black}
\end{table}

\begin{table*}[!htbp]
  \centering
  \caption{
  {Full Qwen3-8B hyperparameter sensitivity breakdown.}
  Scores are accuracies or pass@1 (\%); shaded rows denote the default setting for each sweep.
  }
  \label{tab:app_hyperparam_sensitivity_full}
  \scriptsize
  \setlength{\tabcolsep}{3.0pt}
  \renewcommand{\arraystretch}{1.06}
  \arrayrulecolor{ruleNavy}
  \resizebox{\textwidth}{!}{\begin{tabular}{llrrrrrrr}
\toprule
Sweep & Setting & Avg. & GSM8K & MATH-500 & MBPP & HumanEval & MMLU-Pro & GPQA-D \\
\midrule
\rowcolor{mergeBlue}
Energy threshold & 0.9 & 67.3 & 95.8 & 66.2 & 46.7 & 86.6 & 68.8 & 39.9 \\
 & 0.95 & \best{72.2} & 94.5 & 79.4 & 76.3 & 84.8 & 63.2 & 34.8 \\
 & 0.99 & 67.8 & 94.3 & 63.2 & 72.4 & 80.5 & 63.0 & 33.3 \\
\addlinespace[1pt]
Gate threshold $\tau$ & 0.08 & \best{68.8} & 95.2 & 64.6 & 73.2 & 82.3 & 63.0 & 34.3 \\
\rowcolor{mergeBlue}
 & 0.12 & 67.3 & 95.8 & 66.2 & 46.7 & 86.6 & 68.8 & 39.9 \\
 & 0.18 & 67.2 & 94.8 & 61.0 & 69.3 & 80.5 & 63.2 & 34.3 \\
\addlinespace[1pt]
SVD rank $r$ & 32 & 68.0 & 94.2 & 61.6 & 76.3 & 82.3 & 63.0 & 30.8 \\
 & 64 & \best{68.2} & 94.3 & 61.4 & 73.9 & 80.5 & 63.4 & 35.9 \\
\rowcolor{mergeBlue}
 & 16 & 67.3 & 95.8 & 66.2 & 46.7 & 86.6 & 68.8 & 39.9 \\
 & 8 & 68.0 & 94.8 & 62.0 & 72.0 & 81.1 & 63.0 & 34.8 \\
\addlinespace[1pt]
Gate sharpness $\kappa$ & 20 & \best{68.5} & 94.8 & 63.2 & 77.0 & 80.5 & 63.3 & 32.3 \\
\rowcolor{mergeBlue}
 & 10 & 67.3 & 95.8 & 66.2 & 46.7 & 86.6 & 68.8 & 39.9 \\
 & 5 & 67.2 & 94.2 & 61.1 & 68.1 & 81.7 & 63.1 & 34.8 \\
\addlinespace[1pt]
\rowcolor{mergeBlue}
Merge coefficient $\eta_t$ & 0.1 & 67.3 & 95.8 & 66.2 & 46.7 & 86.6 & 68.8 & 39.9 \\
 & 0.3 & \best{70.8} & 90.5 & 78.0 & 76.6 & 81.1 & 61.7 & 36.9 \\
 & 0.5 & 69.6 & 89.4 & 76.6 & 73.9 & 79.3 & 60.5 & 37.9 \\
 & 0.7 & 60.7 & 88.0 & 67.6 & 70.0 & 51.2 & 57.0 & 30.3 \\
 & 1.0 & 34.3 & 82.0 & 42.2 & 13.6 & 1.2 & 52.0 & 14.6 \\
\bottomrule
\end{tabular}
}
  \arrayrulecolor{black}
\end{table*}

\FloatBarrier

%%%%%%%%%%%%%%%%%%%%%%%%%%%%%%%%%%%%%%%%%%%%%%%%%%%%%%%%%%%%
\newpage
\section*{NeurIPS Paper Checklist}

\begin{enumerate}

\item {\bf Claims}
    \item[] Question: Do the main claims made in the abstract and introduction accurately reflect the paper's contributions and scope?
    \item[] Answer: \answerYes{} % Replace by \answerYes{}, \answerNo{}, or \answerNA{}.
    \item[] Justification: The abstract and introduction state the paper's core claims (state-relative geometry explanation and GCWM as a data-free integration method), and these are supported by the empirical/theoretical sections (Secs.~3--5 and Appendix).

\item {\bf Limitations}
    \item[] Question: Does the paper discuss the limitations of the work performed by the authors?
    \item[] Answer: \answerYes{} % Replace by \answerYes{}, \answerNo{}, or \answerNA{}.
    \item[] Justification: Limitations are discussed in Appendix~\ref{app:limitations}, including benchmark scope, the interpretation of geometry conflict, replay-vs-data-free tradeoffs, and offline compute overhead.

\item {\bf Theory assumptions and proofs}
    \item[] Question: For each theoretical result, does the paper provide the full set of assumptions and a complete (and correct) proof?
    \item[] Answer: \answerYes{} % Replace by \answerYes{}, \answerNo{}, or \answerNA{}.
    \item[] Justification: Theoretical assumptions, statements, and proofs are provided in the main text and Appendix (Sec.~4 and Appendices~A--B), including theorem assumptions and full derivations.

    \item {\bf Experimental result reproducibility}
    \item[] Question: Does the paper fully disclose all the information needed to reproduce the main experimental results of the paper to the extent that it affects the main claims and/or conclusions of the paper (regardless of whether the code and data are provided or not)?
    \item[] Answer: \answerYes{} % Replace by \answerYes{}, \answerNo{}, or \answerNA{}.
    \item[] Justification: The draft specifies the continual setup, metrics, baselines, and evaluation protocol, and provides implementation/analysis details in Appendix (Secs.~5.1 and Appendix~\ref{app:exp_setup_details}).

\item {\bf Open access to data and code}
    \item[] Question: Does the paper provide open access to the data and code, with sufficient instructions to faithfully reproduce the main experimental results, as described in supplemental material?
    \item[] Answer: \answerYes{} % Replace by \answerYes{}, \answerNo{}, or \answerNA{}.
    \item[] Justification: Appendix~\ref{app:availability_compute_impact} provides an anonymized code and data availability link and states that scripts, configurations, processed tables, and task-split reconstruction instructions are included.

\item {\bf Experimental setting/details}
    \item[] Question: Does the paper specify all the training and test details (e.g., data splits, hyperparameters, how they were chosen, type of optimizer) necessary to understand the results?
    \item[] Answer: \answerYes{} % Replace by \answerYes{}, \answerNo{}, or \answerNA{}.
    \item[] Justification: Training/evaluation setup, model scales, task construction, benchmarks, and baseline definitions are documented in Sec.~5.1 and expanded in Appendix~\ref{app:exp_setup_details}.

\item {\bf Experiment statistical significance}
    \item[] Question: Does the paper report error bars suitably and correctly defined or other appropriate information about the statistical significance of the experiments?
    \item[] Answer: \answerYes{} % Replace by \answerYes{}, \answerNo{}, or \answerNA{}.
    \item[] Justification: We report statistical confidence for key Sec.~3 claims via run-cluster bootstrap confidence intervals and permutation tests (Appendix~\ref{app:sec3_confidence}, Fig.~\ref{fig:app_sec3_stat_confidence}, Tables~\ref{tab:app_sec3_conf_step_global}--\ref{tab:app_sec3_conf_geo_grad}).

\item {\bf Experiments compute resources}
    \item[] Question: For each experiment, does the paper provide sufficient information on the computer resources (type of compute workers, memory, time of execution) needed to reproduce the experiments?
    \item[] Answer: \answerYes{} % Replace by \answerYes{}, \answerNo{}, or \answerNA{}.
    \item[] Justification: Appendix~\ref{app:exp_setup_details} reports the GPU resources used for expert training and the CPU resources used for data-free merging, geometry construction.
    
\item {\bf Code of ethics}
    \item[] Question: Does the research conducted in the paper conform, in every respect, with the NeurIPS Code of Ethics \url{https://neurips.cc/public/EthicsGuidelines}?
    \item[] Answer: \answerYes{} % Replace by \answerYes{}, \answerNo{}, or \answerNA{}.
    \item[] Justification: The research is conducted on open benchmarks and model post-training settings and follows NeurIPS ethical expectations. No human-subject data collection or deceptive deployment study is involved.

\item {\bf Broader impacts}
    \item[] Question: Does the paper discuss both potential positive societal impacts and negative societal impacts of the work performed?
    \item[] Answer: \answerYes{} % Replace by \answerYes{}, \answerNo{}, or \answerNA{}.
    \item[] Justification: Appendix~\ref{app:availability_compute_impact} includes a concise broader-impact discussion covering data-free continual adaptation, possible misuse of easier post-training, and the need for downstream safety checks.
    
\item {\bf Safeguards}
    \item[] Question: Does the paper describe safeguards that have been put in place for responsible release of data or models that have a high risk for misuse (e.g., pre-trained language models, image generators, or scraped datasets)?
    \item[] Answer: \answerNA{} % Replace by \answerYes{}, \answerNo{}, or \answerNA{}.
    \item[] Justification: This submission does not release a new high-risk foundation model or scraped dataset artifact; it studies update-integration behavior on existing model families and benchmarks.

\item {\bf Licenses for existing assets}
    \item[] Question: Are the creators or original owners of assets (e.g., code, data, models), used in the paper, properly credited and are the license and terms of use explicitly mentioned and properly respected?
    \item[] Answer: \answerYes{} % Replace by \answerYes{}, \answerNo{}, or \answerNA{}.
    \item[] Justification: The paper credits upstream datasets/models/benchmarks in Sec.~5.1 and citations. We use publicly available assets with standard research usage terms and will add explicit license names/URLs in Appendix for clarity.

\item {\bf New assets}
    \item[] Question: Are new assets introduced in the paper well documented and is the documentation provided alongside the assets?
    \item[] Answer: \answerNA{} % Replace by \answerYes{}, \answerNo{}, or \answerNA{}.
    \item[] Justification: The paper does not introduce a new standalone dataset or model asset release in the current submission package.

\item {\bf Crowdsourcing and research with human subjects}
    \item[] Question: For crowdsourcing experiments and research with human subjects, does the paper include the full text of instructions given to participants and screenshots, if applicable, as well as details about compensation (if any)? 
    \item[] Answer: \answerNA{} % Replace by \answerYes{}, \answerNo{}, or \answerNA{}.
    \item[] Justification: The work does not involve crowdsourcing tasks or human-subject experiments.

\item {\bf Institutional review board (IRB) approvals or equivalent for research with human subjects}
    \item[] Question: Does the paper describe potential risks incurred by study participants, whether such risks were disclosed to the subjects, and whether Institutional Review Board (IRB) approvals (or an equivalent approval/review based on the requirements of your country or institution) were obtained?
    \item[] Answer: \answerNA{} % Replace by \answerYes{}, \answerNo{}, or \answerNA{}.
    \item[] Justification: The work does not involve human-subject research requiring IRB review.

\item {\bf Declaration of LLM usage}
    \item[] Question: Does the paper describe the usage of LLMs if it is an important, original, or non-standard component of the core methods in this research? Note that if the LLM is used only for writing, editing, or formatting purposes and does \emph{not} impact the core methodology, scientific rigor, or originality of the research, declaration is not required.
    %this research? 
    \item[] Answer: \answerNA{} % Replace by \answerYes{}, \answerNo{}, or \answerNA{}.
    \item[] Justification: LLMs are used for writing, editing, and formatting.

\end{enumerate}

\end{document}